\documentclass{article}

\usepackage{microtype}
\usepackage{graphicx}
\usepackage{subfigure}
\usepackage{booktabs} 

\usepackage{hyperref}

\usepackage[accepted]{icml2024}

\usepackage{amsmath}
\usepackage{amssymb}
\usepackage{mathtools}
\usepackage{amsthm}

\usepackage[capitalize,noabbrev]{cleveref}

\usepackage{url}            
\usepackage{amsfonts}       
\usepackage{nicefrac}       
\usepackage{xcolor}         
\usepackage{color}
\usepackage{bm}
\usepackage{bbm}
\usepackage{algorithmic}
\usepackage{algorithm}
\usepackage{enumitem}

\newcommand{\ve}[1]{\mathbf{\bm{#1}}}  
\newcommand{\m}[1]{\mathbf{\bm{#1}}}  
\newcommand{\R}{\mathbb R}
\newcommand{\T}{\mathsf{T}}

\newcommand{\x}{{\ve x}}

\newcommand{\s}{{\ve s}}

\newcommand{\y}{{\ve y}}
\newcommand{\z}{{\ve z}}
\newcommand{\h}{{\ve h}}
\newcommand{\f}{{\ve f}}
\newcommand{\g}{{\ve g}}

\newcommand{\gbt}{{\tilde{\ve g}}}

\newcommand{\phit}{{\tilde{\phi}}}

\newcommand{\lambdat}{{\tilde{\lambda}}}
\newcommand{\alphab}{{\ve \alpha}}
\newcommand{\betab}{{\ve \beta}}

\newcommand{\etab}{{\ve \eta}}

\newcommand{\V}{{\mathcal{V}}}

\newcommand{\M}{{\mathcal{M}}}
\newcommand{\X}{{\mathcal{X}}}

\newcommand{\Vs}{{\V_\mathcal{S}}}
\newcommand{\Vsm}{{\V_\mathcal{S}^m}}
\newcommand{\Vsmp}{{\V_\mathcal{S}^{m'}}}
\newcommand{\Vsmk}{{\V_\mathcal{S}^{m_k}}}
\newcommand{\Vsmast}{{\V_\mathcal{S}^{m_\ast}}}
\newcommand{\Vb}{{\V_{\text{b}}}}
\newcommand{\Vp}{{\V_{\text{p}}}}
\newcommand{\Vc}{{\V_{\text{c}}}}
\newcommand{\Vxm}{{\V_\mathcal{X}^m}}

\newcommand{\dsm}{{d_\mathcal{S}^m}}
\newcommand{\dsmp}{{d_\mathcal{S}^{m'}}}
\newcommand{\dsmk}{{d_\mathcal{S}^{m_k}}}
\newcommand{\dxm}{{d_\mathcal{X}^m}}

\newtheorem{Theorem}{Theorem}
\newtheorem{Proposition}{Proposition}

\newtheorem{Lemma}{Lemma}
\newtheorem{Definition}{Definition}

\theoremstyle{plain}

\theoremstyle{definition}

\theoremstyle{remark}

\usepackage[textsize=tiny]{todonotes}

\icmltitlerunning{Causal Representation Learning Made Identifiable by Grouping of Observational Variables}

\begin{document}

\twocolumn[
\icmltitle{Causal Representation Learning Made Identifiable \\by Grouping of Observational Variables}



\icmlsetsymbol{equal}{*}

\begin{icmlauthorlist}
\icmlauthor{Hiroshi~Morioka}{riken}
\icmlauthor{Aapo~Hyv\"{a}rinen}{helsinki}
\end{icmlauthorlist}

\icmlaffiliation{riken}{RIKEN Center for Advanced Intelligence Project, Kyoto, Japan}
\icmlaffiliation{helsinki}{Department of Computer Science, University of Helsinki, Helsinki, Finland}

\icmlcorrespondingauthor{Hiroshi Morioka}{hiroshi.morioka@riken.jp}

\icmlkeywords{Machine Learning, ICML}

\vskip 0.3in
]



\printAffiliationsAndNotice{}  

\begin{abstract}
A topic of great current interest is Causal Representation Learning (CRL), whose goal is to learn a causal model for hidden features in a data-driven manner. Unfortunately, CRL is severely ill-posed since it is a combination of the two notoriously ill-posed problems of representation learning and causal discovery. Yet, finding practical identifiability conditions that guarantee a unique solution is crucial for its practical applicability. Most approaches so far have been based on assumptions on the latent causal mechanisms, such as temporal causality, or existence of supervision or interventions; these can be too restrictive in actual applications. Here, we show identifiability based on novel, weak constraints, which requires no temporal structure, intervention, nor weak supervision. The approach is based on assuming the observational mixing exhibits a suitable grouping of the observational variables. We also propose a novel self-supervised estimation framework consistent with the model, prove its statistical consistency, and experimentally show its superior CRL performances compared to the state-of-the-art baselines. We further demonstrate its robustness against latent confounders and causal cycles.
\end{abstract}

\section{Introduction}

Causal discovery aims to learn causal interactions among observed variables in a data-driven manner \citep{Pearl2000}. The goal is to estimate a causal graph, also called an adjacency matrix, from passively observed data, with minimal assumptions.
It plays an important role in a wide variety of fields, enabling fundamental insight into causal mechanisms latent in the data; importantly, this is possible without conducting expensive and time-consuming interventional experiments. However, the problem is ill-posed in general, 
and thus the main focus of causal discovery research is to find conditions where the causal graph can be uniquely determined \citep{Andersson1997, Spirtes2001}.
A large number of studies have been conducted so far; they have basically found that imposing some asymmetricity into the model, such as nonlinearity or non-Gaussianity, enables its unique identification \citep{Hoyer2008, Peters2014, Shimizu2006, Shimizu2011, Zhang2009}.

A crucial and implicit assumption of most causal discovery research  is that we know exactly {\it what} constitutes the {\it causal variables}; in most cases, we implicitly assume that each observational variable corresponds to a single causal variable, i.e.\ a node in the causal graph.
However, this is not necessarily true, for example when what is actually observed is raw, high-dimensional sensory data. Consider natural images: we do not really know in advance what kinds of objects are present, while the causal interactions should probably be modeled on the level of the objects.
Therefore, in order to understand what kind of causal mechanism is generating such low-level sensory data, 
we also need to extract the ``high-level'' causal variables constructing the causal graph by performing {\it representation learning} \citep{Bengio2013} simultaneously.

Nonlinear representation learning has its own problems of identifiability. Recent work has solved the identifiability problem in the context of Nonlinear Independent Component Analysis (NICA) by assuming temporal structure or an additional (conditioning, possibly unobservable) auxiliary variable \citep{Hyvarinen2016, Halva2021, Sprekeler2014}. However, if the components are mutually independent, it seems impossible to model causal connections between them, and thus such theory is not directly applicable to this case. Thus, we need to go beyond independent components \citep{Zhang10UAI,Khemakhem2020beem} and build an explicit model of the dependencies resulting from their causal interactions.

Such simultaneous learning of the causal variables (representation learning) and their causal graph (causal discovery) has been a focus of intense attention recently, resulting is what is called as Causal Representation Learning (CRL) \citep{Schlkopf2021}.
Since both of the two separate problems combined here are known to be ill-posed, CRL seems to be 
even more severely ill-posed, and much less is known on identifiability of CRL.
Yet, finding identifiability conditions is crucial for its interpretability, applicability, and reproducibility.
So far very limited frameworks were proposed, and many of them are based on heavy assumptions on the causal mechanisms,
such as supervision or intervention on latent variables or causal graphs \citep{Brehmer2022, Shen2022, Yang2021}, or temporal causality (dynamics) \citep{Lachapelle2022, Lippe2022}.

Here, we propose a new model for CRL based on a novel approach assuming that the observed variables follow a certain \textit{grouping} structure known a priori, as illustrated in Fig.~\ref{fig:model}c. Such grouping is common and naturally appears in many practical situations. For example, the variables could be grouped based on which measurement sensor they come from in multimodal data; or which time point in causal dynamics, or geographical location in sensor networks they are measured at.  We further assume that the causal interactions are \textit{pairwise} as in some Markov network models.
Under these  assumptions, we prove identifiability with much weaker, and very different, constraints than previous work. The model in particular is able to consider instantaneous causal relations rather than temporal (Granger) causality, while autoregressive (AR) dynamics are further contained as a special case. Nor does our model  assume any supervision or interventions. 
Our experiments on synthetic data as well as a realistic high-dimensional image and gene regulatory network datasets show that our framework can indeed extract latent causal variables and their causal structure, with better performance than the state-of-the-art baselines.

\section{Related Works}
\label{sec:related}

The general form of the CRL problem can be defined as an
estimation of a set of causally-related latent variables $\s = [s_1, \ldots, s_{D_\mathcal{S}}]^\T \in \R^{D_\mathcal{S}}$, or causal variables for short, together with their causal structure.
We typically assume that the observed data $\x \in \R^{D_\X}$ with $D_\X \geq D_\mathcal{S}$ are obtained via an unknown observational mixing $\f: \R^{D_\mathcal{S}} \rightarrow \R^{D_\X}$ as
\begin{align}
	\x = \f(\s), \label{eq:fg}
\end{align}
where the latent causal variables $s_1, \ldots, s_{D_\mathcal{S}}$ are \textit{not} mutually independent but follow a causal model to be specified. 
Typically, the causal model would be a Structural Equation Model (SEM) \citep{Shen2022, Yang2021} which has a well-defined causal semantics,
or something simpler such as a Bayesian network (BN) as a kind of proxy. 

In this work we only consider the case of independent and identically distributed (i.i.d.)~sampling, which means the different observations of $\x$ are independent of each other and there is no time structure. This obviously implies the causal relations must also be instantaneous. Note that estimating instantaneous causality is more challenging compared to temporal causality \citep{Lachapelle2022, Li2020, Lippe2022, Lippe2023, Yao2022b, Yao2022a}, 
since we do not have prior knowledge about the causal direction or causal ordering between variables (from past to future; Fig.~\ref{fig:model}b) in instantaneous causality.


CRL can be seen as a generalization of NICA and causal discovery, both of which are known to be ill-posed without any specific assumptions.
NICA can actually be seen as a special case of CRL, where the latent variables follow the degenerate causal graph in the sense of  not having any causal relations.
Recent studies have shown that NICA can be given identifiability by assuming temporal structure \citep{Hyvarinen2019, Halva2021, Klindt2021, Sprekeler2014} instead of i.i.d.\ sampling.
On the other hand, causal discovery is also a special case of CRL, where the causal variables are observed directly. Many studies have shown that some kind of  asymmetricity of the statistical causal model enables the identifiability \citep{Hoyer2008, Park2015, Peters2014, Shimizu2006, Shimizu2011, Zhang2009}.

The CRL model thus violates the important assumptions of the both problems (mutual independence, non-i.i.d., and direct observability). 
The research goal of CRL is thus to find the practical conditions for the identifiability of the model. 
Some studies have shown the identifiability in the instantaneous causality case, 
but they require supervision or intervention on the causal variables \citep{Ahuja2022a, Ahuja2022b, Brehmer2022, Shen2022, Yang2021} (Fig.~\ref{fig:model}a),
or access to some latent information such as mixture oracles \citep{Kivva2021}, which might be too restrictive in actual applications.
Recently some CRL studies proposed to use a grouping of variables instead of supervisions \citep{Daunhawer2023, Lyu2022, Morioka2023, Sturma2023, Yao2023}, similarly to this study.
Most of them, except for \citet{Morioka2023}, especially focused on the intersections between groups.
\citet{Sturma2023} showed identifiability of the intersection of the latent variables across all groups,
based on linearity of the causal and observational models.
\citet{Daunhawer2023, Lyu2022, Yao2023} considered more general causal and observational models,
though the identifiability is limited to up to intersection-wise transformations.
%
%
\citet{Morioka2023} and our study instead focus on the group-specific causal variables not shared across groups.
\citet{Morioka2023} assumed a component-wise dependency as in NICA, and can be seen as a very special case of this study (see Section~\ref{sec:discussion}).
A more detailed discussion about the related works are given in Supplementary Material~\ref{sec:related}.

\section{Model Definition}
\label{sec:model}

Our basic idea is to impose some constraints on the observational model $\f$, based on {\it grouping of the observed variables}, together with some Markov-like (pairwise) constraints on the causal interactions between the groups (Fig.~\ref{fig:model}c). Next we first define our observational model and then the causal model.

\begin{figure}[t]
 \centering
 \includegraphics[width=\linewidth]{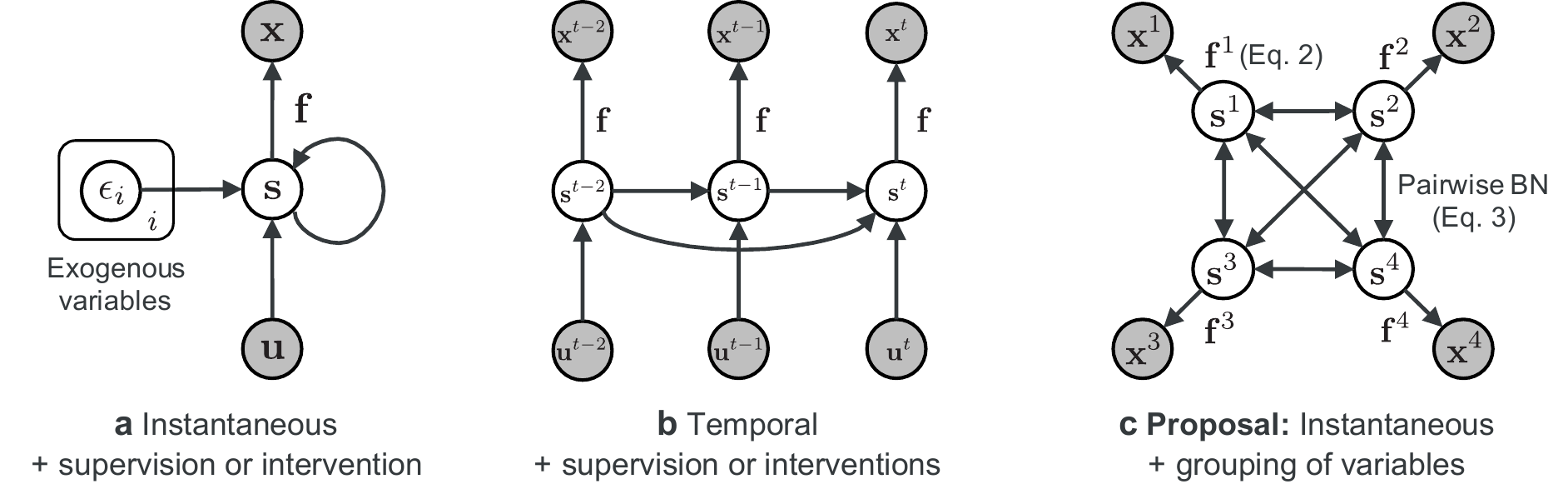}
 \caption{Comparison of the graphical models of major CRL frameworks whose goal is to estimate latent causal variables $\s$ from the low-level observations $\x$, usually with supervision or intervention $\mathbf u$.
 Our proposal in (c) is based on the grouping of variables (Eq.~\ref{eq:f}; $M=4$ groups here) and the causal model based on a pairwise BN (Eq.~\ref{eq:ps}), and does not require any supervision or intervention, greatly generalizing the existing models.
   }
 \label{fig:model}
\end{figure}

\paragraph{Observation Model}
As the original approach in our model, we assume that the observational mixing can be separated into $M > 1$ non-overlapping groups,
as found in many practical cases. After appropriate permutations of the elements of $\s$ and $\x$ without loss of generality, we assume that the observation model Eq.~\ref{eq:fg} can be expressed as
\begin{align}
	\x &= \begin{bmatrix} \x^1, & \ldots, &  \x^M \end{bmatrix} = \f(\begin{bmatrix} \s^1, & \ldots, &  \s^M \end{bmatrix}) \nonumber \\
	&= \begin{bmatrix} \f^1(\s^1), & \ldots, & \f^M(\s^M) \end{bmatrix}, \label{eq:f}
\end{align}
where $\x^m = [x_1^m, \ldots, x_{\dxm}^m]^\T \in \R^{\dxm}$ and $\s^m = [s_1^m, \ldots, s_{\dsm}^m]^\T \in \R^{\dsm}$ are the $m$-th group of the observational and latent variables respectively. 
Each element of the latent and the observational variables belongs to only one of the groups with index in $\M = \{1, \ldots, M\}$,
which means that the $m$-th observational group $\x^m$
is generated only as a function of $\s^m$, without any observational contaminations from the other groups: i.e., $\x^m = \f^m(\s^m)$, $D_\mathcal{S} = \sum_{m=1}^M \dsm$, and $D_\mathcal{X} = \sum_{m=1}^M \dxm$.
The number of variables in a group can be different across groups.
We usually denote the group index as a superscript, which should not be confused with an exponent;  the element index as denoted by a subscript. 
Note that when $M = 1$ this model corresponds to the general CRL (Eq.~\ref{eq:fg}), and when $M = D_\mathcal{S}$ this model simply leads to the ordinary causal discovery problem without observational mixing. 

Such grouping structure is not anything unrealistic, and can be seen in many practical situation, as described in the following illustrative examples.

\paragraph{Illustrative Example~1: Causally Related Sensor Measurements} 
The most intuitive example would be where $m$ is a sensor index. 
Data  is then obtained from a set of $M$ sensors, each measuring different but causally-related multidimensional physical quantities $\x^{m(n)} \in \R^{d_\X^m}$ for each sample $n$. For example, in single-cell multiomics data \citep{Burkhardt2022}, each cell ($n$) could be measured to give chromatin accessibility (DNA) as $\x^{1(n)}$, 
gene expressions (RNA) as $\x^{2(n)}$, and protein levels as $\x^{3(n)}$. These are all multi-dimensional quantities representing causally-interacting latent high-level features $\{\s^{m}\}_m$.
There exist many other possible applications consistent with this observational model; e.g., multimodal biomedical data \citep{Acosta2022}, simultaneous measurements of brain and behavior \citep{Hebart2023}, climate monitoring sensor networks \citep{Longman2018}, and so on, where $m$ corresponds to sensor modalities or locations (see Section~\ref{sec:discussion} for more details).

\paragraph{Illustrative Example~2: Causal Dynamics} 
Although we focus on independent data samples rather than dynamics,
our model can also implement dynamics with dependency across {\it time} by simply defining the group-index $m$ as time-index $t$ (see Figs.~\ref{fig:model}b and c).
We then obtain low-level observations $\x^t$ (such as images) from high-level latent causal process composed of multidimensional variable $\s^t$ through time-dependent mixing model $\x^t=\f^t(\s^t)$ for each time point $t$.
In this case, our model gives a generic form of a time series model,
which is actually more general compared to some existing studies of CRL based on dynamics \citep{Lachapelle2022, Lippe2022, Yao2022b, Yao2022a}
 in the sense that the mixing function $\f^t$ changes as a function of time $t$, which can happen in many practical situations (such as changes of the camera angle capturing the images).
In this case,  we assume we observe the same time-series many times, i.e. we have $\x^{t(n)}$ where $n$ is the index of the time series realization (e.g., capturing images with multiple sequences $n$ with the same transition of camera angles across $t$ every time).

\paragraph{Causally Structured Latent Variable}

Next we model the causal structure of the latent variables based on a BN, which is in particular {\it pairwise}, 
in the spirit of 
pairwise Markov random fields.
Denote by $\phi(\cdot, \cdot)$ a potential function representing causal relations between two variables, which is the same for all variable pairs.
Further, denote by $\bar{\phi}^m$ group-wise potential functions representing causal relations \textit{inside} a group, i.e.\ between the elements of $\s^m$, which are not restricted in any way.
We assume that the joint distribution is factorized as 
\begin{align}
	p(\s) \propto &\left[\prod_{m \in \M} \exp \left(\bar{\phi}^m(\s^m) \right) \right] \label{eq:ps} \\
	&\times \prod_{m \neq m'} \prod_{(a, b) \in \Vsm \times \Vsmp} \exp \left( \lambda_{ab}^{mm'} \phi(s_a^m, s_b^{m'}) \right),  \nonumber
\end{align}
where we denote the set of indices of the latent variables belonging to the $m$-th group by $\Vsm$ ($|\Vsm| = \dsm$).
The idea is to have a model of dependencies between variables which is so general that the estimation of the representation is not biased towards independent components.
The variables in one group $\s^m$ can causally affect all variables on the other groups $m' \neq m$, and also in the same group $m$, in a complex manner, thus breaking any independence of variables.
The coefficient $\lambda_{ab}^{mm'} \in \R$ modulates the strength of the causal relation $s_a^m \rightarrow s_b^{m'}$,
which is constantly zero if $s_a^m$ is not a direct causal parent of $s_b^{m'}$.
The sets of the coefficients $\m L = \{\m L^{mm'}\}_{(m, m')}, \m L^{mm'} = (\lambda_{ab}^{mm'})_{a \in \Vsm, b \in \Vsmp}$ can be interpreted as (inter-group) weighted adjacency matrices.


This model is rather general, and includes a form of exponential family BNs (see Supplementary Material~\ref{sec:ps_exp}).
This allows for assigning causal semantics to the model by considering its equivalence to some SEMs,
such as causal additive models (CAMs; \citet{Buhlmann2014}),
which are more general than linear SEMs on which some CRL frameworks are based \citep{Shen2022, Sturma2023, Yang2021}.
Our model is not even restricted to Gaussian BNs as in CAMs.
%
Although \citet{Morioka2023} used a similar causal model,
their model can be seen as a very special case of ours (see Section~\ref{sec:discussion});
especially, Eq.~\ref{eq:ps} does not require any mutual independence across variables.

Note that Eq.~\ref{eq:ps} just represents the factorization of the joint distribution and does not incorporate any causal directional assumptions between variables.
We thus need some additional assumptions for the identifiability of this factorization model as a {\it causal} model as shown in Theorems below, similarly to causal discovery based on BN.

\paragraph{Illustrative Example~1: Causally Related Sensor Measurements} 

In the single-cell multiomics example, the model says there are causal relations between (and within) the groups, which can be consistent with what is known as the {\it central dogma} in molecular biology; DNA ($\x^1$) $\rightarrow$ RNA ($\x^2$) $\rightarrow$ Protein ($\x^3$). Our model considers that they are interacting on some high-level latent space $\{\s^{m}\}_m$, such as something related to transcription factors.

\paragraph{Illustrative Example~2: Causal Dynamics} 

In the temporal dynamics example above, it is natural to have causal relations across the time-index $t$ (which is the same as the group-index $m$).
Our model extends the previous models in the sense that
any pairs of latent variables can be causally related across time (no sparseness is required unlike \citet{Lachapelle2022}).
It is also worth mentioning that temporal causality from the past to the present is a special case in our model 
since our model (Eq.~\ref{eq:ps}) does not restrict the causal directions between the groups.

\section{Identifiability of Representation Learning}
\label{sec:id}

Based on the grouping assumption of the observational model (Eq.~\ref{eq:f}), together with the latent variable model given above (Eq.~\ref{eq:ps}),
we can prove new identifiability results of the CRL model. In this section, we first consider identifiability of the latent variables.
We assume that each mixing function $\f^m$ is invertible and a $C^2$ diffeomorphism
(thus $\dsm = \dxm$; we later discuss the case $\dsm < \dxm$).
 Apart from that, we do not assume any parametric form for each $\f^m$.
 We consider the situation where the support of the distribution of each variable is connected (i.e.~an interval),
and without loss of generality, the same across all variables, denoted as $\bar{\mathcal{S}}$.
We denote $\phi^{112}(x, y) = \frac{\partial^3}{\partial x^2 \partial y} \phi(x, y)$, $\phi^{122}(x, y) = \frac{\partial^3}{\partial x \partial y^2} \phi(x, y)$,
and $\phi^{12}(x, y) = \frac{\partial^2}{\partial x \partial y} \phi(x, y)$.
Those functions are said to be {\it uniformly dependent} (Definition~\ref{def:dependency} in Supplementary Material~\ref{sec:ps_exp}) if the set of zeros of the function does not contain any open subset in the support of the input distribution.
(\textbf{Neighbor}) For a variable $s_a^m$, we call $s_b^{m'}$ in some {\it other} group a {\it neighbor} if either or both of the adjacency coefficients $\lambda_{ab}^{mm'}$ and $\lambda_{ba}^{m'm}$ are non-zero.
The identifiability condition is then given in the following Theorem, proven in Supplementary Material~\ref{sec:id_proof};
\begin{Theorem} \label{thm:id}
Assume the generative model given by Eqs.~\ref{eq:f} and \ref{eq:ps}, and also the following:
\begin{enumerate}[label=A\arabic*] \vspace*{-2mm}
\item (Nondegeneracy of the graph) \label{AL} 
For any group $m$ in $\M$ (or in a subset of $\M$, that we call ``the groups of interest"),
each variable has a (at least one) neighbor in some other group,
and the collection of inter-group adjacency matrices $\bar{\m L}^m$ given below has full row-rank after removing all-zero rows:
\begin{align}
	\bar{\m L}^m =\begin{bmatrix} \m L^{m1}, & \ldots, & \m L^{mM} \\ (\m L^{1m})^\T, & \ldots, & (\m L^{Mm})^\T \end{bmatrix}, \label{eq:Lg}
\end{align} 
\item (Causal function) $\phi^{12}$, $\phi^{112}$, and $\phi^{122}$ have uniform dependency,
and for any open subset $B$ of $\bar{\mathcal{S}}$, there exist some $z_1 \neq z_2 \in \bar{\mathcal{S}}$
such that any of the following conditions does not hold for $\phi^{12}$:
$\phi^{12}(s, z_1) = c_1 \phi^{12}(s, z_2)$, $\phi^{12}(z_1, s) = c_2 \phi^{12}(z_2, s)$, and $\phi^{12}(s, z_1) = c_3 \phi^{12}(z_1, s)$ for all $s \in B$
with some constants $c_1, c_2, c_3 \in \R$. \label{Aphi}
\end{enumerate}\vspace*{-2mm}
Then, for all groups $m$ in $\M$ (or in the groups of interests), $\s^m$ can be recovered up to permutation and variable-wise invertible transformations from the distribution of $\x$.
\end{Theorem}


The Assumption~\ref{AL} is rather intuitive, and requires the variables (rows) in each group $m$ to have distinctive sets of (or causal strengths to) neighbors (columns), for both causal directions (upper and lower halves), as expressed by the full row-rank condition. This indicates nondegeneracy of the graph, which is known to be crucial in CRL \citep{Kivva2021, Morioka2023}.
Note that \ref{AL} does not require the causal graph to be {\it directed} as in Theorem~\ref{thm:cd3} given below, though requires it to be {\it asymmetric}.
Note also that the variables need to have neighbors only on some of the other groups but not on all of them; e.g.\ in practice, groups somehow ``near-by'' in space or time.
This condition can be evaluated separately for each group $m$;
we cannot identify the latent variables of the groups not satisfying \ref{AL}, while they do not affect the identifiability of the other groups.

The Assumption~\ref{Aphi} requires the (cross-derivative of) $\phi$ to be non-factorizable (the first two equations) and asymmetric (the last equation).
The asymmetricity is a well-known requirement for causal discovery \citep{Hoyer2008, Peters2014}, which excludes linear Gaussian SEMs, and thus reasonable for CRL as well.
The non-factorizability cannot be satisfied when the cross-derivative of $\phi$ is factorized into variable-wise scalar functions,
which in turn requires sufficiently complex dependency of variables.
In the exponential family BN characterization of Eq.~\ref{eq:ps}, the models with model order more than 1 can satisfy this (Supplementary Material~\ref{sec:ps_exp}).
We can give an alternative condition not requiring this non-factorizability, by some additional constraints on the causal graph (Proposition~\ref{thm:id_alt} in Supplementary Material~\ref{sec:id_alt}),
which then allows exponential family BNs with order 1, including CAMs.
This resembles the requirement for some representation learning (e.g., non-quasi-Gaussianity in \citet{Hyvarinen2017}), and is thus reasonable for CRL as well.



\section{Representation Learning Algorithm}
\label{sec:crl}

We now propose a self-supervised estimation framework called Grouped Causal Representation Learning (G-CaRL). 
Again, we start by learning the representation, i.e.\ learning to invert the mixing functions $\{\f^m\}$.
To this end, we propose a new contrastive learning method where the pretext task is to discriminate (classify) the following two datasets obtained from the same observations:
\begin{align}
       &\x^{(n)} = \begin{bmatrix} \x^{1 (n)}, & \ldots, & \x^{M (n)} \end{bmatrix} \nonumber \\
       & \text{vs.} \quad  \x^{(n_\ast)} = \begin{bmatrix} \x^{1 (n_{\ast}^1)}, & \ldots, & \x^{M (n_{\ast}^M)} \end{bmatrix}   \label{eq:xtilde}
\end{align}
where $n$ indicates the sample index, 
while $n_{\ast}^m$ is a shuffled index, generated in practice by randomly selecting a sample index separately for each group $m$ (note that different groups have different sample indices in $\x^{(n_\ast)}$).
We then learn a nonlinear logistic regression (LR) system which discriminates the two classes,
using a cross-entropy loss with a specific form of the regression function with
\begin{align}
        &r(\x) = c + \sum_{m \in \M} \bar{\psi}^m(\h^m(\x^{m})) \label{eq:crl_r} \\ 
        &+ \sum_{m \neq m'} \sum_{(a, b) \in \Vsm \times \Vsmp} w_{ab}^{mm'} \psi(h_a^m(\x^{m}), h_b^{m'}(\x^{m'})), \nonumber
\end{align}
where $\h^m:  \R^{\dxm} \rightarrow \R^{\dsm}$ is a group-wise (nonlinear) feature extractor,
$h_a^m$ is the $a$-th element of $\h^m$,
$\bar{\psi}^m: \R^{\dsm} \rightarrow \R$ and $\psi: \R^2 \rightarrow \R$ are scalar-valued nonlinear functions,
and $w_{ab}^{mm'}$ and $c \in \R$ are weight and bias parameters.
Importantly, this regression function is designed to be consistent with the causal model defined in Section~\ref{sec:model} (Eq.~\ref{eq:ps}):
the feature extractors $\h^m$ and the weight parameters $w_{ab}^{mm'}$ correspond to the causal variables $\s^m$ and the
adjacency coefficients $\lambda_{ab}^{mm'}$ in the causal model (Eq.~\ref{eq:ps}), respectively.
This indicates that learning those parameters in a data-driven manner should lead to CRL automatically, as justified in the following Theorems.
The observational grouping indices are assumed to be given in advance, 
while we only need the information of the size of groups $\dsm$ for the latent variables.
The nonlinear functions are typically learned as neural networks with universal approximation capacity \citep{Hornik1989}.
Any optimization method can be used to minimize the loss (see Supplementary Algorithm~\ref{alg:gcarl} for example), though the theorem below assumes it gives the optimal solution without getting stuck in a local optimum.
After the convergence, we obtain the following consistency Theorem, proven in Supplementary Material~\ref{sec:crl_proof};
\begin{Theorem} \label{thm:crl}
Assume the same as those in Theorem~\ref{thm:id}, and: \vspace*{-2mm}
\begin{enumerate}[label=B\arabic*]
\item (Learning) We train a nonlinear LR system (Eq.~\ref{eq:crl_r}) with universal approximation capability to discriminate two datasets $\x^{(n)}$ and $\x^{(n_\ast)}$ (Eq.~\ref{eq:xtilde}).\label{CAlearn}
\item ($\h$) The functions $\h^{m}$ are $C^2$ diffeomorphisms.\label{Ah}
\end{enumerate}\vspace*{-2mm}
Then, for all groups $m$ in $\M$ (or in the groups of interests in \ref{AL}),
in the limit of infinite samples $n$, 
the function $\h^m(\x^m)$ gives the latent variables on the $m$-th group $\s^m$, up to permutation and variable-wise invertible transformations.
\end{Theorem}
Interestingly, in spite of the lack of clear relevance of this pretext task to CRL at first sight,
this theorem actually shows that learning the correct representation is achieved by learning the functions $\h^m(\cdot)$ through the optimization of the regression function Eq.~\ref{eq:crl_r}.
This Theorem is basically based on the well-known properties of the logistic regression \citep{Gutmann2012}.
%
Intuitively, the group-wise shuffling applied to $\x^{(n_\ast)}$ (Eq.~\ref{eq:xtilde}) breaks the causal relations between groups,
which means for the LR to discriminate the two datasets in Eq.~\ref{eq:xtilde} properly, 
it needs to capture the inter-group causal relations in the latent space, by disentangling the observational mixing. 
Thanks to the compatibility of the factorization assumptions in the generative (Eq.~\ref{eq:ps}) and in the regression model (Eq.~\ref{eq:crl_r}) as mentioned above, the optimal model is achievable only when $\h^m$ essentially gives the inverse model of $\f^m$, which automatically leads to CRL.

\section{Identifiability of Causal Discovery}
\label{sec:cd}

Next, we consider how to learn the causal graph $\m L$.
%
%
In our model, this can be achieved by estimating the (weighted) adjacency coefficients $\lambda_{ab}^{mm'}$ in Eq.~\ref{eq:ps} from the observations in a data-driven manner, similarly to many causal discovery frameworks based on BN \citep{Choi2020, Park2019}. 
We can actually achieve this simultaneously with the representation learning by using the self-supervised algorithm in Section~\ref{sec:crl},
where the adjacency coefficients $\{\lambda_{ab}^{mm'}\}$ are learned as the weight parameters $\{w_{ab}^{mm'}\}$ jointly with the inverse models $\h^m$ in the regression function (Eq.~\ref{eq:crl_r}).

The identifiability of  $\m L$ requires additional assumptions on $\m L$ and $\phi$,
given in the following Theorem, proven in Supplementary Material~\ref{sec:cd_proof}.
We first give the definitions of some terms used in the Theorem.
(\textbf{Directed}) We call a causal relation between two variables $s_a^m$ and $s_b^{m'}$ {\it directed} if only one of $\lambda_{ab}^{mm'}$ and $\lambda_{ba}^{m'm}$ has a non-zero value.
(\textbf{Co-parent and co-child}) For a variable $s_a^m$, we call $s_{a'}^m$ in the {\it same} group a {\it co-parent} (respectively {\it co-child}) of $s_a^m$ if it shares at least one child (respectively parent) $s_b^{m_\ast}$ on some {\it other} group ($m_\ast \in \M \setminus m$) with $s_a^m$.
The $s_b^{m_\ast}$ can be arbitrarily selected for each co-parent and co-child.
%
\begin{Theorem} \label{thm:cd3}
Assume the same as those in Theorem~\ref{thm:id}, and:
\begin{enumerate}[label=C\arabic*] \vspace*{-2mm}
\item (Causal graph) The inter-group causal relations of variables are all directed, and for every group-pair $(m, m')$ in the groups of interest, 
all variables in a group $m$ (and $m'$) have both a co-parent and a co-child in the same group.
In addition, any variables in the group $m$ (and $m'$) can be reached from any other variables in the same group 
by moving from a variable to one of its co-parents, possibly by multiple hops
(and similarly for co-children). \label{An3}
\item (Asymmetricity) 
There is no open subset $B$ of $\bar{\mathcal{S}}$ such that for all $x \neq y \in B$, it holds
\begin{align}
	\phi^{12}(x, y)  = c \phi^{12}(y, x)
\end{align}
with some constant $c \in \R$. \label{Aasym3}
\end{enumerate}
Then, for all group-pairs $(m, m')$ satisfying \ref{AL} and \ref{An3}, $(\m L^{mm'}, \m L^{m'm})$ are identifiable up to permutation of variables, linear scaling, and matrix transpose.
\end{Theorem}
This theorem shows that we can identify the causal graphs $\m L$ from the data distribution up to linear scaling and matrix transpose.
This also indicates that the graph can be estimated indeed consistently by the learning algorithm in Section~\ref{sec:crl} as claimed above (we omit the proof of the consistency since it can be easily shown by following those of Theorems~\ref{thm:crl} and \ref{thm:cd3}).
The indeterminacy of matrix transpose comes from the lack of specification of the functional form of $\phi$;
a transposition of the adjacency matrix could be compensated by switching the order of the two arguments of $\phi$,
which can be resolved by some prior information on $\phi$.

The function $\phi$ is required to be asymmetric (\ref{Aasym3}), 
which is a well-known requirement for causal discovery as mentioned in Section~\ref{sec:id},
though \ref{Aasym3} is a bit stronger than that in \ref{Aphi}.

The causal graph assumption (\ref{An3}) is a special condition related to the grouping of variables 
(we can also consider alternative condition which requires only either co-parent or co-child for each variable; Supplementary Material~\ref{sec:cd_alt}).
This would be fulfilled as long as the connections between groups are not too sparse (See Supplementary Material~\ref{sec:c1} for some illustrative discussion).
One typical example would be fully-connected causal effects from one group to another;
e.g., fully-connected causality DNA $\rightarrow$ RNA and RNA $\rightarrow$ Protein on their latent space in Illustrative Example~1,
and fully-connected temporal causality to some (or all) subsequent time-point(s) in Illustrative Example~2,
since in those cases all other variables in the same group are either (or both) co-parents or co-children.
Of course they do not need to be {\it fully}-connected in practice,
and denseness is not a necessary condition in theory.
The graph is required to be directed though not necessarily {\it acyclic}.
\section{Experiments}
\label{sec:exp}

To validate the effectiveness of our framework, we compare it to several baselines in two simulation settings and two more realistic scenarios
(see Supplementary Material~\ref{sec:app_exp} for the details; the implementation of G-CaRL is available at \url{https://github.com/hmorioka/GCaRL}).
The baselines only include {\it unsupervised} frameworks with {\it instantaneous} (causal) interactions, since our experimental setting does not include supervision, intervention, nor temporal causality.
Specifically, we compared with three CRL frameworks MVCRL \citep{Yao2023}, CausalVAE (\citet{Yang2021} in unsupervised setting), and CCL \citep{Morioka2023}, 
and three representation learning (RL) frameworks MFCVAE 
(\citet{Falck2021, Kivva2022}),
VaDE (\citet{Jiang2017, Kivva2022, Willetts2022}), and $\beta$-VAE \citep{Higgins2017}.
See Supplementary Material~\ref{sec:app_baseline} for the details.
\citet{Daunhawer2023, Lyu2022} are not applicable since they are limited to two-group settings. We also applied \citet{Kivva2021} but it failed due to the difficulty of the estimation of the mixture model in our data.
For a fair comparison, we used group-wise structures for the encoders of those baselines similarly to G-CaRL.
For baselines which do not estimate the (part of) causal graph, we additionally applied a causal discovery framework to the estimated latent variables as post-processing,
from a wide variety of choices so as to maximize the performances (see Supplementary Fig.~\ref{fig:h_cd}).
We only evaluated the {\it inter-group} causal graphs, since only those are identifiable in our model (Theorem~\ref{thm:cd3}).

\subsection{Simulation~1: DAG}
\label{sec:sim1}

We first examined the performance with latent DAG models (Supplementary Fig.~\ref{fig:w_sample}a shows some examples).
The number of groups ($M$) was fixed to 3, the number of variables was 10 for each group ($\dsm = 10, D_\mathcal{S} = 30$).
The latent variables are observed through nonlinear mixings randomly generated as a multilayer-perceptron (MLP) for each group.

The latent variables and the causal graphs were reconstructed reasonably well by G-CaRL, with much higher performances than the baselines (Fig.~\ref{fig:sim}a).
The baselines did not work well basically because of the lack of representation capability (CCL),
lack of supervision in our setting (CausalVAE), or lack of explicit considerations of the latent causality (others).
The worst performance of CCL indicates the inadequacy of assuming mutual independence among some variables in this dataset.

Supplementary Fig.~\ref{fig:sim1_LMD}a shows how the complexity of the mixing model ($L$), the number of variables $D_\mathcal{S}$, groups $M$, and sample size $n$ affect the performances;
a higher $L$, $D_\mathcal{S}$, and $M$ make learning more difficult, while a larger $n$ makes it possible to achieve higher performances,
as expected.

\begin{figure*}[t]
 \centering
 \includegraphics[width=\linewidth]{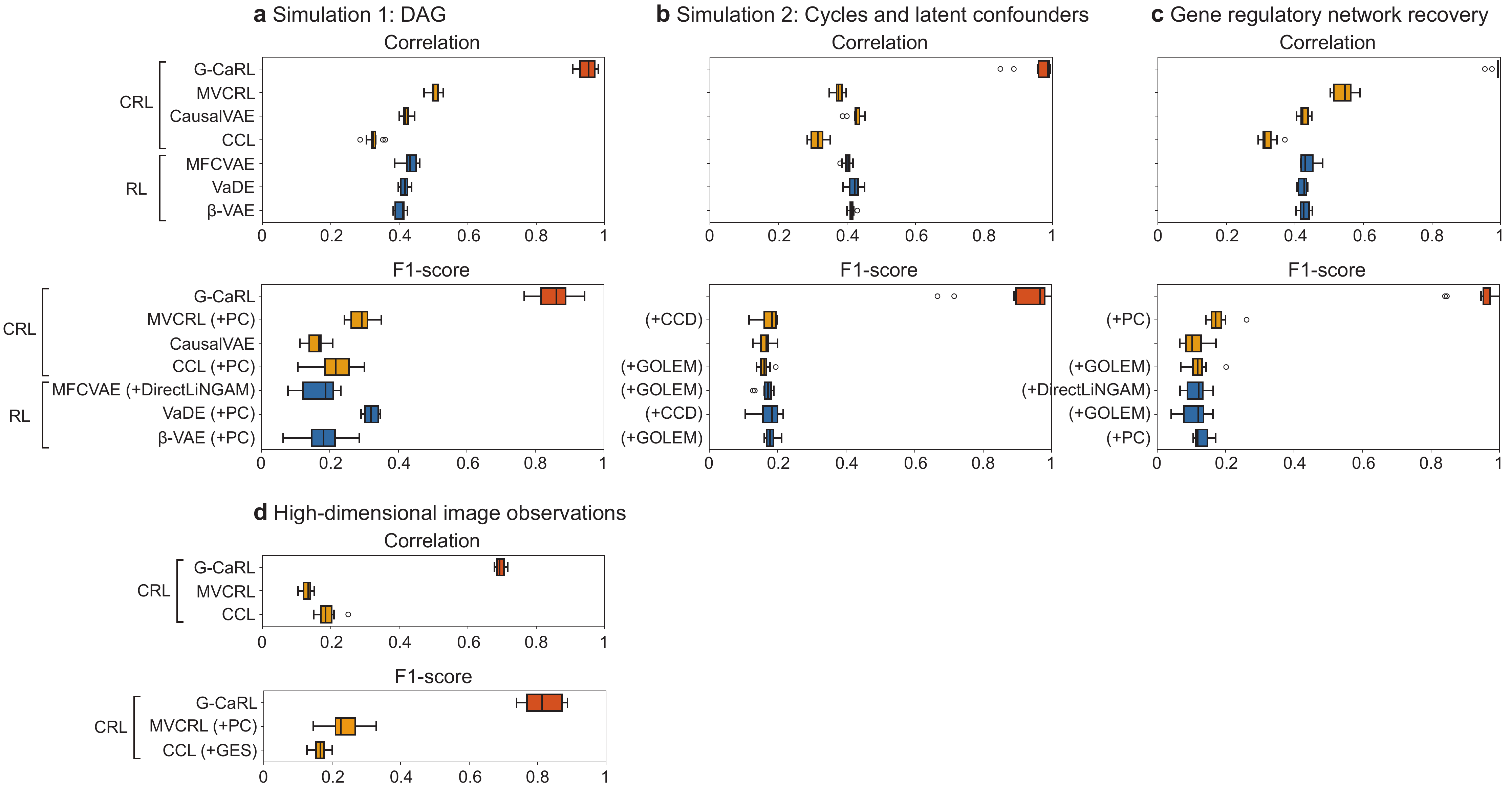}
 \caption{Comparison of CRL performance by the proposed G-CaRL and the baselines. 
   The performances are measured by correlation for the latent variables, and 
   by F1-score for the causal graphs, excluding the intra-group sub-graphs.
 The parentheses after the names of some (C)RL frameworks indicate the causal discovery frameworks additionally applied as post-processing.
  }
 \label{fig:sim}
\end{figure*}

\subsection{Simulation~2: Cyclic Graphs with Latent Confounders}
\label{sec:sim2}

To show the robustness of G-CaRL on more complex causal models,
we next examined the performances with directed cycles and latent confounders with the same number as the observable variables (Fig.~\ref{fig:sim}b; see Supplementary Fig.~\ref{fig:w_sample}b for the difficulty of this setting).
G-CaRL showed reasonably good performance even in this difficult condition, though it requires larger number of samples than Simulation~1 (Supplementary Fig.~\ref{fig:sim1_LMD}b).
Note that this causal model is difficult even for the conventional causal discovery frameworks directly applied to the true latent variables (Supplementary Fig.~\ref{fig:h_cd}).
This result shows the effectiveness of the causal model of G-CaRL against the existence of causal cycles and latent confounders.
Supplementary Fig.~\ref{fig:sim1_LMD}b shows the effects of the model complexity to the estimation performances, and they show a trend similar to Simulation~1.

\subsection{Recovery of Gene Regulatory Network}
\label{sec:grn}

We also evaluated G-CaRL on a more realistic causal data for showing the robustness against misspecification of the causal model.
We used synthetic single-cell gene expression data generated by SERGIO \citep{Dibaeinia2020},
where each gene expression is governed by a stochastic differential equation (SDE) derived from a chemical Langevin equation.
Due to such generative process, the true causal relations cannot really be represented by our causal model (Eq.~\ref{eq:ps}).
The causal graph was designed to be a DAG (as required of SERGIO) similarly to Simulation~1, but with latent confounders similarly to Simulation~2 (Supplementary Fig.~\ref{fig:w_sample}c shows  examples). 
We used the same setting for the observational mixings as those in the simulations above.
G-CaRL showed the best performances among the baselines (Fig.~\ref{fig:sim}c), 
which suggests the robustness of G-CaRL against misspecification of the causal model, and the good applicability of G-CaRL to real datasets. 

\subsection{High-Dimensional Image Observations}
\label{sec:image}

We also evaluated G-CaRL on a more realistic observational model,
by using a high-dimensional image dataset (3DIdent; \citet{Zimmermann2021}).
Images of a tea cup were generated based on ten latent factors for each group, which were causally interacting within/across groups with directed cycles under influence of latent confounders as in Simulation~2 (see Supplementary Fig.~\ref{fig:sim2_image} for the difficulty of this setting).
We compared the performance with two baselines MVCRL and CCL, which do not require learning a decoder,
since learning both encoder and decoder should be unstable in this high-dimensional setting.

The estimation performances were reasonably good even in such high-dimensional observations with a complex causal model (Fig.~\ref{fig:sim}d).
This suggests the good applicability of G-CaRL to high-dimensional real data.


%
%
%
%
%

\section{Discussion}
\label{sec:discussion}

Our proposal extends the existing CRL models in many aspects;
1) the framework is unsupervised (only requires {\it grouping of variables} independent of the samples) rather than supervised \citep{Brehmer2022, Shen2022, Yang2021}, 
2) we consider instantaneous causality rather than temporal \citep{Lachapelle2022, Lippe2022, Yao2022b, Yao2022a}, 
while the temporal causality is also contained as a special case,
3) the observational mixing can be group-(or time-)dependent rather than invariant \citep{Lachapelle2022, Lippe2022, Morioka2023, Yao2022b, Yao2022a}, 
4) the latent variables can be nonlinearly causally related 
rather than linearly \citep{Shen2022, Yang2021}, nor is sparseness \citep{Lachapelle2022} necessary,
and 5) the causal graph can be cyclic, which is even more general than commonly used models for simple causal discovery.

Meanwhile, our framework also has some connections to existing CRL;
it can be understood that our theorems are virtually using the {\it variables on the other groups} as auxiliary variables in the context of (weakly-)supervised CRL.
Importantly, instead of requiring {\it additional} auxiliary variables (supervision) for each sample as in those existing CRL, this study nicely utilized the grouping structure to obtain them automatically.
Our estimation framework based on the group-wise shuffling somewhat resembles 
some multimodal CRL based on contrastive learning (e.g., \citet{Daunhawer2023}). 
This study nicely extends them to more than two-group settings, and gave theoretically guarantee of the identifiability of the 
(group-specific) latent variables up to variable-wise transformations, which is quite new.

This study greatly generalizes the recently proposed CRL framework CCL \citep{Morioka2023} in many aspects;
CCL assumes that 1)~the mixing functions $\f^m$ are the same for all groups (called {\it nodes} in CCL) $m$ rather than group-specific as in ours (Eq.~\ref{eq:f}), 2)~the causal relations are component-wise, similarly to NICA
(in other words, the adjacency coefficients $\lambda_{ab}^{mm'}$ in Eq.~\ref{eq:ps} can have non-zero values only when $a = b$, rather than all pairs of $(a, b)$ as in ours), and 3)~the causal graph is a forest, while ours can be more general and even cyclic.
These indicate much higher generality and applicability of our model.
Such generality requires our new estimation framework G-CaRL, which is very different from CCL,
though both of them can be categorized as self-supervised (contrastive) learning;
CCL used group-paired data for taking contrast, while our G-CaRL used group-wise-shuffled data (Eq.~\ref{eq:xtilde}).


Our model can be seen as a generalization of NICA, 
in the sense that 1) it does not require mutual independence of variables,
and further, 2) the mixing function is time/group-dependent, which is completely new in NICA.
Independently Modulated Component Analysis \citep{Khemakhem2020beem} was recently proposed as an extension of NICA to allow dependency across variables, but it requires an auxiliary variable unlike our framework.

There exist many possible applications where our grouping assumptions are applicable: for example,
1)~multimodal measurements, for example consisting of brain activities (group~1), external stimuli (group~2), and behaviors (group~3) of animals or humans \citep{Hebart2023}. Each group has multidimensional observations (multiple brain regions, stimuli, and behaviors) as time series (samples). Our framework would give new insight into how the brain is organizing behaviors based on external factors at an abstract level.
2)~Sensor network for example in climate monitoring \citep{Longman2018}, where each group is a single sensor location, and each sensor is measuring such as temperature, humidity, rainfall, pollutants, etc.~({\it variables}) in the location. Our method would extract some hidden causal relations between sensor locations on high-level feature space.
3)~Medical multimodal data \citep{Acosta2022}, where {\it groups} are different types of medical record obtained from subjects (samples); e.g., genetic factors (group~1), self-report lifestyle (group~2), and clinical diagnosis by doctors (group~3). Our framework would presumably find high-level connections between those groups; e.g., what kinds of (combinations of) genetic factors and/or (combinations of) lifestyles are causing some type of diseases, and so on.
4)~Our model can also consider very general types of temporal causality (see Illustrative Example~2).

Although we considered the case $\dxm = \dsm$ for theoretical convenience, we can extend our framework to $\dxm \geq \dsm$,
as empirically validated in Section~\ref{sec:image}.
One approach is to assume that $\f^m$ is injective and a $C^2$ diffeomorphism between $\mathcal{S}^m$ and $\X^m \in \R^{d_\X^m}$ which is a $C^2$ differentiable manifold,
as in \citet{Halva2021}. Our proof can be adapted for that case, following that study.
Another approach is to assume that $\dxm - \dsm$ variables are not causally related to any other groups (but possibly within each group, implicitly included in $\bar{\phi}^m$ in Eq.~\ref{eq:ps}), in which case our estimation framework would automatically ignore them, which does
not affect the identifiability of the $\dsm$ variables.
We can also consider noisy mixing models via a standard deconvolution argument as in \citet{Khemakhem2020}.


Our Theorems assume grouping of the observational mixing without any observational contamination across groups (i.e., there are only group-specific variables),
in contrast to some other group-based CRL frameworks focusing rather on intersections of groups \citep{Daunhawer2023, Lyu2022, Yao2023}.
Our assumption should be satisfied in many practical applications, as described in Illustrative Examples above. 
In addition, it makes our causal assumptions of the (group-specific) latent variables much weaker than those studies; our group-specific variables can be causally related across groups in our model (Eq.~\ref{eq:ps}) rather than mutually independent (e.g., Corollary 3.10 and 3.11 of \citet{Yao2023}).
Furthermore, our assumption makes the latent variables identifiable up to variable-wise transformations (Theorem~\ref{thm:id}) rather than block-wise transformations.

The identifiability of the causal graphs only applies for connections {\it across} groups, and those {\it within} each group are left unknown.
Nevertheless, since the latent variables are guaranteed to be identifiable, we can apply existing causal discovery framework for estimating them as a post-processing. 
Although the $\phi$ is assumed to be the same across all group-pairs for simplicity, it could be different across them, but that might require some additional assumptions.

Future research might consider relaxing some of the assumptions in our work.
Although the grouping of variables without overlaps between groups can be satisfied in plenty of situations in practice (see above),
allowing such overlaps should make our framework even more applicable to a wider variety of situations.
Multivariate (vector) causal variables might be more favorable for representing complex and high-level phenomena behind data, compared to univariate causal variables as considered in this study.
In such case, vector-wise identifiability as in \citep{Daunhawer2023, Yao2023} rather than component-wise one (Theorem~\ref{thm:id}) might be enough, which is weaker and thus might relax some of the assumptions in this study.
Applying our framework to more realistic datasets, such as iTHOR \citep{Kolve2022}, would be an important future direction for assessing its applicability and robustness.

\section{Conclusion}
\label{sec:conclusion}

This study proposed a new identifiable model for CRL, together with its self-supervised estimation framework G-CaRL.
The new approach is the assumption of the grouping of the observational variables, which appears naturally in many practical applications such as multi-sensor measurements or time series.
Such an assumption allowed us to significantly weaken any other assumptions required on the latent causal mechanisms in existing frameworks. 
In contrast to existing CRL models, our model does not require temporal structure (although it can use it as a special case), nor does it assume any supervision or interventions.
Although our model restricts the inter-group causal relations of variables to some extent,
it allows nonlinearity and even cycles, which is more general than most of the causal discovery models.
Numerical experiments showed better performances compared to the state-of-the-art baselines, thus making G-CaRL a  promising candidate for real-world CRL in a wide variety of fields.

\section*{Acknowledgements}

This research was supported in part by JST PRESTO JPMJPR2028, JSPS KAKENHI 24H02177, 22H05666, and 22K17956. A.H. was funded by a Fellow Position from CIFAR, and the Academy of Finland (project \#330482). 
We also would like to thank the anonymous reviewers for very useful comments that helped us improve the manuscript.

\section*{Impact Statement}
This paper presents work whose goal is to advance the field of 
Machine Learning. There are many potential societal consequences 
of our work, none which we feel must be specifically highlighted here.


\bibliography{references}

\begin{thebibliography}{107}
\providecommand{\natexlab}[1]{#1}
\providecommand{\url}[1]{\texttt{#1}}
\expandafter\ifx\csname urlstyle\endcsname\relax
  \providecommand{\doi}[1]{doi: #1}\else
  \providecommand{\doi}{doi: \begingroup \urlstyle{rm}\Url}\fi

\bibitem[Acosta et~al.(2022)Acosta, Falcone, Rajpurkar, and Topol]{Acosta2022}
Acosta, J.~N., Falcone, G.~J., Rajpurkar, P., and Topol, E.~J.
\newblock Multimodal biomedical {AI}.
\newblock \emph{Nature Medicine}, 28\penalty0 (9):\penalty0 1773--1784, 2022.

\bibitem[Ahuja et~al.(2022{\natexlab{a}})Ahuja, Hartford, and
  Bengio]{Ahuja2022a}
Ahuja, K., Hartford, J.~S., and Bengio, Y.
\newblock Weakly supervised representation learning with sparse perturbations.
\newblock In \emph{Advances in Neural Information Processing Systems},
  volume~35, pp.\  15516--15528, 2022{\natexlab{a}}.

\bibitem[Ahuja et~al.(2022{\natexlab{b}})Ahuja, Wang, Mahajan, and
  Bengio]{Ahuja2022b}
Ahuja, K., Wang, Y., Mahajan, D., and Bengio, Y.
\newblock Interventional causal representation learning.
\newblock In \emph{NeurIPS 2022 Workshop on Neuro Causal and Symbolic AI
  (nCSI)}, 2022{\natexlab{b}}.

\bibitem[Andersson et~al.(1997)Andersson, Madigan, and Perlman]{Andersson1997}
Andersson, S.~A., Madigan, D., and Perlman, M.~D.
\newblock A characterization of {Markov} equivalence classes for acyclic
  digraphs.
\newblock \emph{The Annals of Statistics}, 25\penalty0 (2):\penalty0 505--541,
  1997.

\bibitem[Bengio et~al.(2013)Bengio, Courville, and Vincent]{Bengio2013}
Bengio, Y., Courville, A., and Vincent, P.
\newblock Representation learning: A review and new perspectives.
\newblock \emph{IEEE Transactions on Pattern Analysis and Machine
  Intelligence}, 35\penalty0 (8):\penalty0 1798--1828, 2013.

\bibitem[Bollen(1989)]{Bollen1989}
Bollen, K.~A.
\newblock \emph{Structural Equations with Latent Variables}.
\newblock John Wiley \& Sons, 1989.

\bibitem[Brehmer et~al.(2022)Brehmer, de~Haan, Lippe, and Cohen]{Brehmer2022}
Brehmer, J., de~Haan, P., Lippe, P., and Cohen, T.~S.
\newblock Weakly supervised causal representation learning.
\newblock In \emph{Advances in Neural Information Processing Systems},
  volume~35, pp.\  38319--38331, 2022.

\bibitem[Buchholz et~al.(2023)Buchholz, Rajendran, Rosenfeld, Aragam,
  Sch{\"o}lkopf, and Ravikumar]{Buchholz2023}
Buchholz, S., Rajendran, G., Rosenfeld, E., Aragam, B., Sch{\"o}lkopf, B., and
  Ravikumar, P.
\newblock Learning linear causal representations from interventions under
  general nonlinear mixing.
\newblock arXiv, 2023.

\bibitem[B{\"u}hlmann et~al.(2014)B{\"u}hlmann, Peters, and
  Ernest]{Buhlmann2014}
B{\"u}hlmann, P., Peters, J., and Ernest, J.
\newblock {CAM: Causal additive models, high-dimensional order search and
  penalized regression}.
\newblock \emph{The Annals of Statistics}, 42\penalty0 (6):\penalty0 2526 --
  2556, 2014.

\bibitem[Burkhardt et~al.(2022)Burkhardt, Luecken, Benz, Holderrieth, Bloom,
  Lance, Chow, and Holbrook]{Burkhardt2022}
Burkhardt, D., Luecken, M., Benz, A., Holderrieth, P., Bloom, J., Lance, C.,
  Chow, A., and Holbrook, R.
\newblock Open problems - multimodal single-cell integration.
\newblock Kaggle, 2022.

\bibitem[Cai et~al.(2019)Cai, Xie, Glymour, Hao, and Zhang]{Cai2019}
Cai, R., Xie, F., Glymour, C., Hao, Z., and Zhang, K.
\newblock Triad constraints for learning causal structure of latent variables.
\newblock In \emph{Advances in Neural Information Processing Systems},
  volume~32, 2019.

\bibitem[Chen et~al.(2018)Chen, Li, Grosse, and Duvenaud]{Chen2018}
Chen, R. T.~Q., Li, X., Grosse, R.~B., and Duvenaud, D.~K.
\newblock Isolating sources of disentanglement in variational autoencoders.
\newblock In \emph{Advances in Neural Information Processing Systems},
  volume~31, 2018.

\bibitem[Chickering(2003)]{Chickering2003}
Chickering, D.~M.
\newblock Optimal structure identification with greedy search.
\newblock \emph{J. Mach. Learn. Res.}, 3:\penalty0 507--554, 2003.

\bibitem[Choi et~al.(2020)Choi, Chapkin, and Ni]{Choi2020}
Choi, J., Chapkin, R., and Ni, Y.
\newblock Bayesian causal structural learning with zero-inflated {Poisson}
  bayesian networks.
\newblock In \emph{Advances in Neural Information Processing Systems},
  volume~33, pp.\  5887--5897, 2020.

\bibitem[Comon(1994)]{Comon1994}
Comon, P.
\newblock Independent component analysis---a new concept?
\newblock \emph{Signal Processing}, 36:\penalty0 287--314, 1994.

\bibitem[Daunhawer et~al.(2023)Daunhawer, Bizeul, Palumbo, Marx, and
  Vogt]{Daunhawer2023}
Daunhawer, I., Bizeul, A., Palumbo, E., Marx, A., and Vogt, J.~E.
\newblock Identifiability results for multimodal contrastive learning.
\newblock In \emph{The Eleventh International Conference on Learning
  Representations}, 2023.

\bibitem[Dibaeinia \& Sinha(2020)Dibaeinia and Sinha]{Dibaeinia2020}
Dibaeinia, P. and Sinha, S.
\newblock {SERGIO}: A single-cell expression simulator guided by gene
  regulatory networks.
\newblock \emph{Cell Systems}, 11\penalty0 (3):\penalty0 252--271.e11, 2020.

\bibitem[Falck et~al.(2021)Falck, Zhang, Willetts, Nicholson, Yau, and
  Holmes]{Falck2021}
Falck, F., Zhang, H., Willetts, M., Nicholson, G., Yau, C., and Holmes, C.~C.
\newblock Multi-facet clustering variational autoencoders.
\newblock In \emph{Advances in Neural Information Processing Systems},
  volume~34, pp.\  8676--8690, 2021.

\bibitem[Geiger \& Heckerman(1994)Geiger and Heckerman]{Geiger1994}
Geiger, D. and Heckerman, D.
\newblock Learning {Gaussian} networks.
\newblock In \emph{Proceedings of the Tenth international conference on
  Uncertainty in artificial intelligence}, pp.\  235--243, 1994.

\bibitem[Gong et~al.(2015)Gong, Zhang, Schoelkopf, Tao, and Geiger]{Gong2015}
Gong, M., Zhang, K., Schoelkopf, B., Tao, D., and Geiger, P.
\newblock Discovering temporal causal relations from subsampled data.
\newblock In \emph{Proceedings of the 32nd International Conference on Machine
  Learning}, pp.\  1898--1906, 2015.

\bibitem[Gresele et~al.(2020)Gresele, Rubenstein, Mehrjou, Locatello, and
  Sch{\"{o}}lkopf]{Gresele2020}
Gresele, L., Rubenstein, P.~K., Mehrjou, A., Locatello, F., and
  Sch{\"{o}}lkopf, B.
\newblock The incomplete {Rosetta Stone} problem: Identifiability results for
  multi-view nonlinear {ICA}.
\newblock In \emph{Proceedings of The 35th Uncertainty in Artificial
  Intelligence Conference}, volume 115, pp.\  217--227, 2020.

\bibitem[Gresele et~al.(2021)Gresele, K\"ugelgen, Stimper, Sch\"olkopf, and
  Besserve]{Gresele2021}
Gresele, L., K\"ugelgen, J.~V., Stimper, V., Sch\"olkopf, B., and Besserve, M.
\newblock Independent mechanism analysis, a new concept?
\newblock In \emph{Advances in Neural Information Processing Systems}, 2021.

\bibitem[Gutmann \& Hyv{{\"a}}rinen(2012)Gutmann and
  Hyv{{\"a}}rinen]{Gutmann2012}
Gutmann, M.~U. and Hyv{{\"a}}rinen, A.
\newblock Noise-contrastive estimation of unnormalized statistical models, with
  applications to natural image statistics.
\newblock \emph{Journal of Machine Learning Research}, 13\penalty0
  (11):\penalty0 307--361, 2012.

\bibitem[H\"{a}lv\"{a} \& Hyv\"{a}rinen(2020)H\"{a}lv\"{a} and
  Hyv\"{a}rinen]{Halva2020}
H\"{a}lv\"{a}, H. and Hyv\"{a}rinen, A.
\newblock Hidden {Markov} nonlinear {ICA}: Unsupervised learning from
  nonstationary time series.
\newblock In \emph{Proceedings of the 36th Conference on Uncertainty in
  Artificial Intelligence (UAI)}, volume 124, pp.\  939--948, 2020.

\bibitem[H\"{a}lv\"{a} et~al.(2021)H\"{a}lv\"{a}, Le~Corff, Leh\'{e}ricy, So,
  Zhu, Gassiat, and Hyvarinen]{Halva2021}
H\"{a}lv\"{a}, H., Le~Corff, S., Leh\'{e}ricy, L., So, J., Zhu, Y., Gassiat,
  E., and Hyvarinen, A.
\newblock Disentangling identifiable features from noisy data with structured
  nonlinear ica.
\newblock In \emph{Advances in Neural Information Processing Systems},
  volume~34, pp.\  1624--1633, 2021.

\bibitem[Hastie et~al.(2001)Hastie, Tibshirani, and Friedman]{Hastie2001}
Hastie, T., Tibshirani, R., and Friedman, J.
\newblock \emph{The Elements of Statistical Learning}.
\newblock Springer, New York, NY, 2001.

\bibitem[He et~al.(2016)He, Zhang, Ren, and Sun]{He2016}
He, K., Zhang, X., Ren, S., and Sun, J.
\newblock Deep residual learning for image recognition.
\newblock In \emph{Proceedings of the IEEE Conference on Computer Vision and
  Pattern Recognition (CVPR)}, 2016.

\bibitem[Hebart et~al.(2023)Hebart, Contier, Teichmann, Rockter, Zheng, Kidder,
  Corriveau, Vaziri-Pashkam, and Baker]{Hebart2023}
Hebart, M.~N., Contier, O., Teichmann, L., Rockter, A.~H., Zheng, C.~Y.,
  Kidder, A., Corriveau, A., Vaziri-Pashkam, M., and Baker, C.~I.
\newblock Things-data, a multimodal collection of large-scale datasets for
  investigating object representations in human brain and behavior.
\newblock \emph{eLife}, 12:\penalty0 e82580, 2023.

\bibitem[Higgins et~al.(2017)Higgins, Matthey, Pal, Burgess, Glorot, Botvinick,
  Mohamed, and Lerchner]{Higgins2017}
Higgins, I., Matthey, L., Pal, A., Burgess, C., Glorot, X., Botvinick, M.,
  Mohamed, S., and Lerchner, A.
\newblock $\beta$-vae: Learning basic visual concepts with a constrained
  variational framework.
\newblock In \emph{International Conference on Learning Representations}, 2017.

\bibitem[Horan et~al.(2021)Horan, Richardson, and Weiss]{Horan2021}
Horan, D., Richardson, E., and Weiss, Y.
\newblock When is unsupervised disentanglement possible?
\newblock In \emph{Advances in Neural Information Processing Systems},
  volume~34, pp.\  5150--5161, 2021.

\bibitem[Hornik et~al.(1989)Hornik, Stinchcombe, and White]{Hornik1989}
Hornik, K., Stinchcombe, M., and White, H.
\newblock Multilayer feedforward networks are universal approximators.
\newblock \emph{Neural Networks}, 2\penalty0 (5):\penalty0 359 -- 366, 1989.

\bibitem[Hoyer et~al.(2008{\natexlab{a}})Hoyer, Janzing, Mooij, Peters, and
  Sch\"{o}lkopf]{Hoyer2008}
Hoyer, P.~O., Janzing, D., Mooij, J., Peters, J., and Sch\"{o}lkopf, B.
\newblock Nonlinear causal discovery with additive noise models.
\newblock In \emph{Proceedings of the 21st International Conference on Neural
  Information Processing Systems}, pp.\  689--696, 2008{\natexlab{a}}.

\bibitem[Hoyer et~al.(2008{\natexlab{b}})Hoyer, Shimizu, Kerminen, and
  Palviainen]{Hoyer2008IJAR}
Hoyer, P.~O., Shimizu, S., Kerminen, A.~J., and Palviainen, M.
\newblock Estimation of causal effects using linear non-{Gaussian} causal
  models with hidden variables.
\newblock \emph{International Journal of Approximate Reasoning}, 49\penalty0
  (2):\penalty0 362--378, 2008{\natexlab{b}}.

\bibitem[Hyv\"{a}rinen \& Morioka(2016)Hyv\"{a}rinen and
  Morioka]{Hyvarinen2016}
Hyv\"{a}rinen, A. and Morioka, H.
\newblock Unsupervised feature extraction by time-contrastive learning and
  nonlinear {ICA}.
\newblock In \emph{Advances in Neural Information Processing Systems (NIPS)
  29}, pp.\  3765--3773. 2016.

\bibitem[Hyvarinen \& Morioka(2017)Hyvarinen and Morioka]{Hyvarinen2017}
Hyvarinen, A. and Morioka, H.
\newblock Nonlinear {ICA} of temporally dependent stationary sources.
\newblock In \emph{AISTATS}, pp.\  460--469, 2017.

\bibitem[Hyv{\"a}rinen \& Pajunen(1999)Hyv{\"a}rinen and
  Pajunen]{Hyvarinen1999}
Hyv{\"a}rinen, A. and Pajunen, P.
\newblock Nonlinear independent component analysis: Existence and uniqueness
  results.
\newblock \emph{Neural Netw.}, 12\penalty0 (3):\penalty0 429 -- 439, 1999.

\bibitem[Hyv\"{a}rinen \& Smith(2013)Hyv\"{a}rinen and Smith]{Hyvarinen2013}
Hyv\"{a}rinen, A. and Smith, S.~M.
\newblock Pairwise likelihood ratios for estimation of non-{Gaussian}
  structural equation models.
\newblock \emph{Journal of Machine Learning Research}, 14\penalty0
  (1):\penalty0 111--152, 2013.

\bibitem[Hyv{{\"a}}rinen et~al.(2010)Hyv{{\"a}}rinen, Zhang, Shimizu, and
  Hoyer]{Hyvarinen2010}
Hyv{{\"a}}rinen, A., Zhang, K., Shimizu, S., and Hoyer, P.~O.
\newblock Estimation of a structural vector autoregression model using
  non-{Gaussianity}.
\newblock \emph{Journal of Machine Learning Research}, 11\penalty0
  (56):\penalty0 1709--1731, 2010.

\bibitem[Hyvarinen et~al.(2019)Hyvarinen, Sasaki, and Turner]{Hyvarinen2019}
Hyvarinen, A., Sasaki, H., and Turner, R.
\newblock Nonlinear {ICA} using auxiliary variables and generalized contrastive
  learning.
\newblock In \emph{AISTATS}, pp.\  859--868, 2019.

\bibitem[Jiang et~al.(2017)Jiang, Zheng, Tan, Tang, and Zhou]{Jiang2017}
Jiang, Z., Zheng, Y., Tan, H., Tang, B., and Zhou, H.
\newblock Variational deep embedding: An unsupervised and generative approach
  to clustering.
\newblock In \emph{Proceedings of the 26th International Joint Conference on
  Artificial Intelligence}, number~8, pp.\  1965--1972, 2017.

\bibitem[Khemakhem et~al.(2020{\natexlab{a}})Khemakhem, Kingma, Monti, and
  Hyv{\"a}rinen]{Khemakhem2020}
Khemakhem, I., Kingma, D.~P., Monti, R.~P., and Hyv{\"a}rinen, A.
\newblock Variational autoencoders and nonlinear {ICA}: A unifying framework.
\newblock In \emph{AISTATS}, 2020{\natexlab{a}}.

\bibitem[Khemakhem et~al.(2020{\natexlab{b}})Khemakhem, Monti, Kingma, and
  Hyvarinen]{Khemakhem2020beem}
Khemakhem, I., Monti, R., Kingma, D., and Hyvarinen, A.
\newblock Ice-beem: Identifiable conditional energy-based deep models based on
  nonlinear ica.
\newblock In \emph{Advances in Neural Information Processing Systems},
  volume~33, pp.\  12768--12778, 2020{\natexlab{b}}.

\bibitem[Kim \& Mnih(2018)Kim and Mnih]{Kim2018}
Kim, H. and Mnih, A.
\newblock Disentangling by factorising.
\newblock In \emph{Proceedings of the 35th International Conference on Machine
  Learning}, volume~80, pp.\  2649--2658, 2018.

\bibitem[Kingma \& Welling(2014)Kingma and Welling]{Kingma2014}
Kingma, D.~P. and Welling, M.
\newblock Auto-encoding variational bayes.
\newblock In \emph{ICLR}, 2014.

\bibitem[Kivva et~al.(2021)Kivva, Rajendran, Ravikumar, and Aragam]{Kivva2021}
Kivva, B., Rajendran, G., Ravikumar, P., and Aragam, B.
\newblock Learning latent causal graphs via mixture oracles.
\newblock In \emph{Advances in Neural Information Processing Systems},
  volume~34, pp.\  18087--18101, 2021.

\bibitem[Kivva et~al.(2022)Kivva, Rajendran, Ravikumar, and Aragam]{Kivva2022}
Kivva, B., Rajendran, G., Ravikumar, P., and Aragam, B.
\newblock Identifiability of deep generative models without auxiliary
  information.
\newblock In \emph{Advances in Neural Information Processing Systems},
  volume~35, pp.\  15687--15701, 2022.

\bibitem[Klindt et~al.(2021)Klindt, Schott, Sharma, Ustyuzhaninov, Brendel,
  Bethge, and Paiton]{Klindt2021}
Klindt, D.~A., Schott, L., Sharma, Y., Ustyuzhaninov, I., Brendel, W., Bethge,
  M., and Paiton, D.
\newblock Towards nonlinear disentanglement in natural data with temporal
  sparse coding.
\newblock In \emph{International Conference on Learning Representations}, 2021.

\bibitem[Kolve et~al.(2022)Kolve, Mottaghi, Han, VanderBilt, Weihs, Herrasti,
  Deitke, Ehsani, Gordon, Zhu, Kembhavi, Gupta, and Farhadi]{Kolve2022}
Kolve, E., Mottaghi, R., Han, W., VanderBilt, E., Weihs, L., Herrasti, A.,
  Deitke, M., Ehsani, K., Gordon, D., Zhu, Y., Kembhavi, A., Gupta, A., and
  Farhadi, A.
\newblock Ai2-thor: An interactive 3d environment for visual ai.
\newblock \emph{arXiv}, 2022.

\bibitem[Lacerda et~al.(2008)Lacerda, Spirtes, Ramsey, and Hoyer]{Lacerda2008}
Lacerda, G., Spirtes, P., Ramsey, J., and Hoyer, P.~O.
\newblock Discovering cyclic causal models by independent components analysis.
\newblock In \emph{Proceedings of the Twenty-Fourth Conference on Uncertainty
  in Artificial Intelligence}, pp.\  366--374, 2008.

\bibitem[Lachapelle et~al.(2022)Lachapelle, Rodriguez, Sharma, Everett, PRIOL,
  Lacoste, and Lacoste-Julien]{Lachapelle2022}
Lachapelle, S., Rodriguez, P., Sharma, Y., Everett, K.~E., PRIOL, R.~L.,
  Lacoste, A., and Lacoste-Julien, S.
\newblock Disentanglement via mechanism sparsity regularization: A new
  principle for nonlinear {ICA}.
\newblock In \emph{Proceedings of the First Conference on Causal Learning and
  Reasoning}, volume 177, pp.\  428--484, 2022.

\bibitem[Leeb et~al.(2022)Leeb, Lanzillotta, Annadani, Besserve, Bauer, and
  Sch{\"o}lkopf]{Leeb2022}
Leeb, F., Lanzillotta, G., Annadani, Y., Besserve, M., Bauer, S., and
  Sch{\"o}lkopf, B.
\newblock Structure by architecture: Disentangled representations without
  regularization.
\newblock In \emph{UAI 2022 Workshop on Causal Representation Learning}, 2022.

\bibitem[Li et~al.(2020)Li, Torralba, Anandkumar, Fox, and Garg]{Li2020}
Li, Y., Torralba, A., Anandkumar, A., Fox, D., and Garg, A.
\newblock Causal discovery in physical systems from videos.
\newblock In \emph{Advances in Neural Information Processing Systems},
  volume~33, pp.\  9180--9192, 2020.

\bibitem[Lippe et~al.(2022)Lippe, Magliacane, L{\"o}we, Asano, Cohen, and
  Gavves]{Lippe2022}
Lippe, P., Magliacane, S., L{\"o}we, S., Asano, Y.~M., Cohen, T., and Gavves,
  E.
\newblock {CITRIS}: Causal identifiability from temporal intervened sequences.
\newblock In \emph{ICLR2022 Workshop on the Elements of Reasoning: Objects,
  Structure and Causality}, 2022.

\bibitem[Lippe et~al.(2023)Lippe, Magliacane, L{\"o}we, Asano, Cohen, and
  Gavves]{Lippe2023}
Lippe, P., Magliacane, S., L{\"o}we, S., Asano, Y.~M., Cohen, T., and Gavves,
  E.
\newblock Causal representation learning for instantaneous and temporal effects
  in interactive systems.
\newblock In \emph{The Eleventh International Conference on Learning
  Representations}, 2023.

\bibitem[Liu et~al.(2023)Liu, Zhang, Gong, Gong, Huang, van~den Hengel, Zhang,
  and Shi]{Liu2023}
Liu, Y., Zhang, Z., Gong, D., Gong, M., Huang, B., van~den Hengel, A., Zhang,
  K., and Shi, J.~Q.
\newblock Identifying weight-variant latent causal models.
\newblock arXiv, 2023.

\bibitem[Locatello et~al.(2019)Locatello, Bauer, Lucic, Raetsch, Gelly,
  Sch{\"o}lkopf, and Bachem]{Locatello2019}
Locatello, F., Bauer, S., Lucic, M., Raetsch, G., Gelly, S., Sch{\"o}lkopf, B.,
  and Bachem, O.
\newblock Challenging common assumptions in the unsupervised learning of
  disentangled representations.
\newblock In \emph{Proceedings of the 36th International Conference on Machine
  Learning}, volume~97, pp.\  4114--4124, 2019.

\bibitem[Locatello et~al.(2020)Locatello, Poole, Raetsch, Sch{\"o}lkopf,
  Bachem, and Tschannen]{Locatello2020}
Locatello, F., Poole, B., Raetsch, G., Sch{\"o}lkopf, B., Bachem, O., and
  Tschannen, M.
\newblock Weakly-supervised disentanglement without compromises.
\newblock In \emph{Proceedings of the 37th International Conference on Machine
  Learning}, volume 119, pp.\  6348--6359, 2020.

\bibitem[Longman et~al.(2018)Longman, Giambelluca, Nullet, Frazier, Kodama,
  Crausbay, Krushelnycky, Cordell, Clark, Newman, and Arnold]{Longman2018}
Longman, R.~J., Giambelluca, T.~W., Nullet, M.~A., Frazier, A.~G., Kodama, K.,
  Crausbay, S.~D., Krushelnycky, P.~D., Cordell, S., Clark, M.~P., Newman,
  A.~J., and Arnold, J.~R.
\newblock Compilation of climate data from heterogeneous networks across the
  {Hawaiian} islands.
\newblock \emph{Scientific Data}, 5\penalty0 (1):\penalty0 180012, 2018.

\bibitem[Lyu et~al.(2022)Lyu, Fu, Wang, and Lu]{Lyu2022}
Lyu, Q., Fu, X., Wang, W., and Lu, S.
\newblock Understanding latent correlation-based multiview learning and
  self-supervision: An identifiability perspective.
\newblock In \emph{International Conference on Learning Representations}, 2022.

\bibitem[Maeda \& Shimizu(2020)Maeda and Shimizu]{Maeda2020}
Maeda, T.~N. and Shimizu, S.
\newblock {RCD}: Repetitive causal discovery of linear non-{Gaussian} acyclic
  models with latent confounders.
\newblock In \emph{Proceedings of the Twenty Third International Conference on
  Artificial Intelligence and Statistics}, pp.\  735--745, 2020.

\bibitem[Monti et~al.(2020)Monti, Zhang, and Hyv{\"{a}}rinen]{Monti2020}
Monti, R.~P., Zhang, K., and Hyv{\"{a}}rinen, A.
\newblock Causal discovery with general non-linear relationships using
  non-linear {ICA}.
\newblock In \emph{Proceedings of The 35th Uncertainty in Artificial
  Intelligence Conference}, volume 115, pp.\  186--195, 2020.

\bibitem[Moran et~al.(2021)Moran, Sridhar, Wang, and Blei]{Moran2021}
Moran, G.~E., Sridhar, D., Wang, Y., and Blei, D.
\newblock Identifiable deep generative models via sparse decoding.
\newblock \emph{Transactions on Machine Learning Research}, 2021.

\bibitem[Morioka \& Hyvarinen(2023)Morioka and Hyvarinen]{Morioka2023}
Morioka, H. and Hyvarinen, A.
\newblock Connectivity-contrastive learning: Combining causal discovery and
  representation learning for multimodal data.
\newblock In \emph{Proceedings of The 26th International Conference on
  Artificial Intelligence and Statistics}, volume 206, pp.\  3399--3426, 2023.

\bibitem[Morioka et~al.(2020)Morioka, Calhoun, and Hyv{\"a}rinen]{Morioka2020}
Morioka, H., Calhoun, V., and Hyv{\"a}rinen, A.
\newblock Nonlinear ica of fmri reveals primitive temporal structures linked to
  rest, task, and behavioral traits.
\newblock In 218 (ed.), \emph{NeuroImage}, pp.\  116989, 2020.

\bibitem[Morioka et~al.(2021)Morioka, H{\"a}lv{\"a}, and
  Hyvarinen]{Morioka2021}
Morioka, H., H{\"a}lv{\"a}, H., and Hyvarinen, A.
\newblock Independent innovation analysis for nonlinear vector autoregressive
  process.
\newblock In \emph{Proceedings of The 24th International Conference on
  Artificial Intelligence and Statistics}, volume 130, pp.\  1549--1557, 2021.

\bibitem[Munkres(1957)]{Munkres1957}
Munkres, J.
\newblock Algorithms for the assignment and transportation problems.
\newblock \emph{Journal of the Society for Industrial and Applied Mathematics},
  5\penalty0 (1):\penalty0 32--38, 1957.

\bibitem[Ng et~al.(2020)Ng, Ghassami, and Zhang]{Ng2020}
Ng, I., Ghassami, A., and Zhang, K.
\newblock On the role of sparsity and dag constraints for learning linear dags.
\newblock In \emph{Advances in Neural Information Processing Systems},
  volume~33, pp.\  17943--17954, 2020.

\bibitem[Park \& Park(2019{\natexlab{a}})Park and Park]{Park2019aistats}
Park, G. and Park, H.
\newblock Identifiability of generalized hypergeometric distribution ({GHD})
  directed acyclic graphical models.
\newblock In \emph{Proceedings of the Twenty-Second International Conference on
  Artificial Intelligence and Statistics}, volume~89, pp.\  158--166,
  2019{\natexlab{a}}.

\bibitem[Park \& Park(2019{\natexlab{b}})Park and Park]{Park2019}
Park, G. and Park, S.
\newblock High-dimensional {Poisson} structural equation model learning via
  $\ell_1$-regularized regression.
\newblock \emph{Journal of Machine Learning Research}, 20\penalty0
  (95):\penalty0 1--41, 2019{\natexlab{b}}.

\bibitem[Park \& Raskutti(2015)Park and Raskutti]{Park2015}
Park, G. and Raskutti, G.
\newblock Learning large-scale {Poisson} {DAG} models based on overdispersion
  scoring.
\newblock In \emph{Advances in Neural Information Processing Systems},
  volume~28, 2015.

\bibitem[Pearl(2000)]{Pearl2000}
Pearl, J.
\newblock \emph{Causality: Models, Reasoning, and Inference}.
\newblock Cambridge University Press, 2000.

\bibitem[Peters et~al.(2014)Peters, Mooij, Janzing, and
  Sch{{\"o}}lkopf]{Peters2014}
Peters, J., Mooij, J.~M., Janzing, D., and Sch{{\"o}}lkopf, B.
\newblock Causal discovery with continuous additive noise models.
\newblock \emph{Journal of Machine Learning Research}, 15\penalty0
  (58):\penalty0 2009--2053, 2014.

\bibitem[Reisach et~al.(2021)Reisach, Seiler, and Weichwald]{Reisach2021}
Reisach, A., Seiler, C., and Weichwald, S.
\newblock Beware of the simulated dag! causal discovery benchmarks may be easy
  to game.
\newblock In \emph{Advances in Neural Information Processing Systems},
  volume~34, pp.\  27772--27784, 2021.

\bibitem[Roeder et~al.(2021)Roeder, Metz, and Kingma]{Roeder2021}
Roeder, G., Metz, L., and Kingma, D.
\newblock On linear identifiability of learned representations.
\newblock In \emph{Proceedings of the 38th International Conference on Machine
  Learning}, volume 139, pp.\  9030--9039, 2021.

\bibitem[Sch{\"o}lkopf et~al.(2021)Sch{\"o}lkopf, Locatello, Bauer, Ke,
  Kalchbrenner, Goyal, and Bengio]{Schlkopf2021}
Sch{\"o}lkopf, B., Locatello, F., Bauer, S., Ke, N.~R., Kalchbrenner, N.,
  Goyal, A., and Bengio, Y.
\newblock Towards causal representation learning.
\newblock arXiv, 2021.

\bibitem[Shen et~al.(2022)Shen, Liu, Dong, Lian, Chen, and Zhang]{Shen2022}
Shen, X., Liu, F., Dong, H., Lian, Q., Chen, Z., and Zhang, T.
\newblock Weakly supervised disentangled generative causal representation
  learning.
\newblock \emph{Journal of Machine Learning Research}, 23\penalty0
  (241):\penalty0 1--55, 2022.

\bibitem[Shimizu \& Bollen(2014)Shimizu and Bollen]{Shimizu2014}
Shimizu, S. and Bollen, K.
\newblock Bayesian estimation of causal direction in acyclic structural
  equation models with individual-specific confounder variables and
  non-{Gaussian} distributions.
\newblock \emph{Journal of Machine Learning Research}, 15\penalty0
  (76):\penalty0 2629--2652, 2014.

\bibitem[Shimizu et~al.(2006)Shimizu, Hoyer, Hyv\"{a}rinen, and
  Kerminen]{Shimizu2006}
Shimizu, S., Hoyer, P.~O., Hyv\"{a}rinen, A., and Kerminen, A.
\newblock A linear non-{Gaussian} acyclic model for causal discovery.
\newblock \emph{Journal of Machine Learning Research}, 7\penalty0
  (72):\penalty0 2003--2030, 2006.

\bibitem[Shimizu et~al.(2011)Shimizu, Inazumi, Sogawa, Hyv{{\"a}}rinen,
  Kawahara, Washio, Hoyer, and Bollen]{Shimizu2011}
Shimizu, S., Inazumi, T., Sogawa, Y., Hyv{{\"a}}rinen, A., Kawahara, Y.,
  Washio, T., Hoyer, P.~O., and Bollen, K.
\newblock {DirectLiNGAM}: A direct method for learning a linear non-{Gaussian}
  structural equation model.
\newblock \emph{Journal of Machine Learning Research}, 12\penalty0
  (33):\penalty0 1225--1248, 2011.

\bibitem[Sorrenson et~al.(2020)Sorrenson, Rother, and K{\"o}the]{Sorrenson2020}
Sorrenson, P., Rother, C., and K{\"o}the, U.
\newblock Disentanglement by nonlinear ica with general incompressible-flow
  networks ({GIN}).
\newblock In \emph{International Conference on Learning Representations}, 2020.

\bibitem[Spirtes \& Glymour(1991)Spirtes and Glymour]{Spirtes1991}
Spirtes, P. and Glymour, C.
\newblock An algorithm for fast recovery of sparse causal graphs.
\newblock \emph{Social Science Computer Review}, 9\penalty0 (1):\penalty0
  62--72, 1991.

\bibitem[Spirtes et~al.(2001)Spirtes, Glymour, and Scheines]{Spirtes2001}
Spirtes, P., Glymour, C., and Scheines, R.
\newblock \emph{Causation, Prediction, and Search}.
\newblock MIT Press, 2001.

\bibitem[Sprekeler et~al.(2014)Sprekeler, Zito, and Wiskott]{Sprekeler2014}
Sprekeler, H., Zito, T., and Wiskott, L.
\newblock An extension of slow feature analysis for nonlinear blind source
  separation.
\newblock \emph{Journal of Machine Learning Research}, 15\penalty0
  (26):\penalty0 921--947, 2014.

\bibitem[Squires et~al.(2023)Squires, Seigal, Bhate, and Uhler]{Squires2023}
Squires, C., Seigal, A., Bhate, S., and Uhler, C.
\newblock Linear causal disentanglement via interventions.
\newblock In \emph{Proceedings of the 40th International Conference on Machine
  Learning}, 2023.

\bibitem[Sriperumbudur et~al.(2017)Sriperumbudur, Fukumizu, Gretton,
  Hyv\"{a}rinen, and Kumar]{Sriperumbudur2017}
Sriperumbudur, B., Fukumizu, K., Gretton, A., Hyv\"{a}rinen, A., and Kumar, R.
\newblock Density estimation in infinite dimensional exponential families.
\newblock \emph{Journal of Machine Learning Research}, 18\penalty0
  (57):\penalty0 1--59, 2017.

\bibitem[Sturma et~al.(2023)Sturma, Squires, Drton, and Uhler]{Sturma2023}
Sturma, N., Squires, C., Drton, M., and Uhler, C.
\newblock Unpaired multi-domain causal representation learning.
\newblock arXiv, 2023.

\bibitem[Varici et~al.(2023)Varici, Acarturk, Shanmugam, Kumar, and
  Tajer]{Varici2023}
Varici, B., Acarturk, E., Shanmugam, K., Kumar, A., and Tajer, A.
\newblock Score-based causal representation learning with interventions.
\newblock arXiv, 2023.

\bibitem[von K\"{u}gelgen et~al.(2021)von K\"{u}gelgen, Sharma, Gresele,
  Brendel, Sch\"{o}lkopf, Besserve, and Locatello]{Kugelgen2021}
von K\"{u}gelgen, J., Sharma, Y., Gresele, L., Brendel, W., Sch\"{o}lkopf, B.,
  Besserve, M., and Locatello, F.
\newblock Self-supervised learning with data augmentations provably isolates
  content from style.
\newblock In \emph{Advances in Neural Information Processing Systems},
  volume~34, pp.\  16451--16467, 2021.

\bibitem[von K{\"u}gelgen et~al.(2023)von K{\"u}gelgen, Besserve, Wendong,
  Gresele, Keki{\'c}, Bareinboim, Blei, and Sch{\"o}lkopf]{Kugelgen2023}
von K{\"u}gelgen, J., Besserve, M., Wendong, L., Gresele, L., Keki{\'c}, A.,
  Bareinboim, E., Blei, D.~M., and Sch{\"o}lkopf, B.
\newblock Nonparametric identifiability of causal representations from unknown
  interventions.
\newblock arXiv, 2023.

\bibitem[Wang et~al.(2021)Wang, Blei, and Cunningham]{Wang2021b}
Wang, Y., Blei, D., and Cunningham, J.~P.
\newblock Posterior collapse and latent variable non-identifiability.
\newblock In \emph{Advances in Neural Information Processing Systems},
  volume~34, pp.\  5443--5455, 2021.

\bibitem[Willetts \& Paige(2022)Willetts and Paige]{Willetts2022}
Willetts, M. and Paige, B.
\newblock I don't need u: Identifiable non-linear ica without side information.
\newblock arXiv, 2022.

\bibitem[Wu \& Fukumizu(2020)Wu and Fukumizu]{Wu2020}
Wu, P. and Fukumizu, K.
\newblock Causal mosaic: Cause-effect inference via nonlinear {ICA} and
  ensemble method.
\newblock In \emph{Proceedings of the Twenty Third International Conference on
  Artificial Intelligence and Statistics}, volume 108, pp.\  1157--1167, 2020.

\bibitem[Xie et~al.(2020)Xie, Cai, Huang, Glymour, Hao, and Zhang]{Xie2020}
Xie, F., Cai, R., Huang, B., Glymour, C., Hao, Z., and Zhang, K.
\newblock Generalized independent noise condition for estimating latent
  variable causal graphs.
\newblock In \emph{Advances in Neural Information Processing Systems},
  volume~33, pp.\  14891--14902, 2020.

\bibitem[Xie et~al.(2022)Xie, Huang, Chen, He, Geng, and Zhang]{Xie2022}
Xie, F., Huang, B., Chen, Z., He, Y., Geng, Z., and Zhang, K.
\newblock Identification of linear non-{G}aussian latent hierarchical
  structure.
\newblock In \emph{Proceedings of the 39th International Conference on Machine
  Learning}, volume 162, pp.\  24370--24387, 2022.

\bibitem[Yang et~al.(2021)Yang, Liu, Chen, Shen, Hao, and Wang]{Yang2021}
Yang, M., Liu, F., Chen, Z., Shen, X., Hao, J., and Wang, J.
\newblock {CausalVAE}: Disentangled representation learning via neural
  structural causal models.
\newblock In \emph{Proceedings of the IEEE/CVF Conference on Computer Vision
  and Pattern Recognition (CVPR)}, pp.\  9593--9602, 2021.

\bibitem[Yang et~al.(2022)Yang, Wang, Sun, Zhang, Zhang, Li, and Yan]{Yang2022}
Yang, X., Wang, Y., Sun, J., Zhang, X., Zhang, S., Li, Z., and Yan, J.
\newblock Nonlinear {ICA} using volume-preserving transformations.
\newblock In \emph{International Conference on Learning Representations}, 2022.

\bibitem[Yao et~al.(2023)Yao, Xu, Lachapelle, Magliacane, Taslakian, Martius,
  von K{\"u}gelgen, and Locatello]{Yao2023}
Yao, D., Xu, D., Lachapelle, S., Magliacane, S., Taslakian, P., Martius, G.,
  von K{\"u}gelgen, J., and Locatello, F.
\newblock Multi-view causal representation learning with partial observability.
\newblock arXiv, 2023.

\bibitem[Yao et~al.(2022{\natexlab{a}})Yao, Chen, and Zhang]{Yao2022b}
Yao, W., Chen, G., and Zhang, K.
\newblock Learning latent causal dynamics.
\newblock arXiv, 2022{\natexlab{a}}.

\bibitem[Yao et~al.(2022{\natexlab{b}})Yao, Sun, Ho, Sun, and Zhang]{Yao2022a}
Yao, W., Sun, Y., Ho, A., Sun, C., and Zhang, K.
\newblock Learning temporally causal latent processes from general temporal
  data.
\newblock In \emph{International Conference on Learning Representations},
  2022{\natexlab{b}}.

\bibitem[Zhang \& Hyv\"{a}rinen(2009)Zhang and Hyv\"{a}rinen]{Zhang2009}
Zhang, K. and Hyv\"{a}rinen, A.
\newblock On the identifiability of the post-nonlinear causal model.
\newblock In \emph{Proceedings of the Twenty-Fifth Conference on Uncertainty in
  Artificial Intelligence}, pp.\  647--655, 2009.

\bibitem[Zhang \& Hyv\"arinen(2010)Zhang and Hyv\"arinen]{Zhang10UAI}
Zhang, K. and Hyv\"arinen, A.
\newblock Source separation and higher-order causal analysis of {MEG} and
  {EEG}.
\newblock In \emph{Proc.\ 26th Conference on Uncertainty in Artificial
  Intelligence (UAI2010)}, Catalina Island, California, 2010.

\bibitem[Zheng et~al.(2018)Zheng, Aragam, K, and Xing]{Zheng2018}
Zheng, X., Aragam, B., K, P.~R., and Xing, E.~P.
\newblock {DAGs} with {NO TEARS}: Continuous optimization for structure
  learning.
\newblock In \emph{Advances in Neural Information Processing Systems},
  volume~31, 2018.

\bibitem[Zheng et~al.(2020)Zheng, Dan, Aragam, Ravikumar, and Xing]{Zheng2020}
Zheng, X., Dan, C., Aragam, B., Ravikumar, P., and Xing, E.
\newblock Learning sparse nonparametric {DAGs}.
\newblock In \emph{Proceedings of the Twenty Third International Conference on
  Artificial Intelligence and Statistics}, volume 108, pp.\  3414--3425, 2020.

\bibitem[Zheng \& Zhang(2023)Zheng and Zhang]{Zheng2023}
Zheng, Y. and Zhang, K.
\newblock Generalizing nonlinear {ICA} beyond structural sparsity.
\newblock In \emph{Thirty-seventh Conference on Neural Information Processing
  Systems}, 2023.

\bibitem[Zheng et~al.(2022)Zheng, Ng, and Zhang]{Zheng2022}
Zheng, Y., Ng, I., and Zhang, K.
\newblock On the identifiability of nonlinear {ICA}: Sparsity and beyond.
\newblock In \emph{Advances in Neural Information Processing Systems}, 2022.

\bibitem[Zimmermann et~al.(2021)Zimmermann, Sharma, Schneider, Bethge, and
  Brendel]{Zimmermann2021}
Zimmermann, R.~S., Sharma, Y., Schneider, S., Bethge, M., and Brendel, W.
\newblock Contrastive learning inverts the data generating process.
\newblock In \emph{Proceedings of the 38th International Conference on Machine
  Learning}, volume 139, pp.\  12979--12990, 2021.

\bibitem[Zou(2006)]{Zou2006}
Zou, H.
\newblock The adaptive lasso and its oracle properties.
\newblock \emph{Journal of the American Statistical Association}, 10\penalty0
  (476):\penalty0 1418--1429, 2006.

\end{thebibliography}
\bibliographystyle{icml2024}

\newpage
\appendix
\onecolumn
\section{Supplementary Materials for \textit{\textbf Causal Representation Learning Made Identifiable by Grouping of Observational Variables}}
\addcontentsline{toc}{section}{Appendices}

\section{Exponential Family Bayesian Network Representation of the Causal Model}
\label{sec:ps_exp}

We consider a special case where the potential function $\phi$ is factorized as
\begin{align}
	\phi(x, y) = \etab(x)^\T \m T(y), \label{eq:ps_exp}
\end{align}
where the factors $\etab(x) = [\eta_1(x), \ldots, \eta_N(x)]^\T$ and 
$\m T(y) = [T_1(y), \ldots, T_N(y)]^\T$ are some $N$-dimensional vector functions of a scalar input, which are differentiable.
We assume that the factors $\etab$ and $\m T$ are minimal without loss of generality, whose definition is given below (Definition~\ref{def:minimal}).
In this parameterization, if the whole causal graph is acyclic and
the intra-group causal relations are given in the pairwise form as in those of inter-groups,
the conditional distribution of a variable $s_a^m$ is given, from Eqs.~\ref{eq:ps}, by the following (conditional) exponential family of order $N$,
\begin{align}
	p(s_a^m \lvert \text{pa}(s_a^m)) &= \frac{1}{Z_a^m(\text{pa}(s_a^m))} h_a^m(s_a^m) \exp\left( \sum_{s_b^{m'} \in \text{pa}(s_a^m)} \lambda_{ba}^{m'm} \etab(s_b^{m'})^\T \m T(s_a^m) \right), \nonumber \\
	&= \frac{1}{Z_a^m(\text{pa}(s_a^m))} h_a^m(s_a^m) \exp\left(  \tilde{\etab}_{a}^{m}(\text{pa}(s_a^m))^\T \m T(s_a^m) \right) \label{eq:p_sb}
\end{align}
where $\text{pa}(s_a^m)$ is the set of parents of the variable $s_a^m$,
$\m T(s_a^m)$ represents the sufficient statistic of the conditional distribution of $s_a^m$.
The overall natural parameter $\tilde{\etab}_a^m(\text{pa}(s_a^m)) = \sum_{s_b^{m'} \in \text{pa}(s_a^m)} \lambda_{ba}^{m'm} \etab(s_b^{m'})$ is
simply given as a summation of the causal effects from the all parents, depending on the causal strengths $\lambda_{ba}^{m'm}$ and the function $\etab$.
The base measure $h_a^m$ and the partition function $Z_a^m$ depend on the type of the factors and the graph structure.
The crucial point is the additivity of the (nonlinear) causal effects from the parents, which determines the natural parameters of the target variable distribution.
Apart from that, this parameterization is not very restrictive, since exponential families have universal approximation capabilities \citep{Sriperumbudur2017}.

This parameterization also simplifies some of the assumptions of Theorems.
The non-factorizability of $\phi$ (Assumption~\ref{Aphi} of Theorem~\ref{thm:id}) can be satisfied if the model order $N$ is more than one, as shown in Lemma~\ref{lemma:Aphi} given below.
Although this exponential family parameterization with order $N = 1$ cannot satisfy the non-factorizability, 
they can be supported by the other variant of the identifiability condition (Proposition~\ref{thm:id_alt}), which does not require such non-factorizability.
The asymmetricity assumption of $\phi$ (\ref{Aphi} and \ref{Aasym3})
indicates that the sets of the elements (functions) need to be sufficiently different between $\etab$ and $\m T$.

\paragraph{Some SEMs can be represented by Eq.~\ref{eq:p_sb}} 
We can also show that some state-equation models (SEMs) can be represented by this parameterization as a special case.
One such example is causal additive models (CAMs; \citet{Buhlmann2014}), given by
\begin{align}
	\s = \m L \betab(\s) + \ve\epsilon, \label{eq:ngsem}
\end{align}
where $\m L \in \R^{D_\mathcal{S} \times D_\mathcal{S}}$ is an adjacency matrix,
$\betab(\s) = [\beta(s_1), \ldots, \beta(s_{D_\mathcal{S}})]^\T$ is an element-wise (nonlinear) function of $\s$,
and $\ve\epsilon \sim N(\ve 0, \sigma \m I)$ is $D_\mathcal{S}$-dimensional additive Gaussian noise with diagonal covariance matrix.
In this SEM, the conditional distribution of a variable $s_a^m$ is given by
\begin{align}
	p(s_a^m \lvert \text{pa}(s_a^m)) &= \frac{1}{Z} \exp \left\{ - \frac{1}{2 \sigma^2} \left(s_a^m - \sum_{s_b^{m'} \in \text{pa}(s_a^m)} \lambda_{ba}^{m'm} \beta(s_b^{m'}) \right)^2 \right\} \label{eq:ngsem_ps} \\
	&= \left[ \frac{1}{Z}  \exp \left( - \frac{\left(\sum_{s_b^{m'} \in \text{pa}(s_a^m)} \lambda_{ba}^{m'm} \beta(s_b^{m'}) \right)^2}{2 \sigma^2} \right) \right]  \left[ \exp \left( - \frac{(s_a^m)^2}{2 \sigma^2} \right) \right] \nonumber \\
	&\times \left[ \exp \left( \sum_{s_b^{m'} \in \text{pa}(s_a^m)} \frac{\lambda_{ba}^{m'm}}{\sigma^2} s_a^m \beta(s_b^{m'}) \right) \right],
\end{align}
where $Z$ denotes a normalizing constant.
We can clearly see that the first, the second, and the third factors correspond to those of Eq.~\ref{eq:p_sb}, respectively,
and the causal function $\phi$ of this model is given by $\phi(x, y) = y \beta(x)$.
Therefore, the CAMs given by Eq.~\ref{eq:ngsem} can be represented by our causal model (Eqs.~\ref{eq:ps}), especially by the exponential family parameterization (Eqs.~\ref{eq:ps_exp} and \ref{eq:p_sb}) of order $N=1$, linear sufficient statistics $\m T(y) = y$, and (nonlinear) natural parameter $\etab(x) = \beta(x)$. 
As mentioned above, such model with order $N=1$ cannot satisfy the non-factorizability (\ref{Aphi}), and thus needs to consider Proposition~\ref{thm:id_alt} instead of Theorem~\ref{thm:id}.

This model includes linear Gaussian SEMs as a special case, where $\beta$ is a linear function.
However, they do not satisfy the conditions of both Theorem~\ref{thm:id} (uniform-dependency, non-factorizability, and asymmetricity) 
and Proposition~\ref{thm:id_alt} (uniform-dependency and asymmetricity), where the definition of uniform-dependency is given below (Definition~\ref{def:dependency}).
Thus linear Gaussian SEMs cannot have identifiability in our model in any case, 
which is consistent with the well-known result of causal discovery \citep{Hoyer2008, Peters2014}.


Note that the exponential family BNs (Eq.~\ref{eq:p_sb}) can represent many other causal models in addition to CAMs;
crucially, the distributional type of Eq.~\ref{eq:p_sb} is not restricted to Gaussian (more specifically, the sufficient statistic $\m T$ can be nonlinear and multidimensional), unlike many conventional causal models based on additive Gaussian error terms (e.g., Eq.~\ref{eq:ngsem}).

\begin{Definition} \label{def:minimal}
(Minimality) We say that a function $\alphab: \R \rightarrow \R^N$ is minimal if for any open subset $\mathcal{X}$ of $\R$ the following is true:
\begin{align}
	\left( \exists \ve\theta \in \R^N \mid \forall x \in \mathcal{X}, \ve\theta^\T \alphab(x) = \text{const.} \right) \implies \ve\theta = \ve 0.
\end{align}
\end{Definition}
The minimality is similar to the linear independence of the elements, but stronger;
minimality also forbids the existence of elements which only have differences of scaling and biases.
Note that a non-minimal model can always be reduced to minimal one via a suitable transformation and reparameterization.

\begin{Definition} \label{def:dependency}
(Uniform-dependency) We call a function $q(x, y): \bar{\mathcal{S}} \times \bar{\mathcal{S}} \rightarrow \R$ is uniform-dependent if the set of zeros of $q(x, y)$ is a meagre subset of $\bar{\mathcal{S}} \times \bar{\mathcal{S}}$, i.e., it contains no open subset. 
\end{Definition}

\begin{Lemma} \label{lemma:Aphi}
In the exponential family characterization of the causal model (Eq.~\ref{eq:ps_exp}), 
the non-factorizability conditions in Assumption~\ref{Aphi} can be satisfied if the model order $N \geq 2$.
\end{Lemma}

\begin{proof}
(Non-factorizability)
We firstly show the non-factorizability condition of the first equation $\phi^{12}(s, z_1) = c_1 \phi^{12}(s, z_2)$ in Assumption~\ref{Aphi}; the second equation can be proven in the same manner.
We give a proof by contradiction.
We suppose the negation; there exist some open subset $B \subset \bar{\mathcal{S}}$ such that the equation $\phi^{12}(s, z_1) = c_1 \phi^{12}(s, z_2)$ hold for all $s \in B$ for any $z_1 \neq z_2$.
We consider one of such open subset $B$ here. By substituting Eq.~\ref{eq:ps_exp} into the equation, we have
\begin{align}
	\etab'(s)^\T \left(\m T'(z_1) - c_1 \m T'(z_2) \right) = 0. \label{eq:lemma_Aphi_1}
\end{align}
From Lemma~3 of \citet{Khemakhem2020}, there exist $N$ distinctive values $s_1$ to $s_N$ such that $(\etab'(s_1), \ldots, \etab'(s_N))$ are linearly independent.
By substituting those values into Eq.~\ref{eq:lemma_Aphi_1} with concatenating vertically, we obtain
\begin{align}
	\begin{bmatrix} \etab'(s_1)^\T \\ \vdots \\ \etab'(s_N)^\T \end{bmatrix} \left(\m T'(z_1) - c_1 \m T'(z_2) \right) = \ve 0. \label{eq:lemma_Aphi_2}
\end{align}
Since the first factor ($N \times N$ matrix) has full-rank, we have
\begin{align}
	\m T'(z_1) - c_1 \m T'(z_2) = \ve 0. \label{eq:lemma_Aphi_3}
\end{align}
However, this contradicts the fact that there should exist at leatst two distinctive values $z_1$ and $z_2$ such that $(\m T'(z_1), \m T'(z_2))$ are linearly independent, again from Lemma~3 of \citet{Khemakhem2020}.
From this contradiction, we conclude that we can make the equation not hold by properly choosing some $z_1$ and $z_2$,
which indicates the non-factorizability.

\end{proof}

\section{Proof of Theorem~\ref{thm:id}}
\label{sec:id_proof}

\begin{proof}
We denote by $\mathcal{S} = \mathcal{S}^1 \times \ldots \times \mathcal{S}^M$ the support of the distribution of $\s$,
where $\mathcal{S}^m = \mathcal{S}_1^m \times \ldots \times \mathcal{S}_{\dsm}^m$ is that of the distribution of each group $\s^m$,
and $\mathcal{S}_a^m \subset \R$ is that of the $a$-th element.
We consider the situation where each $\mathcal{S}_a^m$ is connected (i.e. an interval),
and additionally, without loss of generality, $\mathcal{S}_a^m$ are the same across all variables, denoted as $\bar{\mathcal{S}}$.

Writing the joint log-density of the random vector $\x = (\x^1, \ldots, \x^M)$ for the two parameterizations, yields
\begin{align}
        &\log p \left(\g^1(\x^1), \ldots, \g^M(\x^M) \right) + \sum_{m \in \M} \log | \det \m J_{\g^m}(\x^m) | \nonumber \\
        &= \log \tilde{p} \left(\gbt^1(\x^1), \ldots, \gbt^M(\x^M) \right) + \sum_{m \in \M} \log | \det \m J_{\gbt^m}(\x^m) |,  \label{eq:p=pt}
\end{align}
where we denote the (true) demixing models as $\g^m = (\f^m)^{-1}$, and their other parameterizations as $\gbt^m$, and $\m J$ is the Jacobian for the change of variables.
Let a compound demixing-mixing function $\ve v^m(\s^m) = \gbt^m \circ \f^m(\s^m)$, we then have
\begin{align}
        &\log p \left(\s^1, \ldots, \s^M \right) + \sum_{m \in \M} \log | \det \m J_{\g^m}(\f^m(\s^m)) | \nonumber \\
        &= \log \tilde{p} \left(\ve v^1(\s^1), \ldots, \ve v^M(\s^M) \right) + \sum_{m \in \M} \log | \det \m J_{\gbt^m}(\f^m(\s^m)) |.  \label{eq:p=pt_v}
\end{align}
We substitute the factorization model Eq.~\ref{eq:ps} into this, and differentiate the both sides with respect to 
$s_a^m$ and $s_b^{m'}$, where $a \in \Vsm$, $b \in \Vsmp$, and obtain
\begin{align}
        &\frac{\partial^2}{\partial s_a^m \partial s_b^{m'}} \left( \lambda_{ab}^{mm'} \phi(s_a^m, s_b^{m'}) + \lambda_{ba}^{m'm} \phi(s_b^{m'}, s_a^m) \right) \nonumber\\
        &=  \frac{\partial^2}{\partial s_a^m \partial s_b^{m'}} \sum_{(i, j)} \left( \lambdat_{ij}^{mm'} \phit \left(v_i^m(\s^m), v_j^{m'}(\s^{m'}) \right) + \lambdat_{ji}^{m'm} \phit \left(v_j^{m'}(\s^{m'}), v_i^m(\s^m) \right) \right).  \label{eq:p=pt_v2_element}
\end{align}
The Jacobians disappeared here due to the grouped-observational assumption and the cross-derivatives.

By collecting the cross-derivatives for all $a \in \Vsm$ and $b \in \Vsmp$, with $a$ giving row index and $b$ the column index, we have a matrix equation of the size $\dsm \times \dsmp$,
\begin{align}
        &\left( \frac{\partial^2}{\partial s_a^m \partial s_b^{m'}} \left( \lambda_{ab}^{mm'} \phi(s_a^m, s_b^{m'}) + \lambda_{ba}^{m'm} \phi(s_b^{m'}, s_a^m) \right) \right)_{a \in \Vsm, b \in \Vsmp} \nonumber \\
        &= \m J_{\ve v^m} (\s^m)^\T \nonumber \\
        & \cdot \left( \frac{\partial^2}{\partial v_a^m \partial v_b^{m'}} \left( \lambdat_{ab}^{mm'} \phit \left(v_a^m(\s^m), v_b^{m'}(\s^{m'}) \right) + \lambdat_{ba}^{m'm} \phit \left(v_b^{m'}(\s^{m'}), v_a^m(\s^m) \right) \right) \right)_{a \in \Vsm, b \in \Vsmp} \nonumber \\
        & \cdot \m J_{\ve v^{m'}} (\s^{m'}).  \label{eq:p=pt_v2}
\end{align}
We then focus on the $a$-th row of Eq.~\ref{eq:p=pt_v2}, and differentiate each element of the both sides with respect to $s_{a'}^m$, $a' \neq a$.
Concatenating it horizontally with substituting some $K^m$ vectors $\{ \z_k^{m_k} \}_{k=1}^{K^m}$ into $\s^{m'}$, each of which is on some group $m_k \neq m$ with allowing repetitions, we have
a vector equation of the size $1 \times \sum_{k=1}^{K^m} \dsmk$
\begin{align}
	&\ve 0^\T = \begin{bmatrix} (\ve v^m)^{a \times a'}(\s^m)^\T, & (\ve v^m)^{aa'}(\s^m)^\T \end{bmatrix} \nonumber\\
	& \cdot \begin{bmatrix} \tilde{\m\Phi}^{mm_1}(\ve v^m(\s^m), \ve v^{m_1}(\z_1^{m_1})), \ldots, \tilde{\m\Phi}^{mm_{K^m}}(\ve v^m(\s^m), \ve v^{m_{K^m}}(\z_{K^m}^{m_{K^m}}) \end{bmatrix} \nonumber\\
	& \cdot \begin{bmatrix} \m J_{\ve v^{m_1}}(\z_1^{m_1}) & \m 0 & \m 0 \\ \m 0 & \ddots & \m 0 \\ \m 0 & \m 0 & \m J_{\ve v^{m_{K^m}}}(\z_{K^m}^{m_{K^m}}) \end{bmatrix}, \label{eq:p=pt_v3v}
\end{align}
where 
$(\ve v^m)^{a \times a'}(\s^m) = \begin{bmatrix} \frac{\partial}{\partial s_a^m} v_1^m(\s^m) \frac{\partial}{\partial s_{a'}^m} v_1^m(\s^m), & \ldots, & \frac{\partial}{\partial s_a^m} v_{\dsm}^m(\s^m) \frac{\partial}{\partial s_{a'}^m} v_{\dsm}^m(\s^m) \end{bmatrix}^\T$
and $(\ve v^m)^{a a'}(\s^m) = \begin{bmatrix} \frac{\partial^2}{\partial s_a^m \partial s_{a'}^m} v_1^m(\s^m), & \ldots, & \frac{\partial^2}{\partial s_a^m \partial s_{a'}^m} v_{\dsm}^m(\s^m) \end{bmatrix}^\T$
are $\dsm$ dimensional vectors,
and 
$\tilde{\m\Phi}^{mm_k}(\y^m, \y^{m_k})$ is a $2\dsm \times \dsmk$ matrix given as a collection of cross-derivatives of $\tilde{\phi}$,
\begin{align}
	&\tilde{\m\Phi}^{mm_k}(\y^m, \y^{m_k}) = \begin{bmatrix} \left( \frac{\partial^3}{\partial y_a^{m2} \partial y_b^{m_k}} \left( \lambdat_{ab}^{mm_k} \phit(y_a^m, y_b^{m_k}) + \lambdat_{ba}^{m_km} \phit(y_b^{m_k}, y_a^m) \right) \right)_{a \in \Vsm, b \in \Vsmk} \\ \left( \frac{\partial^2}{\partial y_a^m \partial y_b^{m_k}} \left( \lambdat_{ab}^{mm_k} \phit(y_a^m, y_b^{m_k}) + \lambdat_{ba}^{m_km} \phit(y_b^{m_k}, y_a^m) \right) \right)_{a \in \Vsm, b \in \Vsmk} \end{bmatrix}. \label{eq:phit}
\end{align}
The left-hand-side is now a zero-vector due to the pair-wise factorization assumption of the joint distribution (Eq.~\ref{eq:ps}).

We show here the second factor on the right-hand side of Eq.~\ref{eq:p=pt_v3v} has full row-rank
for all $\s^m \in A$ in any open subset $A$ of $\mathcal{S}^m$, by properly choosing the $K^m$ vectors $\{ \z_k^{m_k} \}_{k=1}^{K^m}$, due to the assumptions.
We differentiate each of $a$-th row of the both sides of Eq.~\ref{eq:p=pt_v2} with respect to $s_a^m$ again, and get
\begin{align}
        &\left( \frac{\partial^3}{\partial s_a^{m2} \partial s_b^{m'}} \left( \lambda_{ab}^{mm'} \phi(s_a^m, s_b^{m'}) + \lambda_{ba}^{m'm} \phi(s_b^{m'}, s_a^m) \right) \right)_{a \in \Vsm, b \in \Vsmp} \nonumber \\
        &= \begin{bmatrix} \m J_{\ve v^m} (\s^m)^\T \circ \m J_{\ve v^m} (\s^m)^\T, & \m J_{\ve v^m}^\ast (\s^m)^\T \end{bmatrix} \tilde{\m\Phi}^{mm'} \left(\ve v^m(\s^m), \ve v^{m'}(\s^{m'}) \right) \m J_{\ve v^{m'}} (\s^{m'}),  \label{eq:p=pt_v3}
\end{align}
where $\circ$ is Hadamard product, $\m J_{\ve v^m}^\ast (\s^m) = (\frac{\partial^2}{\partial s_j^2} v_i^m(\s^m))_{i, j}$ is the row-wise derivatives of the Jacobian $\m J_{\ve v^m} (\s^m) = (\frac{\partial}{\partial s_j} v_i^m(\s^m))_{i, j}$,
and $\tilde{\m\Phi}^{mm'}$ is given by Eq.~\ref{eq:phit}.
Since Eqs.~\ref{eq:p=pt_v2} and \ref{eq:p=pt_v3} have some common factors,
we can concatenate Eqs.~\ref{eq:p=pt_v2} and \ref{eq:p=pt_v3} vertically, and represent them as a single matrix equation of the size $2\dsm \times \dsmp$,
\begin{align}
        &\m\Phi^{mm'} \left(\s^m, \s^{m'} \right) \nonumber \\
        &= \begin{bmatrix} \m J_{\ve v^m} (\s^m)^\T \circ \m J_{\ve v^m} (\s^m)^\T & \m J_{\ve v^m}^\ast (\s^m)^\T \\ \m 0 & \m J_{\ve v^m} (\s^m)^\T  \end{bmatrix} \tilde{\m\Phi}^{mm'} \left(\ve v^m(\s^m), \ve v^{m'}(\s^{m'}) \right) \m J_{\ve v^{m'}} (\s^{m'}),  \label{eq:p=pt_v23}
\end{align}
where $\m\Phi^{mm'} \left(\s^m, \s^{m'} \right)$ is a $2\dsm \times \dsmp$ matrix, which has the same form as Eq.~\ref{eq:phit} and is given by
\begin{align}
	&\m\Phi^{mm'}(\y^m, \y^{m'}) = \begin{bmatrix} \left( \frac{\partial^3}{\partial y_a^{m2} \partial y_b^{m'}} \left( \lambda_{ab}^{mm'} \phi(y_a^m, y_b^{m'}) + \lambda_{ba}^{m'm} \phi(y_b^{m'}, y_a^m) \right) \right)_{a \in \Vsm, b \in \Vsmp} \\ \left( \frac{\partial^2}{\partial y_a^m \partial y_b^{m'}} \left( \lambda_{ab}^{mm'} \phi(y_a^m, y_b^{m'}) + \lambda_{ba}^{m'm} \phi(y_b^{m'}, y_a^m) \right) \right)_{a \in \Vsm, b \in \Vsmp} \end{bmatrix}. \label{eq:phi}
\end{align}
We concatenate Eq.~\ref{eq:p=pt_v23} horizontally with substituting the same vectors $\{ \z_k^{m_k} \}_{k=1}^{K^m}$ used above into $\s^{m'}$, then get a matrix equation
of the size $2\dsm \times \sum_{k=1}^{K^m} \dsmk$
\begin{align}
	&\begin{bmatrix} \m\Phi^{mm_1}(\s^m, \z_1^{m_1}), \ldots, \m\Phi^{mm_{K^m}}(\s^m, \z_{K^m}^{m_{K^m}}) \end{bmatrix} \nonumber\\
	&= \begin{bmatrix} \m J_{\ve v^m} (\s^m)^\T \circ \m J_{\ve v^m} (\s^m)^\T & \m J_{\ve v^m}^\ast (\s^m)^\T \\ \m 0 & \m J_{\ve v^m} (\s^m)^\T \end{bmatrix} \nonumber\\
	& \cdot \begin{bmatrix}\tilde{\m\Phi}^{mm_1}(\ve v^m(\s^m), \ve v^{m_1}(\z_1^{m_1})), \ldots, \tilde{\m\Phi}^{mm_{K^m}}(\ve v^m(\s^m), \ve v^{m_{K^m}}(\z_{K^m}^{m_{K^m}}) \end{bmatrix} \nonumber\\
	& \cdot \begin{bmatrix} \m J_{\ve v^{m_1}}(\z_1^{m_1}) & \m 0 & \m 0 \\ \m 0 & \ddots & \m 0 \\ \m 0 & \m 0 & \m J_{\ve v^{m_{K^m}}}(\z_{K^m}^{m_{K^m}}) \end{bmatrix}. \label{eq:p=pt_v23aav}
\end{align}
From Lemma~\ref{lemma:z} given below, we can choose the vectors $\{ \z_k^{m_k} \}_{k=1}^{K^m}$ so as to make the left-hand side has full row-rank $2\dsm$ for all $\s^m$ in any open subset of $\mathcal{S}^m$ based on the assumptions, 
which implies that the second factor (the concatenation of $\tilde{\m\Phi}^{mm_k}$) in the right-hand side has full row-rank $2\dsm$ as well.
Therefore, the second factor on the right-hand side of Eq.~\ref{eq:p=pt_v3v} has full row-rank.

Since the last term of Eq.~\ref{eq:p=pt_v3v} (collection of Jacobians) has full rank because all $\ve v^{m}$ are invertible, we can multiply the both sides by its inverse.
In addition, since the second factor of Eq.~\ref{eq:p=pt_v3v} has full row-rank due to the discussion above, 
we can multiply the both sides of Eq.~\ref{eq:p=pt_v3v} with its pseudo-inverse from the right side, and finally get
\begin{align}
	\begin{bmatrix} (\ve v^m)^{a \times a'}(\s^m)^\T, & (\ve v^m)^{aa'}(\s^m)^\T \end{bmatrix} = \m 0^\T, 
\end{align}
which is true for the all combinations of $a \neq a' \in \Vsm$.
This particularly indicates that, $\frac{\partial}{\partial s_a^m} v_j^m(\s^m) \cdot \frac{\partial}{\partial s_{a'}^m} v_j^m(\s^m) = 0$ for all $1 \le j \le \dsm$, $a \neq a'$.
This means that the Jacobian of $\ve v^m$ has at most one non-zero entry in each row.
Now, by invertibility and continuity of $\m J_{\ve v^m}$, we deduce that the location of the non-zero entries are fixed
and do not change as a function of $\s^m$. This proves that $v_a^m(\s^m)$ is represented by only one variable $s_{\sigma^m(a)}^m$ up to a scalar (variable-specific) invertible transformation for each $a \in \Vsm$, where $\sigma^m(a): \Vsm \rightarrow \Vsm$ represents a permutation of variables, which is indeterminate, and the Theorem is proven.
\end{proof}

\begin{Lemma} \label{lemma:z}
With assumptions of Theorem~\ref{thm:id}, we have the following for all group $m$ satisfying \ref{AL}:
For any open subset $A$ of $(\bar{\mathcal{S}})^\dsm$, 
there exists a set of $K^m \geq 1$ vectors $\{\z_k \in (\bar{\mathcal{S}})^\dsmk \}_{k=1}^{K^m}$, each of which belongs to some other group $m_k \neq m$ with allowing repetitions,
such that the concatenated matrix
\begin{align}
	\begin{bmatrix} \ve\Phi^{mm_1}(\s^m, \z_1), & \ldots,  & \ve\Phi^{mm_{K^m}}(\s^m, \z_{K^m}) \end{bmatrix} \label{eq:phicat}
\end{align}
with the size $2\dsm \times \sum_{k=1}^{K^m} \dsmk$ has full row-rank $2\dsm$ for all $\s^m$ in $A$.
%
%
%
%
\end{Lemma}
\begin{proof}
To show that there indeed exists such a set of vectors, we especially select here the groups $\{m_k\}$ as $\M \setminus m$ repeating twice, with some specific values of $\z_k$ for each;
more specifically, $\begin{bmatrix} m_1, &\ldots, &m_{K^m} \end{bmatrix} = \begin{bmatrix} 1, \ldots, &m-1, &m+1, \ldots, &M, & 1, \ldots, &m-1, &m+1, \ldots, & M \end{bmatrix}$, $K^m$ = $2(M-1)$, and $\z_k = \z_1 = [z_1, \ldots, z_1]^\T$ for the first half $k=1, \ldots, M-1$, and 
$\z_k = \z_2 = [z_2, \ldots, z_2]^\T$ for the second half $k=M, \ldots, 2(M-1)$ with some $z_1$ and $z_2 \in \bar{\mathcal{S}}$ (note that the size of those vectors can be different across $k$).
We denote a collection of the all inter-group adjacency coefficients related to the group $m$ as
\begin{align}
	\bar{\m L}^m = \begin{bmatrix} \m L^{m:} \\ \m L^{:m} \end{bmatrix} = \begin{bmatrix} \m L^{m1}, & \ldots, & \m L^{mM} \\ (\m L^{1m})^\T, & \ldots, & (\m L^{Mm})^\T \end{bmatrix}, \label{eq:Lcat}
\end{align}
which is a $2\dsm \times \sum_{m' \in \M \setminus m} \dsmp$ matrix given in Assumption~\ref{AL},
where $\m L^{m:}$ and $\m L^{:m}$ denote upper and lower-half matrices of $\bar{\m L}^m$, corresponding to the adjacency coefficients from group $m$ to the other groups, and those from the other groups to the group $m$, respectively.
%
Substituting those values into Eq.~\ref{eq:phicat}, we obtain a $2\dsm \times 2 \sum_{m' \in \M \setminus m} \dsmp$ matrix with a factorized form
\begin{align}
	&\begin{bmatrix} \ve\Phi^{m1}(\s^m, \z_1), & \ldots,  & \ve\Phi^{mM}(\s^m, \z_1), & \ve\Phi^{m1}(\s^m, \z_2), & \ldots,  & \ve\Phi^{mM}(\s^m, \z_2) \end{bmatrix} \nonumber \\
	& = \begin{bmatrix} \m B^{112}(\s^m, z_1), & \m B^{122}(z_1, \s^m), & \m B^{112}(\s^m, z_2), & \m B^{122}(z_2, \s^m) \\ \m B^{12}(\s^m, z_1), & \m B^{12}(z_1, \s^m), & \m B^{12}(\s^m, z_2), & \m B^{12}(z_2, \s^m) \end{bmatrix} \begin{bmatrix} \m L^{m:}, & \m 0 \\ \m L^{:m}, & \m 0 \\ \m 0, & \m L^{m:} \\ \m 0, & \m L^{:m} \end{bmatrix}, \label{eq:phicat_BL}
\end{align}
where $\m B^{112}(\s^m, z) = \text{diag}\left(\phi^{112}(s_a^m, z)\right)_{a \in \Vsm}$,
$\m B^{122}(z, \s^m) = \text{diag}\left(\phi^{122}(z, s_a^m)\right)_{a \in \Vsm}$,
$\m B^{12}(\s^m, z) = \text{diag}\left(\phi^{12}(s_a^m, z)\right)_{a \in \Vsm}$,
and $\m B^{12}(z, \s^m) = \text{diag}\left(\phi^{12}(z, s_a^m)\right)_{a \in \Vsm}$.
Note that those cross-derivatives of the function $\phi$ have uniform-dependency from Assumption~\ref{Aphi} (Definition~\ref{def:dependency}).

Considering that the adjacency matrix $\bar{\m L}^m$ possibly has some rows with all-zeros, depending on the graph structure,
we explicitly divide the set of latent variable indices $\Vsm$ into three groups $[\V_{\text{b}}, \V_{\text{p}}, \V_{\text{c}}]$ (we omit the group index $m$ for simplicity here);
the variables with indices $\V_{\text{b}}$ are both parents and children of some variables in some other group, 
the variables with $\V_{\text{p}}$ are parents (but not children) of some variable in some other group,
and the variables with $\V_{\text{c}}$ are children (but not parents) of some variable in some other group.
We assume without loss of generality that the variable indices $\Vsm$ are sorted in the order $[\V_{\text{b}}, \V_{\text{p}}, \V_{\text{c}}]$.
Eq.~\ref{eq:phicat_BL} can then be re-written as
\begin{align}
	&\begin{bmatrix} \ve\Phi^{m1}(\s^m, \z_1), & \ldots,  & \ve\Phi^{mM}(\s^m, \z_1), & \ve\Phi^{m1}(\s^m, \z_2), & \ldots,  & \ve\Phi^{mM}(\s^m, \z_2) \end{bmatrix}  \nonumber \\
	& = \left[ \begin{matrix} \m B^{112}(\s_\Vb^m, z_1), & \m 0, & \m B^{122}(z_1, \s_\Vb^m), & \m 0,  \\
		\m 0, &  \m B^{112}(\s_\Vp^m, z_1), & \m 0, & \m 0, \\
		\m 0, &  \m 0, & \m 0, & \m B^{122}(z_1, \s_\Vc^m), \\
		\m B^{12}(\s_\Vb^m, z_1), & \m 0, & \m B^{12}(z_1, \s_\Vb^m), & \m 0,  \\
		\m 0, &  \m B^{12}(\s_\Vp^m, z_1), & \m 0, & \m 0, \\
		\m 0, &  \m 0, & \m 0, & \m B^{12}(z_1, \s_\Vc^m), \end{matrix} \right. \nonumber \\
		& \quad \quad \left. \begin{matrix} \m B^{112}(\s_\Vb^m, z_2), & \m 0, & \m B^{122}(z_2, \s_\Vb^m), & \m 0  \\
		\m 0, &  \m B^{112}(\s_\Vp^m, z_2), & \m 0, & \m 0 \\
		\m 0, &  \m 0, & \m 0, & \m B^{122}(z_2, \s_\Vc^m) \\
		\m B^{12}(\s_\Vb^m, z_2), & \m 0, & \m B^{12}(z_2, \s_\Vb^m), & \m 0  \\
		\m 0, &  \m B^{12}(\s_\Vp^m, z_2), & \m 0, & \m 0 \\
		\m 0, &  \m 0, & \m 0, & \m B^{12}(z_2, \s_\Vc^m) \end{matrix} \right]
		 \begin{bmatrix} \m L_\Vb^{m:}, & \m 0 \\ \m L_\Vp^{m:}, & \m 0 \\ \m L_\Vb^{:m}, & \m 0 \\ \m L_\Vc^{:m}, & \m 0 \\ \m 0, & \m L_\Vb^{m:} \\ \m 0, & \m L_\Vp^{m:} \\ \m 0, & \m L_\Vb^{:m} \\ \m 0, & \m L_\Vc^{:m} \end{bmatrix}, \label{eq:phicat_BLsep}
\end{align}
where $\m L_\Vb^{m:}$ denotes a submatrix of $\m L^{m:}$ corresponding to the indices (rows) $\Vb$,
and similarly for $\m L_\Vp^{m:}$, $\m L_\Vb^{:m}$, and $\m L_\Vc^{:m}$.
Now the second factor of the right-hand side has the size $2 (2 | \Vb | + | \Vp | + | \Vc |) \times 2 \sum_{m' \in \M \setminus m} \dsmp$, and has full row-rank from the assumption (note that the number of rows of this factor is lower than that in Eq.~\ref{eq:phicat_BL} since we removed all-zero rows).
Since the number of rows of the first factor ($2\dsm$) is always smaller or equal to that of the second factor, $2 (2 | \Vb | + | \Vp | + | \Vc |) \geq 2\dsm$,
what we need to show for this Lemma is the full row-rankness of the first factor.
From its structure, we can show this separately for each subset of its rows corresponding to $\V_{\text{b}}$, $\V_{\text{p}}$, and $\V_{\text{c}}$.

\paragraph{The Rows Corresponding to $\Vb$} We start from the submatrix (rows) corresponding to $\Vb$. 
We especially consider the $2 | \Vb | \times 2 | \Vb |$ submatrix given by
\begin{align}
	\begin{bmatrix} \m B^{112}(\s_\Vb^m, z_1), & \m B^{122}(z_1, \s_\Vb^m) \\
	\m B^{12}(\s_\Vb^m, z_1), & \m B^{12}(z_1, \s_\Vb^m) \end{bmatrix}. \label{eq:B_vb}
\end{align}
If this submatrix has full-rank, the submatrix (rows) of the first factor of Eq.~\ref{eq:phicat_BLsep} corresponding to $\Vb$ also has full row-rank.
Considering that the matrices $\m B^{112}$, $\m B^{122}$, and $\m B^{12}$ are all diagonal, we focus on a $2 \times 2$ submatrix corresponding to a variable index $a \in \Vb$, given by
\begin{align}
	\begin{bmatrix} \phi^{112}(s_a^m, z_1) & \phi^{122}(z_1, s_a^m) \\
	\phi^{12}(s_a^m, z_1) & \phi^{12}(z_1, s_a^m) \end{bmatrix}. \label{eq:vb_2x2}
\end{align}
Calculating the determinant, with the uniform dependency assumption of the all elements (Assumption~\ref{Aphi}), this submatrix has full-rank (non-zero determinant) if the following condition does not hold:
\begin{align}
	&\frac{\phi^{112}(s_a^m, z_1)}{\phi^{12}(s_a^m, z_1)} = \frac{\phi^{122}(z_1, s_a^m)}{\phi^{12}(z_1, s_a^m)} \nonumber \\
	\implies &\frac{\partial}{\partial s_a^m} \log \lvert \phi^{12}(s_a^m, z_1) \rvert = \frac{\partial}{\partial s_a^m} \log \lvert \phi^{12}(z_1, s_a^m) \rvert \nonumber \\
	\implies &\phi^{12}(s_a^m, z_1) = c(z_1) \phi^{12}(z_1, s_a^m),
\end{align}
for all $s_a^m$ with some constant $c(z_1)$ not dependent on $s_a^m$. 
This is exactly the condition assumed in Assumption~\ref{Aphi}.
Since this is true for each $2 \times 2$ submatrices corresponding to all $a \in \Vb$,
we conclude that the matrix Eq.~\ref{eq:B_vb} has full-rank, and thus the submatrix (rows) corresponding to $\Vb$ has full row-rank $2 | \Vb |$.

\paragraph{The Rows Corresponding to $\Vp$} We next show the full row-rankness of the submatrix (rows) corresponding to $\Vp$.
We especially consider the $2 | \Vp | \times 2 | \Vp |$ submatrix given by
\begin{align}
	\begin{bmatrix} \m B^{112}(\s_\Vp^m, z_1), & \m B^{112}(\s_\Vp^m, z_2) \\
	 \m B^{12}(\s_\Vp^m, z_1), & \m B^{12}(\s_\Vp^m, z_2) \end{bmatrix}. \label{eq:B_vp}
\end{align}
If this submatrix has full-rank, the submatrix (rows) of the first factor of Eq.~\ref{eq:phicat_BLsep} corresponding to $\Vp$ also has full row-rank.
We focus on a $2 \times 2$ submatrix corresponding to a variable index $a \in \Vp$, given by
\begin{align}
	\begin{bmatrix} \phi^{112}(s_a^m, z_1) & \phi^{112}(s_a^m, z_2) \\
	\phi^{12}(s_a^m, z_1) & \phi^{12}(s_a^m, z_2) \end{bmatrix}. \label{eq:vp_2x2}
\end{align}
Calculating the determinant, with the uniform dependency assumption of the all elements (Assumption~\ref{Aphi}), this submatrix has full-rank (non-zero determinant) if the following condition does not hold:
\begin{align}
	&\frac{\phi^{112}(s_a^m, z_1)}{\phi^{12}(s_a^m, z_1)} = \frac{\phi^{112}(s_a^m, z_2)}{\phi^{12}(s_a^m, z_2)} \nonumber \\
	\implies &\frac{\partial}{\partial s_a^m} \log \lvert \phi^{12}(s_a^m, z_1) \rvert = \frac{\partial}{\partial s_a^m} \log \lvert \phi^{12}(s_a^m, z_2) \rvert \nonumber \\
	\implies &\phi^{12}(s_a^m, z_1) = c(z_1, z_2) \phi^{12}(s_a^m, z_2),
\end{align}
for all $s_a^m$ with some constant $c(z_1, z_2)$ not dependent on $s_a^m$. 
This is exactly the condition assumed in Assumption~\ref{Aphi}.
Since this is true for each $2 \times 2$ submatrices corresponding to the all $a \in \Vp$,
we conclude that the matrix Eq.~\ref{eq:B_vp} has full-rank, and thus the submatrix (rows) corresponding to $\Vp$ has full row-rank $2 | \Vp |$.

\paragraph{The Rows Corresponding to $\Vc$} We lastly show the full row-rankness of the submatrix (rows) corresponding to $\Vc$.
With the similar discussion to that for the rows $\Vp$ given above,
this submatrix has full row-rank $2 | \Vc |$ if the following condition does not hold:
\begin{align}
	\phi^{12}(z_1, s_a^m) = c(z_1, z_2) \phi^{12}(z_2, s_a^m),
\end{align}
for all $s_a^m$ with some constant $c(z_1, z_2)$ not dependent on $s_a^m$. 
This is exactly the condition assumed in Assumption~\ref{Aphi}.

Combining the all results above, we finally conclude that the first factor in Eq.~\ref{eq:phicat_BLsep} has full row-rank $2 \dsm$ ($= 2 | \Vb | + 2 | \Vp | + 2 | \Vc |$).
This indicates that the right-hand of Eq.~\ref{eq:phicat_BLsep} has full row-rank, and so does the left-hand side.
Then the Lemma is proven.

\end{proof}

\section{Alternative Identifiability Condition of Theorem~\ref{thm:id}}
\label{sec:id_alt}

In Theorem~\ref{thm:id}, we can weaken the constraints on the causal function $\phi$ (Assumption~\ref{Aphi}) by strengthening that on the causal graph (\ref{AL}),
especially by assuming that all variables have both parent and child in some other group.
The alternative conditions of Theorem~\ref{thm:id} is given in the following Proposition:
\begin{Proposition} \label{thm:id_alt}
Assume the generative model given by Eqs.~\ref{eq:f} and \ref{eq:ps}, and also the following:
\begin{enumerate}[label=A'\arabic*]\vspace*{-2mm}
\item (Nondegeneracy of the graph) \label{ALalt} 
For any group $m$ in $\M$ (or in a subset of $\M$, that we call ``the groups of interest"),
each variable has both a (at least one) parent and child in some other group,
and the collection of inter-group adjacency matrices $\bar{\m L}^m$ given below has full row-rank:
\begin{align}
	\bar{\m L}^m =\begin{bmatrix} \m L^{m1}, & \ldots, & \m L^{mM} \\ (\m L^{1m})^\T, & \ldots, & (\m L^{Mm})^\T \end{bmatrix}, \label{eq:Lg2}
\end{align} 
\item (Causal function) $\phi^{12}$ has uniform dependency, and either of the following conditions is satisfied:\label{Aphic}
\begin{enumerate}
\item \label{Aphic1} Both $\phi^{112}$ and $\phi^{122}$ have uniform dependency, 
and for any open subset $B$ of $\bar{\mathcal{S}}$, there exist some $z \in \bar{\mathcal{S}}$
such that the following condition does not hold for $\phi^{12}$: $\phi^{12}(s, z) = c \phi^{12}(z, s)$ with some constants $c \in \R$ for all $s \in B$.
\item Either one of $\phi^{112}$ and $\phi^{122}$ has uniform dependency, and the other one is constantly zero. \label{Aphic2}
\end{enumerate}
\end{enumerate}\vspace*{-2mm}
Then, for all groups $m$ in $\M$ (or in the groups of interests), $\s^m$ can be recovered up to permutation and variable-wise invertible transformations from the distribution of the observations $\x$.
\end{Proposition}

Assumption~\ref{ALalt} is similar to \ref{AL}, while it requires all variables to have both parent and child.
Note that in this case the matrix $\bar{\m L}^m$ has at least one non-zero element for each row, and thus does not have all-zero rows.
Although \ref{ALalt} imply that the whole causal graph cannot be acyclic,
this does not mean we cannot identify any of the latent variables in acyclic graphs; since \ref{ALalt} is {\it group-wise}, we can still identify the variables on the groups (groups of interest) satisfying them. 
For example, in the illustrative example of the causal dynamics (Illustrative Example~2), 
we cannot identify the latent variables on the first (no-parents) and the last (no-children) time points (groups),
while we would be able to identify the other time-points (groups) $t$ 
since they usually obtain significant causal effects from nearby preceding time-points (groups) $<t$, then cause the other nearby subsequent time-points (groups) $>t$. 

The assumptions on the function $\phi$ (\ref{Aphic}) is weaker than that of Theorem~\ref{thm:id} (\ref{Aphi}) since they do not require the factorizability of $\phi$ and also uniform-dependency of either $\phi^{112}$ or $\phi^{122}$ in \ref{Aphic2}.
For example, Gaussian CAMs (Eq.~\ref{eq:ngsem}; $N=1$ and linear $\m T$ in Eq.~\ref{eq:p_sb}) are allowed in \ref{Aphic2},
while they are not accepted in Theorem~\ref{thm:id}.

\begin{proof}
Proof is basically the same as that of Theorem~\ref{thm:id} (Supplementary Material~\ref{sec:id_proof}),
while we only need to consider the rows $\Vb$ in Lemma~\ref{lemma:z} since all of the variables have both parent and child in some other group here.
We thus only need to show the full row-rankness of Eq.~\ref{eq:vb_2x2} in Lemma~\ref{lemma:z}, corresponding to the rows $\Vb$, under the condition~\ref{Aphic1} or \ref{Aphic2}.
Condition~\ref{Aphic1} represents the asymmetricity, which is actually the same as that assumed in Theorem~\ref{thm:id}, and thus can be proven by the same discussion given in the proof of Lemma~\ref{lemma:z}.
We can also easily see the full row-rankness of Eq.~\ref{eq:vb_2x2} under Condition~\ref{Aphic2} since the determinant of the matrix (Eq.~\ref{eq:vb_2x2}) is non-zero from the uniform dependency assumptions.

Then the Proposition is proven.
\end{proof}

\section{Proof of Theorem~~\ref{thm:crl}}
\label{sec:crl_proof}

By well-known theory \citep{Gutmann2012,Hastie2001}, 
after convergence of logistic regression, with infinite data and a function approximator with universal approximation capability, 
the regression function (Eq.~\ref{eq:crl_r}) will equal the difference of the log-pdfs in the two classes $\x^{(n)}$ and $\x^{(n_\ast)}$ in Eq.~\ref{eq:xtilde}:
\begin{align}
	&\sum_{m \in \M} \bar{\psi}^m(\h^m(\x^{m})) + \sum_{m \neq m'} \sum_{(a, b) \in \Vsm \times \Vsmp} w_{ab}^{mm'} \psi(h_a^m(\x^{m}), h_b^{m'}(\x^{m'})) + c \nonumber \\
	&= \log p_\x(\x^1, \ldots, \x^M) - \log p_{\x^\ast}(\x^1, \ldots, \x^M) \nonumber \\
	&= \log p(\g^1(\x^1), \ldots, \g^M(\x^M)) - \log p_{\s^\ast}(\g^1(\x^1), \ldots, \g^M(\x^M)) \nonumber \\
        &= \log p(\g^1(\x^1), \ldots, \g^M(\x^M)) -  \sum_{m \in \M} \log p^m(\g^m(\x^m)),\label{eq:crl_eq}
\end{align}
where $p_\x$, $p_{\x^\ast}$, and $p_{\s^\ast}$ are the joint densities of the observational vector $\x^{(n)}$ (the first dataset in Eq.~\ref{eq:xtilde}), observational vector with randomized samples for each group $\x^{(n_\ast)}$ (the second dataset in Eq.~\ref{eq:xtilde}), and that on the latent space $\s^{(n_\ast)}$, respectively,
and $p^m$ is the marginal distribution of the $m$-th latent variable group,
$\g^m = (\f^m)^{-1}$ are the (true) demixing models. 
The second equation comes from the well-known theory that the changes of variables do not change the density-ratio (subtraction of log-densities; the Jacobians for the changes of variables cancel out), and the third equation comes from the fact that there is no causal relations across groups on the shuffled dataset 
because the samples are obtained randomly and independently for each group (while causal relations can still exist within each group, implicitly involved in $p^m$).

Let a compound demixing-mixing function $\ve v^m(\s^m) = \h^m \circ \f^m(\s^m)$, we then have
\begin{align}
        &\log p(\s^1, \ldots, \s^M) -  \sum_{m \in \M} \log p^m(\s^m) \nonumber \\
        &= \sum_{m \in \M} \bar{\psi}^m(\ve v^m(\s^{m})) + \sum_{m \neq m'} \sum_{(a, b) \in \Vsm \times \Vsmp} w_{ab}^{mm'} \psi(v_a^m(\s^{m}), v_b^{m'}(\s^{m'})) + c.  \label{eq:crl_p=pt_v}
\end{align}
We substitute the factorization model Eq.~\ref{eq:ps} into this, and differentiate the both sides with respect to 
$s_a^m$ and $s_b^{m'}$, where $a \in \Vsm$, $b \in \Vsmp$, $m \neq m'$, and then obtain,
\begin{align}
        &\frac{\partial^2}{\partial s_a^m \partial s_b^{m'}} \left( \lambda_{ab}^{mm'} \phi(s_a^m, s_b^{m'}) + \lambda_{ba}^{m'm} \phi(s_b^{m'}, s_a^m) \right) \nonumber\\
        &=  \frac{\partial^2}{\partial s_a^m \partial s_b^{m'}} \sum_{(i, j)} \left( w_{ij}^{mm'} \psi \left(v_i^m(\s^m), v_j^{m'}(\s^{m'}) \right) + w_{ji}^{m'm} \psi \left(v_j^{m'}(\s^{m'}), v_i^m(\s^m) \right) \right).  \label{eq:crl_p=pt_v2_element}
\end{align}
Now compare this equation to Eq.~\ref{eq:p=pt_v2_element} of the proof of Theorem~\ref{thm:id} in Supplementary Material~\ref{sec:id_proof}.
The functions $\psi$ and $\phit$, and the coefficients $\lambdat_{ij}^{mm'}$ and $w_{ij}^{mm'}$ denote the same things in the two proofs.
Now, we can proceed with the proof of Theorem~\ref{thm:id}, 
and the consistency of the estimation framework is thus proven.

\section{Proof of Theorem~\ref{thm:cd3}}
\label{sec:cd_proof}

%
%
\begin{proof}
From the result of Theorem~\ref{thm:id} with the required assumptions, 
for each $m \in \M$, we so far have the identifiability of the latent variables $(\s^m)$ up to variable-wise nonlinear scalings and a permutation;
i.e., the compound function $\ve v^m$ in the proof of Theorem~\ref{thm:id} (Supplementary Material~\ref{sec:id_proof}) is given by, for each element,
\begin{align}
	v_a^m(\s^m) = k_{\sigma^m(a)}^m(s_{\sigma^m(a)}^m), \label{eq:k_exp}
\end{align}
where $k_{\sigma^m(a)}^m: \R \rightarrow \R$ is a scalar invertible functions,
and $\sigma^m(a): \Vsm \rightarrow \Vsm$ represents the permutation of variables, which are indeterminate according to Theorem~\ref{thm:id}.
Without loss of generality, we assume that the variables were sorted properly ($\sigma^m(a) = a$),
and the nonlinear functions were scaled properly so that the image is embedded on the same space (interval) to that of the input (i.e. $s_a^m \in \mathcal{S}_a^m \rightarrow k_a^m(s_a^m) \in \mathcal{S}_a^m$).

We now discuss identifiability of the model (Eq.~\ref{eq:ps}) by considering two sets of parameters $\theta = \{\m L, \phi\}$ (true) and 
$\tilde{\theta} = \{\tilde{\m L}, \phit \}$ (another parameterization or estimate) satisfying the assumptions of the Theorem,
such that they both give the same data distributions $p(\x; \theta) = p(\x; \tilde{\theta})$.

\paragraph{Resolving the Element-wise Nonlinear Scaling:}
We first show that the element-wise (nonlinear) scaling $k_a^m$ can be resolved to some extent by some additional assumptions given in Theorem~\ref{thm:cd3};
more specifically, the scaling $k_a^m$ can be given by the same function $k^m$ for each group $m$, rather than the variable-wise manner (Eq.~\ref{eq:k_exp}).
We focus on co-parents $s_a^m$ and $s_{a'}^m \in \Vsm$ and their child $s_b^{m_\ast} \in \Vsmast$, assumed in Assumption~\ref{An3}.
We then have the $(a, b)$-th element of Eq.~\ref{eq:p=pt_v2} (causal relation between $s_a^m$ and $s_b^{m_\ast}$) with substituting Eq.~\ref{eq:k_exp},
\begin{align}
	&\frac{\partial^2}{\partial s_a^m \partial s_b^{m_\ast}} \left( \lambda_{ab}^{mm_\ast} \phi(s_a^m, s_b^{m_\ast}) + \lambda_{ba}^{m_\ast m} \phi(s_b^{m_\ast}, s_a^{m}) \right) \nonumber \\
        &= \frac{\partial^2}{\partial s_a^m \partial s_b^{m_\ast}} \left( \lambdat_{ab}^{mm_\ast} \phit(k_a^m(s_a^m), k_b^{m_\ast}(s_b^{m_\ast}))
        + \lambdat_{ba}^{m_\ast m} \phit(k_b^{m_\ast}(s_b^{m_\ast}), k_a^m(s_a^m)) \right), \label{eq:cd_eq12_exp3}
\end{align}
and likewise the $(a', b)$-th element (causal relation between variables $s_{a'}^m$ and $s_{b}^{m_\ast}$).

From Assumption~\ref{An3}, $\lambda_{ab}^{mm_\ast} \neq 0$ and $\lambda_{a'b}^{mm_\ast} \neq 0$ on the left-hand side (true parameter $\theta$; the opposite directions are zeros $\lambda_{ba}^{m_\ast m} =  \lambda_{ba'}^{m_\ast m} =0$ since the graph is directed; Assumption~\ref{An3}).
By taking a division of Eq.~\ref{eq:cd_eq12_exp3} corresponding to those two variable-pairs, 
which is possible thanks to the uniform-dependency of the cross-derivatives of the functions (Assumption~\ref{Aphi}), we obtain four possible equations, depending on
which combination between $(\lambdat_{ab}^{mm_\ast}, \lambdat_{ba}^{m_\ast m})$ and  $(\lambdat_{a'b}^{mm_\ast}, \lambdat_{ba'}^{m_\ast m})$
 has non-zero values on the right-hand side;
\begin{align}
  &\frac{\lambda_{ab}^{mm_\ast} \frac{\partial^2}{\partial s_a^m \partial s_b^{m_\ast}} \phi(s_a^m, s_b^{m_\ast})}{\lambda_{a'b}^{mm_\ast} \frac{\partial^2}{\partial s_{a'}^m \partial s_{b}^{m_\ast}}\phi(s_{a'}^m, s_{b}^{m_\ast})} \nonumber \\
  &= \begin{cases}
    \frac{\lambdat_{ab}^{mm_\ast} \frac{\partial^2}{\partial s_a^m \partial s_b^{m_\ast}} \phit(k_a^m(s_a^m), k_b^{m_\ast}(s_b^{m_\ast}))}{\lambdat_{a'b}^{mm_\ast} \frac{\partial^2}{\partial s_{a'}^m \partial s_{b}^{m_\ast}} \phit(k_{a'}^m(s_{a'}^m), k_{b}^{m_\ast}(s_{b}^{m_\ast}))} & (\lambdat_{ab}^{mm_\ast} \lambdat_{a'b}^{mm_\ast} \neq 0), \\
    \frac{\lambdat_{ba}^{m_\ast m} \frac{\partial^2}{\partial s_a^m \partial s_b^{m_\ast}}  \phit(k_b^{m_\ast}(s_b^{m_\ast}), k_a^m(s_a^m))}{\lambdat_{ba'}^{m_\ast m} \frac{\partial^2}{\partial s_{a'}^m \partial s_{b}^{m_\ast}}  \phit(k_{b}^{m_\ast}(s_{b}^{m_\ast}), k_{a'}^m(s_{a'}^m))} & (\lambdat_{ba}^{m_\ast m} \lambdat_{ba'}^{m_\ast m} \neq 0), \\
    \frac{\lambdat_{ab}^{mm_\ast} \frac{\partial^2}{\partial s_a^m \partial s_b^{m_\ast}} \phit(k_a^m(s_a^m), k_b^{m_\ast}(s_b^{m_\ast}))}{\lambdat_{ba'}^{m_\ast m} \frac{\partial^2}{\partial s_{a'}^m \partial s_{b}^{m_\ast}} \phit(k_{b}^{m_\ast}(s_{b}^{m_\ast}), k_{a'}^m(s_{a'}^m))} & (\lambdat_{ab}^{mm_\ast} \lambdat_{ba'}^{m_\ast m} \neq 0), \\
    \frac{\lambdat_{ba}^{m_\ast m} \frac{\partial^2}{\partial s_a^m \partial s_b^{m_\ast}} \phit(k_b^{m_\ast}(s_b^{m_\ast}), k_a^m(s_a^m))}{\lambdat_{a'b}^{mm_\ast} \frac{\partial^2}{\partial s_{a'}^m \partial s_{b}^{m_\ast}} \phit(k_{a'}^m(s_{a'}^m), k_{b}^{m_\ast}(s_{b}^{m_\ast}))} & (\lambdat_{ba}^{m_\ast m} \lambdat_{a'b}^{mm_\ast} \neq 0). \label{eq:cd_eq12vec_div_exp3}
  \end{cases}
\end{align}
On the right-hand side (estimate $\tilde{\theta}$), only one of them is possible due to the directed causal graph assumption (Assumption~\ref{An3}).
The first two cases are when the causal directions are the same between the two variable-pairs on the parameterization $\tilde{\theta}$, 
similarly to $\theta$ (but possibly both flipped from $\theta$), while they are opposite each other in the latter two cases.

We first show that the latter two cases of Eq.~\ref{eq:cd_eq12vec_div_exp3} (opposite causal directions between the two pairs $(a, b)$ and $(a', b)$) contradict the assumptions, as we expected.
We replace $s_a^m$ and $s_{a'}^m$ by a common variable $y_1 \in \bar{\mathcal{S}}$
(this is possible because we consider the case where the supports of the all latent variables are the same, denoted as $\bar{\mathcal{S}}$),
and $s_b^{m_\ast}$ by $y_2 \in \bar{\mathcal{S}}$, then obtain
\begin{align}
  \frac{\lambda_{ab}^{mm_\ast}}{\lambda_{a'b}^{mm_\ast}} 
  = \begin{cases}
    \frac{\lambdat_{ab}^{mm_\ast} \frac{\partial^2}{\partial y_1 \partial y_2} \phit(k_a^m(y_1), k_b^{m_\ast}(y_2))}{\lambdat_{ba'}^{m_\ast m} \frac{\partial^2}{\partial y_1 \partial y_2} \phit(k_{b}^{m_\ast}(y_2), k_{a'}^m(y_1))} & (\lambdat_{ab}^{mm_\ast} \lambdat_{ba'}^{m_\ast m} \neq 0), \\
    \frac{\lambdat_{ba}^{m_\ast m} \frac{\partial^2}{\partial y_1 \partial y_2} \phit(k_b^{m_\ast}(y_2), k_a^m(y_1))}{\lambdat_{a'b}^{mm_\ast} \frac{\partial^2}{\partial y_1 \partial y_2} \phit(k_{a'}^m(y_1), k_{b}^{m_\ast}(y_2))} & (\lambdat_{ba}^{m_\ast m} \lambdat_{a'b}^{mm_\ast} \neq 0), \label{eq:cd_eq12vec_div_y_34_exp3}
  \end{cases}
\end{align}
where the left-hand-side is constant.
However, Lemma~\ref{lemma:4} given below indicates that these contradict the assumptions,
and thus the latter two cases of Eq~\ref{eq:cd_eq12vec_div_exp3} are indeed excluded.

On the other hand, in the first two cases of Eq~\ref{eq:cd_eq12vec_div_exp3}, we again replace the variables by $y_1$ and $y_2$, then obtain
\begin{align}
  \frac{\lambda_{ab}^{mm_\ast} }{\lambda_{a'b}^{mm_\ast}} 
  = \begin{cases}
    \frac{\lambdat_{ab}^{mm_\ast} \frac{\partial^2}{\partial y_1 \partial y_2} \phit(k_a^m(y_1), k_b^{m_\ast}(y_2))}{\lambdat_{a'b}^{mm_\ast} \frac{\partial^2}{\partial y_1 \partial y_2} \phit(k_{a'}^m(y_1), k_{b}^{m_\ast}(y_2))} & (\lambdat_{ab}^{mm_\ast} \lambdat_{a'b}^{mm_\ast} \neq 0), \\
    \frac{\lambdat_{ba}^{m_\ast m} \frac{\partial^2}{\partial y_1 \partial y_2}  \phit(k_b^{m_\ast}(y_2), k_a^m(y_1))}{\lambdat_{ba'}^{m_\ast m} \frac{\partial^2}{\partial y_1 \partial y_2}  \phit(k_{b}^{m_\ast}(y_2), k_{a'}^m(y_1))} & (\lambdat_{ba}^{m_\ast m} \lambdat_{ba'}^{m_\ast m} \neq 0). \label{eq:cd_eq12vec_div_y_12_exp3}
  \end{cases}
\end{align}
We now show that those equations are possible only when $k_a^m = k_{a'}^m$ due to the assumptions.
From Assumption~\ref{An3}, there exists a path from a variable to any other variable by following the co-parents on group $m$,
and for each co-parents we have either one of the cases in Eq.~\ref{eq:cd_eq12vec_div_y_12_exp3}.
However, once whether the former or the latter case of Eq.~\ref{eq:cd_eq12vec_div_y_12_exp3} is determined for some co-parent,
all other co-parents also need to have the same side of the equation, since the existence of inconsistent causal directions is not allowed due to Lemma~\ref{lemma:4}.
This indicates that we have a relation of either 
\begin{align}
	\frac{\partial^2}{\partial y_1 \partial y_2} \phit \left(k_a^m(y_1), k_b^{m_\ast}(y_2) \right) = \alpha_{1aa'} \frac{\partial^2}{\partial y_1 \partial y_2} \phit \left(k_{a'}^m(y_1), k_{b}^{m_\ast}(y_2) \right), \nonumber \\
	\frac{\partial^2}{\partial y_1 \partial y_2} \phit \left(k_b^{m_\ast}(y_2), k_a^m(y_1) \right) = \beta_{1aa'} \frac{\partial^2}{\partial y_1 \partial y_2} \phit \left(k_{b}^{m_\ast}(y_2), k_{a'}^m(y_1) \right), \label{eq:cd3_eq_copa_2}
\end{align}
consistently for all co-parents $(a, a') \in \Vsm \times \Vsm$ assumed in \ref{An3}, where $\alpha_{1aa'}$ and $\beta_{1aa'}$ are some scalar constants depending on the co-parents.

We next consider the co-children $s_a^m$ and $s_{a''}^m$ and their parent $s_{b'}^{m_\dagger}$, assumed in Assumption~\ref{An3} ($m_\dagger$ does not need to be same as $m_\ast$). Based on the same discussions for the co-parents given above, we have a relation of either 
\begin{align}
	\frac{\partial^2}{\partial y_1 \partial y_2} \phit \left(k_{b'}^{m_\dagger}(y_2), k_a^m(y_1) \right) = \alpha_{2aa''} \frac{\partial^2}{\partial y_1 \partial y_2} \phit \left(k_{b'}^{m_\ast}(y_2), k_{a''}^m(y_1) \right), \nonumber \\
	\frac{\partial^2}{\partial y_1 \partial y_2} \phit \left(k_a^m(y_1), k_{b'}^{m_\dagger}(y_2) \right) = \beta_{2aa''} \frac{\partial^2}{\partial y_1 \partial y_2} \phit \left(k_{a''}^m(y_1), k_{b'}^{m_\dagger}(y_2) \right), \label{eq:cd3_eq_coch_2}
\end{align}
consistently for all co-children $(a, a'') \in \Vsm \times \Vsm$ assumed in \ref{An3}, where $\alpha_{2aa'}$ and $\beta_{2aa'}$ are some scalar constants depending on the co-children.
The former and the latter cases in Eqs.~\ref{eq:cd3_eq_copa_2} and \ref{eq:cd3_eq_coch_2} correspond to each other;
if we have the first case of Eq.~\ref{eq:cd3_eq_copa_2}, we also have the first case of Eq.~\ref{eq:cd3_eq_coch_2}, excluding the second cases, and vice versa.
We show this by contradiction;
we suppose there exist three variables with causal relation $s_{b'}^{m_\dagger} \rightarrow s_a^m \rightarrow s_b^{m_\ast}$ as assumed in \ref{An3} on the true model $\theta$,
but they are wrongly learned as $s_{b'}^{m_\dagger} \leftarrow s_a^m \rightarrow s_b^{m_\ast}$ on the estimate $\tilde{\theta}$ (the following discussions are the same when they are  learned as $s_{b'}^{m_\dagger} \rightarrow s_a^m \leftarrow s_b^{m_\ast}$ as well).
This gives us two equations from Eq.~\ref{eq:cd_eq12_exp3},
\begin{align}
	&\lambda_{b'a}^{m_\dagger m} \frac{\partial^2}{\partial s_a^m \partial s_{b'}^{m_\dagger}} \phi(s_{b'}^{m_\dagger}, s_a^m) = \lambdat_{ab'}^{mm_\dagger} \frac{\partial^2}{\partial s_a^m \partial s_{b'}^{m_\dagger}} \phit(k_a^m(s_a^m), k_{b'}^{m_\dagger}(s_{b'}^{m_\dagger})), \nonumber \\
	&\lambda_{ab}^{mm_\ast} \frac{\partial^2}{\partial s_a^m \partial s_b^{m_\ast}} \phi(s_a^m, s_b^{m_\ast}) = \lambdat_{ab}^{mm_\ast} \frac{\partial^2}{\partial s_a^m \partial s_b^{m_\ast}} \phit(k_a^m(s_a^m), k_b^{m_\ast}(s_b^{m_\ast})).
\end{align}
We substitute $s_{b'}^{m_\dagger}$ with $(k_{b'}^{m_\dagger})^{-1} \circ k_b^{m_\ast}(s_b^{m_\ast})$, which makes the right-hand side the same up to scaling, after applying the change of variables and canceling out the derivatives of the functions $k$ on the both sides. However, this contradicts Lemma~\ref{lemma:4} given below.
This indicates that the relations of {\it parents} and {\it children} from a variable are preserved between $\theta$ and $\tilde{\theta}$, though could be flipped,
and thus the cases of Eqs.~\ref{eq:cd3_eq_copa_2} and \ref{eq:cd3_eq_coch_2} are consistent.

Now we focus on the first cases of Eqs.~\ref{eq:cd3_eq_copa_2} and \ref{eq:cd3_eq_coch_2}.
Since those equations are true for any variable-pairs, we especially consider a pair $(a, a')$ for both of them, and divide the both-sides (after replacing $y_2$ by $y_3$ for the second equation), which is possible thanks to the uniform dependency (Assumption~\ref{Aphi}), and then obtain
\begin{align}
	\frac{\frac{\partial^2}{\partial y_1 \partial y_2} \phit \left(k_a^m(y_1), k_b^{m_\ast}(y_2) \right)}{\frac{\partial^2}{\partial y_1 \partial y_3} \phit \left(k_{b'}^{m_\dagger}(y_3), k_a^m(y_1) \right)} &= \frac{\alpha_{1aa'} \frac{\partial^2}{\partial y_1 \partial y_2} \phit \left(k_{a'}^m(y_1), k_{b}^{m_\ast}(y_2) \right)}{\alpha_{2aa'} \frac{\partial^2}{\partial y_1 \partial y_3} \phit \left(k_{b'}^{m_\dagger}(y_3), k_{a'}^m(y_1) \right)}, \nonumber \\
	\implies \frac{\phit^{12} \left(k_a^m(y_1), k_b^{m_\ast}(y_2) \right)}{\phit^{12} \left(k_{b'}^{m_\dagger}(y_3), k_a^m(y_1) \right)} &= \frac{\alpha_{1aa'}}{\alpha_{2aa'}} \cdot \frac{\phit^{12} \left(k_{a'}^m(y_1), k_{b}^{m_\ast}(y_2) \right)}{\phit^{12} \left(k_{b'}^{m_\dagger}(y_3), k_{a'}^m(y_1) \right)}, \label{eq:cd3_eq_k}
\end{align}
where $\phi^{12}(x, y) = \frac{\partial^2}{\partial x \partial y} \phit(x, y)$, and the derivatives of the scalar functions $k_a^m$, $k_{a'}^m$, $k_b^{m_\ast}$, and $k_{b'}^{m_\dagger}$ canceled out between the left- and the right-hand sides or between the numerator and the denominator of the equation
(note that they are non-zeros almost everywhere due to the invertibility).
Since this equation is true for any choices of $y_2$ and $y_3$, we set $y_2 = (k_b^{m_\ast})^{-1} \circ k_{a'}^m (y_1)$ and $y_3 = (k_{b'}^{m_\dagger})^{-1} \circ k_a^m (y_1)$, 
and we get 
\begin{align}
	\frac{\phit^{12} \left(k_a^m(y_1), k_{a'}^m(y_1) \right)}{\phit^{12} \left(k_a^m(y_1), k_a^m(y_1) \right)} &= \frac{\alpha_{1aa'}}{\alpha_{2aa'}} \cdot \frac{\phit^{12} \left(k_{a'}^m(y_1), k_{a'}^m(y_1) \right)}{\phit^{12} \left(k_a^m(y_1), k_{a'}^m(y_1) \right)}, \nonumber \\
	\implies \phit^{12} \left(k_a^m(y_1), k_{a'}^m(y_1) \right)^2 &= \frac{\alpha_{1aa'}}{\alpha_{2aa'}} \cdot \phit^{12} \left(k_a^m(y_1), k_a^m(y_1) \right) \phit^{12} \left(k_{a'}^m(y_1), k_{a'}^m(y_1) \right).\label{eq:cd3_eq_k}
\end{align}
Since this equation indicates symmetricity of $\phit^{12}$ (flipping $k_a^m(y_1)$ and $k_{a'}^m(y_1)$ on the left-hand side gives the same value on the right-hand side), 
which is prohibited by Assumption~\ref{Aasym3}, we need to have $k_a^m = k_{a'}^m$.
Therefore we conclude that $k_a^m = k_{a'}^m$.
Since this is true for each index-pair $(a, a')$ in the group $m$ considered in Assumption~\ref{An3},
$k_a^m$ can be given as a single function $k^m$ for all variable index $a \in \Vsm$.
This is also true when we focus on the latter cases Eqs.~\ref{eq:cd3_eq_copa_2} and \ref{eq:cd3_eq_coch_2}.

\paragraph{Identifiability of the Causal Graph:}
With the same discussions above, the functions $\{k_b^{m'}\}_b$ on a group $m'$ can be also simply denoted as $k^{m'}$ for all $b$, 
based on the relations of the variables $\{s_b^{m'}\}_b$ on the group $m'$ to the variables on the other groups.
Using this to Eq.~\ref{eq:cd_eq12_exp3} with a group-pair $(m, m')$,
and by gathering this equation for all variable-index-pairs $(a, b) \in \Vsm \times \Vsmp$ on the group-pair $(m, m')$ in a matrix form ($a$ giving rows, and $b$ columns), 
and also by replacing all $\{s_a^m\}_a$ by a common variable $y_1 \in \bar{\mathcal{S}}$
and similarly all $\{s_b^{m'}\}_b$ by $y_2 \in \bar{\mathcal{S}}$, we get a matrix equation of the size $\dsm \times \dsmp$,
\begin{align}
        &\m L^{mm'} \frac{\partial^2}{\partial y_1 \partial y_2} \left( \phi(y_1, y_2) \right) + (\m L^{m' m})^\T \frac{\partial^2}{\partial y_1 \partial y_2} \left( \phi(y_2, y_1) \right) \nonumber \\
        &= \tilde{\m L}^{mm'} \frac{\partial^2}{\partial y_1 \partial y_2} \left( \phit (k^m(y_1), k^{m'}(y_2)) \right) \nonumber\\
        &+ (\tilde{\m L}^{m' m})^\T \frac{\partial^2}{\partial y_1 \partial y_2} \left( \phit (k^{m'}(y_2), k^m(y_1)) \right). \label{eq:cd_eq12vec_1_vec_exp3}
\end{align}
The factors other than the adjacency matrices $\m L^{mm'}$, $\m L^{m'm}$, $\tilde{\m L}^{mm'}$, and $\tilde{\m L}^{m'm}$ 
are now scalar values and do not change across rows and columns.

Firstly, for the elements of Eq.~\ref{eq:cd_eq12vec_1_vec_exp3} where $\lambda_{ab}^{mm'} = \lambda_{ba}^{m'm} = 0$ on the left-hand side,
the corresponding coefficients $\lambdat_{ab}^{mm'}$ and $\lambdat_{ba}^{m'm}$ on the right-hand side should be also zeros, due to the directed causal graph assumption (Assumption~\ref{An3}) and uniform dependency of the cross-derivative of $\phit$ (\ref{Aphi}).

We next focus on the elements of Eq.~\ref{eq:cd_eq12vec_1_vec_exp3} where $\m L^{mm'}$ are non-zeros (the corresponding elements of $(\m L^{m'm})^\T$ are constantly zeros due to the directed causal graph assumption \ref{An3}).
In this case, we can say that only either of $\tilde{\m L}^{mm'}$ or $(\tilde{\m L}^{m'm})^\T$ are non-zeros consistently for all of the corresponding elements.
This is because, if both of $\tilde{\m L}^{mm'}$ and $(\tilde{\m L}^{m'm})^\T$ have non-zero values on some different elements,
it is easy to show that it contradicts the asymmetricity of the function $\phit$ (Assumption~\ref{Aasym3}).
This is the same for the case when we focus on the elements where $(\m L^{m'm})^\T$ are non-zeros.

Similarly, this is also true when we focus on the right-hand side; if some elements of $\tilde{\m L}^{mm'}$ are non-zeros,
only the corresponding elements of either $\m L^{mm'}$ or $(\m L^{m'm})^\T$ are consistently non-zeros, due to the asymmetricity of $\phi$ (Assumption~\ref{Aasym3}).

These indicate that we can identify the causal graph $\m L^{mm'}$ and $\m L^{m'm}$ up to scaling and matrix-transpose (flipping of $\tilde{\m L}^{mm'}$ and $(\tilde{\m L}^{m' m})^\T$).
Additionally considering the permutation of variables for each group, which are indeterminate according to Theorem~\ref{thm:id}, we eventually have:
if we have two sets of causal graphs $\m L$ and $\tilde{\m L}$ giving the same data distributions,
we have $(\tilde{\m L}^{mm'}, \tilde{\m L}^{m'm}) = c^{mm'} (\m L_{\ve\sigma^m \ve\sigma^{m'}}^{mm'}, \m L_{\ve\sigma^{m'} \ve\sigma^m}^{m' m})$ or $c^{mm'} ((\m L_{\ve\sigma^{m'} \ve\sigma^m}^{m' m})^\T, (\m L_{\ve\sigma^m \ve\sigma^{m'}}^{m m'})^\T)$ with a scalar constant $c^{mm'}$, and permutations of variables (rows and columns) represented by $\ve\sigma^m$ and $\ve\sigma^{m'}$ on groups $m$ and $m'$, respectively. Theorem is then proven.

\end{proof}

\begin{Lemma} \label{lemma:4}
Assume $k_{\{1, 2, 3\}}: \R \rightarrow \R$ are $C^1$ scalar invertible functions, 
and $\phi(\cdot, \cdot): \R^2 \rightarrow \R$ is a function whose cross-derivative satisfies the asymmetricity (Assumption~\ref{Aasym3}).
Then for any open subset $B$ of $\bar{\mathcal{S}}$, the following relation cannot hold;
\begin{align}
	\frac{\partial^2}{\partial x \partial y} \phi(k_1(x), k_2(y)) = \gamma \frac{\partial^2}{\partial x \partial y} \phi(k_2(y), k_3(x)) \label{eq:lemma4}
\end{align}
with some scalar constant $\gamma \neq 0 \in \R$, for all $x$ and $y \in B$.
\end{Lemma}
\begin{proof}
We give a proof by contradiction.
We suppose the negation; there exist an open subset $B \subset \bar{\mathcal{S}}$ such that the equation hold for all $x$ and $y \in B$.
By the chain rule of the derivatives,
\begin{align}
	\phi^{12}(k_1(x), k_2(y)) \frac{\partial}{\partial x} k_1(x) = \gamma \phi^{12}(k_2(y), k_3(x)) \frac{\partial}{\partial x} k_3(x), \label{eq:lemma4_2}
\end{align}
where $\phi^{12}(x, y) = \frac{\partial^2}{\partial x \partial y} \phi(x, y)$, and the derivatives of $k_{\{1, 2, 3\}}$ are non-zeros almost everywhere from their invertibility.
The derivatives of $k_2(y)$ on both sides canceled out.

By a simple calculation from Eq.~\ref{eq:lemma4_2} and the uniform dependency (Assumption~\ref{Aasym3} with $c=0$), 
we can say that a function $Q(x, y) = \frac{\phi^{12}(k_1(x), k_2(y))}{\phi^{12}(k_2(y), k_3(x))} $ does not depend on $y$.
Due to this, we can consider $Q(x, y)$ with two different values of $y$, especially $y = (k_2)^{-1} \circ k_1(x)$ and $(k_2)^{-1} \circ k_3(x)$, and obtain an equation
\begin{align}
	\frac{\phi^{12}(k_1(x), k_1(x))}{\phi^{12}(k_1(x), k_3(x))} &= \frac{\phi^{12}(k_1(x), k_3(x))}{\phi^{12}(k_3(x), k_3(x))}, \nonumber \\
	\implies \left( \phi^{12}(k_1(x), k_3(x)) \right)^2 &= \phi^{12}(k_1(x), k_1(x)) \phi^{12}(k_3(x), k_3(x)). \label{eq:lemma4_3}
\end{align}
Since this equation indicates symmetricity of $\phi^{12}$ (flipping $k_1(x)$ and $k_3(x)$ on the left-hand side gives the same value on the right-hand side), 
which is prohibited by Assumption~\ref{Aasym3}, we need to have $k_1 = k_3$.
However, substituting this result to Eq.~\ref{eq:lemma4_2} indicates that this is contradictory to the asymmetricity of $\phi^{12}$ (Assumption~\ref{Aasym3}).
From this contradiction, we conclude that Eq.~\ref{eq:lemma4} cannot hold with those assumptions, and thus the Lemma is proven.

\end{proof}

\section{Alternative Identifiability Condition of Theorem~\ref{thm:cd3}}
\label{sec:cd_alt}


By adding an additional constraint on the causal function $\phi$, we can weaken the assumption on the causal graph $\m L$ in Theorem~\ref{thm:cd3}.
The alternative condition of Theorem~\ref{thm:cd3} is given below:
\begin{Proposition} \label{thm:cd2}
Assume the same as those in Theorem~\ref{thm:id}, and also:
\begin{enumerate}[label=C'\arabic*]
\item (Causal graph) The inter-group causal relations of variables are all directed, and for every group-pair $(m, m')$ in the groups of interest, 
all variables in a group $m$ (and $m'$) have either co-parent or co-child in the same group.
In addition, any variables in the group $m$ (and $m'$) can be reached from any other variables in the same group 
by moving from a variable to one of its co-parents or co-children, possibly by multiple hops. \label{An2}
\item (Asymmetricity) 
There is no open subset $B$ of $\bar{\mathcal{S}}$ such that for all $x \neq y \in B$, it holds
\begin{align}
	\phi^{12}(x, y)  = c \phi^{12}(y, x)
\end{align}
with some constant $c \in \R$. \label{Aasym2}
\item (Non-factorizability) 
There is no open subset $B$ of $\bar{\mathcal{S}}$ such that for all $x \neq y \in B$, it holds \label{Aphi2}
\begin{align}
	\phi^{12}(x, y) = \alpha(x) \beta(y) \exp \left( \gamma(x, y) - \gamma(y, x) \right) \label{eq:phi_fact}
\end{align}
with some scalar functions $\alpha$, $\beta$, and $\gamma$.
%
%
%
\end{enumerate}
Then, for all group-pairs $(m, m')$ satisfying \ref{AL} and \ref{An2}, $(\m L^{mm'}, \m L^{m'm})$ are identifiable up to permutation of variables, linear scaling, and matrix transpose.
\end{Proposition}

This Proposition requires additional constraint on the function $\phi$ (\ref{Aphi2}) compared to Theorem~\ref{thm:cd3}.
It restricts some factorization form of the cross-derivative of $\phi$.
This is similar to the non-factorizability in \ref{Aphi} of Theorem~\ref{thm:id}, but a bit stronger.
Such restriction on the factorization of the (cross-derivative of) $\phi$ is reasonable because the factorization of $\phi^{12}$ into some input-variable-wise factors 
$\alpha$ and $\beta$ would not be informative enough to fully determine the causal direction,
which is also the case for the factorization into an anti-symmetric function with some factor $\gamma$.


Thanks to such stronger constraint on $\phi$, the assumption on the causal graph (\ref{An2}) is weaker than \ref{An3} (Theorem~\ref{thm:cd3});
it requires only either co-parent or co-children for each variable, rather than both of them as in \ref{An3}.
In addition, we can consider both co-parents and co-children pairs to reach from one variable to another (see Supplementary Material~\ref{sec:c1} for some illustrative examples).

\begin{proof}
Proof is basically the same as that of Theorem~\ref{thm:cd3} (Supplementary Material~\ref{sec:cd_proof}).
For showing $k_a^m = k_{a'}^m$ from Eq.~\ref{eq:cd_eq12vec_div_y_12_exp3}, we use Lemma~\ref{lemma:3} given below,
which only requires either co-parent or co-child for each variable in contrast to both of them as in Theorem~\ref{thm:cd3} (from Eq.~\ref{eq:cd3_eq_copa_2} to Eq.~\ref{eq:cd3_eq_k}), thanks to the additional assumption of the non-factorizability of the function $\phi$ (\ref{Aphi2}).
The remaining proof is the same as that of Theorem~\ref{thm:cd3} (Supplementary Material~\ref{sec:cd_proof}), and thus omitted.
\end{proof}

\begin{Lemma} \label{lemma:3}
Assume we have
\begin{align}
	\frac{\partial^2}{\partial x \partial y} \phi(k_1(x), k_2(y)) = \gamma \frac{\partial^2}{\partial x \partial y} \phi(k_3(x), k_2(y)) \label{eq:lemma3}
\end{align}
for all $x$ and $y$ in an open subset $B$ of $\bar{\mathcal{S}}$,
where $k_{\{1, 2, 3\}}: \R \rightarrow \R$ are $C^1$ scalar invertible functions, 
$\gamma \neq 0 \in \R$ is a constant scalar value,
and $\phi(\cdot, \cdot): \R^2 \rightarrow \R$ is a function whose cross-derivative satisfies the uniform-dependency (Definition~\ref{def:dependency}) and the condition~\ref{Aphi2}.
Then we have $k_1(x) = k_3(x)$, and $\gamma = 1$.
In addition, this is also the case even if the two arguments of $\phi$ are switched on the both-sides.
\end{Lemma}
\begin{proof}
By the chain rule of the derivatives,
\begin{align}
	\phi^{12}(k_1(x), k_2(y)) \frac{\partial}{\partial x} k_1(x) = \gamma \phi^{12}(k_3(x), k_2(y)) \frac{\partial}{\partial x} k_3(x), \label{eq:lemma3_2}
\end{align}
where $\phi^{12}(x, y) = \frac{\partial^2}{\partial x \partial y} \phi(x, y)$, and the derivatives of $k_{\{1, 2, 3\}}$ are non-zeros almost everywhere from their invertibility.
The derivatives of $k_2(y)$ on both sides canceled out.

By a simple calculation from Eq.~\ref{eq:lemma3_2} and the uniform dependency, 
we can see that a function $Q(x, y) = \frac{\phi^{12}(k_1(x), k_2(y))}{\phi^{12}(k_3(x), k_2(y))} $ does not depend on $y$.
Due to this, we can consider $Q(x, y)$ with two different values of $y$, especially $y = (k_2)^{-1} \circ k_1(x)$ and $(k_2)^{-1} \circ k_3(x)$, and obtain an equation
\begin{align}
	\frac{\phi^{12}(k_1(x), k_1(x))}{\phi^{12}(k_3(x), k_1(x))} &= \frac{\phi^{12}(k_1(x), k_3(x))}{\phi^{12}(k_3(x), k_3(x))}, \nonumber \\
	\implies \phi^{12}(k_1(x), k_3(x)) \phi^{12}(k_3(x), k_1(x)) &= \phi^{12}(k_1(x), k_1(x)) \phi^{12}(k_3(x), k_3(x)).
\end{align}
However, this equation indicates that the function $\phi^{12}$ can be factorized as Eq.~\ref{eq:phi_fact} from Lemma~\ref{lemma:q} given below, which is prohibited by Assumption~\ref{Aphi2} unless $k_1 = k_3$.
Therefore we have $k_1 = k_3$ by contradiction. Putting this into Eq.~\ref{eq:lemma3} also indicates that $\gamma = 1$.
The same result can be obtained in the same manner even if the two arguments of $\phi$ are switched on the both-sides of Eq.~\ref{eq:lemma3}.
Then Lemma is proven.


\end{proof}

\begin{Lemma} \label{lemma:q}
The equation of a two variable function $q$ with uniform-dependency given as
\begin{align}
	q(x, y) q(y, x) = q(x, x) q(y, y) \label{eq:q_skew}
\end{align}
holds if and only if the function $q$ can be factorized as
\begin{align}
	q(x, y) = \alpha(x) \beta(y) \exp \left( \gamma(x, y) - \gamma(y, x) \right) \label{eq:q_fact}
\end{align}
for some scalar functions $\alpha$, $\beta$, and $\gamma$.
\end{Lemma}

\begin{proof}
We take absolute values and logarithms on both sides of Eq.~\ref{eq:q_skew}, and then take derivatives with respect to $x$ and $y$, and obtain
\begin{align}
	\frac{\partial^2}{\partial x \partial y} \left( \log \lvert q(x, y) \rvert \right) + \frac{\partial^2}{\partial x \partial y} \left( \log \lvert q(y, x) \rvert \right) = 0.
\end{align}
This skew-symmetric equation holds if and only if the function can be represented by
\begin{align}
	\frac{\partial^2}{\partial x \partial y} \left( \log \lvert q(x, y) \rvert \right) = \bar{\gamma}(x, y) - \bar{\gamma}(y, x),
\end{align}
with some scalar function $\bar{\gamma}$.
Taking integrals and exponential both sides (also considering the possible flipping of signs), we obtain Eq.~\ref{eq:q_fact}, where $\gamma$ corresponds to the integral function of $\bar{\gamma}$.
Note that this factorization form of $q$ indeed gives the Eq.~\ref{eq:q_skew}.
Then the Lemma is proven.
\end{proof}

\section{Illustrative Discussion about Assumptions~\ref{An3} and \ref{An2}}
\label{sec:c1}

\begin{figure*}[t]
 \centering
 \includegraphics[width=0.7\columnwidth]{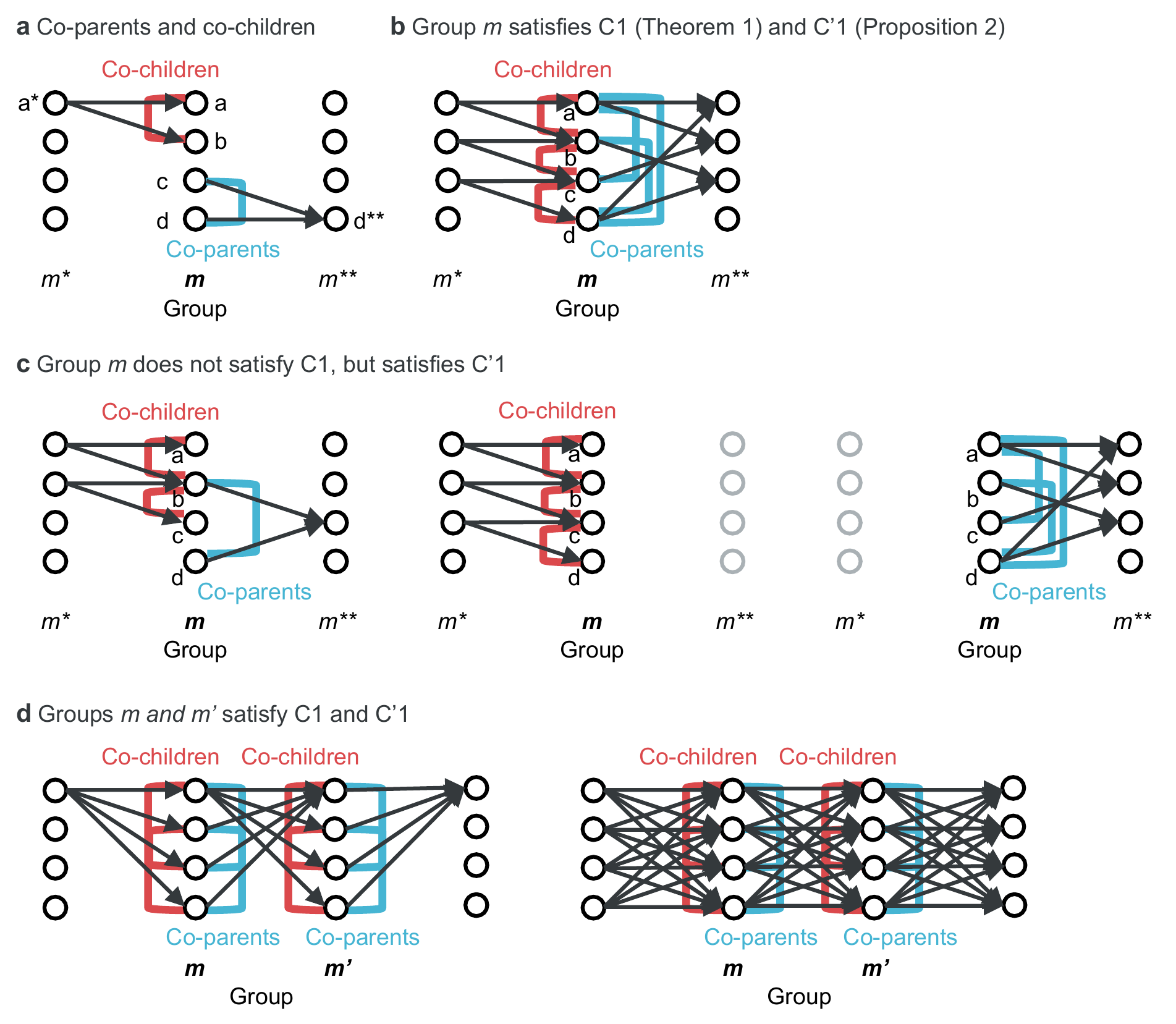}
 \caption{(\textsf{\textbf{a}}) Illustrative description of the {\it co-parents} and {\it co-children}. (\textsf{\textbf{b}}--\textsf{\textbf{d}}) Illustrative examples of some causal graphs, which (do not) satisfy Assumption~\ref{An3} (Theorem~\ref{thm:cd3}) and/or \ref{An2} (Proposition~\ref{thm:cd2}). 
 }
 \label{fig:c1}
\end{figure*}

We give here additional discussion about Assumption~\ref{An3} (Theorem~\ref{thm:cd3}) and \ref{An2} (Proposition~\ref{thm:cd2}).
Firstly, Fig.~\ref{fig:c1}a illustrates the definitions of the {\it co-parent} and {\it co-child}.
From the perspective of variable~$a$ in group $m$ ($s_a^m$), variable~$b$ ($s_b^m$) is a {\it co-child} since it has the same parent in some other group $m^\ast$ ($s_{a^\ast}^{m^\ast}$).
Similarly, from the perspective of variable~$c$ in group $m$ ($s_c^m$), variable~$d$ ($s_d^m$) is a {\it co-parent} since it has the same child in some other group $m^{\ast\ast}$ ($s_{d^{\ast\ast}}^{m^{\ast\ast}}$).
The parent $s_{a^\ast}^{m^\ast}$ and the child $s_{d^{\ast\ast}}^{m^{\ast\ast}}$ to be considered can be arbitrary selected from any other group and variable-index, and $m^\ast$ can be even the same to $m^{\ast\ast}$ (but not $m$).
In this example, group~$m$ does not satisfy both \ref{An3} and \ref{An2} since variables do not have either co-parent or co-child (\ref{An3}),
and thus some variables cannot be reached from some other variables through co-parent-and-child-paths (\ref{An3} and \ref{An2}).

For illustrative purpose, we start from a rather sparse causal graph as shown in Fig.~\ref{fig:c1}b, where group~$m$ can satisfy Assumption~\ref{An3}
(and also \ref{An2}; note that we here focus only on a single group $m$ for simplicity, while \ref{An3} and \ref{An2} actually require both of the two target groups $(m, m')$ to satisfy the same condition to identify the causal graph between them).
In this example, every variable in the group $m$ has both (at least one) co-parent and co-child in the same group,
and can reach any other variables in the same group by moving to one of its co-children (red paths; similarly for co-parents, blue paths), possibly by multiple hops.
For example, we can reach from $s_a^m$ to $s_d^m$ on the co-children side via the path $s_a^m \rightarrow s_b^m \rightarrow s_c^m \rightarrow s_d^m$ in three hops (red paths),
and via the path $s_a^m \rightarrow s_d^m$ in a single hop on the co-parents side (blue paths).

On the other hand, if we remove some edges from Fig.~\ref{fig:c1}b, \ref{An3} cannot be satisfied anymore;
in the all three examples shown in Fig.~\ref{fig:c1}c, some variables do not have either co-parent or co-child,
and thus cannot reach some other variables on either or both of the co-parent-side and co-child-side.
On the other hand, all of them can satisfy \ref{An2} (Proposition~\ref{thm:cd2}) since every variable has {\it either} co-parent or co-child in the same group,
and can reach any other variables in the same group by moving through {\it either} co-parent-path or co-child-path for each hop.
Those examples indicate that \ref{An2} allows much sparser graphs than \ref{An3}, though it requires an additional assumption on the function $\phi$ (\ref{Aphi2}).

Assumption~\ref{An3} would be fulfilled as long as the connections between groups are not too sparse.
Fig.~\ref{fig:c1}d show some such examples, where both of the two target groups $(m, m')$ satisfy \ref{An3},
and thus the causal graph between them can be identified.
Especially the right side of Fig.~\ref{fig:c1}d corresponds to the example given in Section~\ref{sec:cd};
fully-connected autoregressive temporal causality to the subsequent time-point (group), where the groups $m$ and $m'$ correspond to time-points $t-1$ and $t$, respectively.
As we can see, for every variable in a group (time-point in time-series cases), all other variables in the same group are both co-parents and co-children in this case.
Note that due to Assumption~\ref{AL} of Theorem~\ref{thm:id}, causal strengths need to be sufficiently different across edges in this case.

\section{Related Works}
\label{sec:related}

%

\paragraph{Disentangled Representation Learning and Nonlinear ICA}

Revealing fundamental representation (latent variables $\s$) generating the observational data $\x$ in a data-driven manner is called representation learning \citep{Bengio2013}.
This is supposed to be achieved by assuming some underlying observational process (e.g., mixing function $\f$ in Eq.~\ref{eq:fg}; $\x = \f(\s)$), and then by disentangling it (estimating the inverse model $\g = \f^{-1}$) in a data-driven manner from the observations.
However, such inverse problem is in general ill-posed, and there could exist many different combinations of components $\s$ and mixing function $\f$ which can explain the same observational data.
The main focus of representation learning is thus to find the conditions where the components can be determined uniquely, for its interpretability, applicability, and reproducibility.
Such model is called {\it identifiable}, and satisfies the condition $\forall (\theta, \theta')$, $p(\x; \theta) = p(\x; \theta') \Rightarrow \theta = \theta'$,
which indicates that we can uniquely identify the parameters $\theta$ of the model (the observation model $\f$ or equivalently the demixing model $\g = \f^{-1}$) from the data distribution alone.

Independent component analysis (ICA) is one of such representation learning frameworks made to identifiable based on some assumptions on the latent variables.
As the name suggests, ICA assumes that the latent components are  {\it mutually independent}, $p(\s) = \prod_i p_i(s_i)$,
as in many other representation learning frameworks including variants of variational autoencoders (VAEs) \citep{Chen2018, Higgins2017, Kim2018, Kingma2014}, and so on.
However, it is well-known that such independence assumption alone is not sufficient for the identifiability \citep{Comon1994}.
The key idea of ICA is thus to give some additional assumptions on the components so as to give the identifiability.
For example, it is well-known that when the mixing $\f$ is linear and samples are independent and identically distributed (i.i.d.), an additional assumption of non-Gaussianity (up to one Gaussian component) can make the model identifiable \citep{Comon1994}.

When the observational mixing $\f$ is nonlinear, the identifiability condition becomes much severer;
it is known that the i.i.d.~with non-Gaussianity assumption successfully used in the linear case above is not acceptable anymore in the nonlinear models \citep{Hyvarinen1999, Locatello2019}.
Nonlinear ICA (NICA) was firstly shown to be identifiable by assuming that the components are not i.i.d., but rather modulated depending on some additional (possibly latent) information associated with each sample \citep{Khemakhem2020, Hyvarinen2016, Hyvarinen2017, Hyvarinen2019, Klindt2021, Sorrenson2020, Sprekeler2014}.
More specifically, by assuming that the components are {\it conditionally} mutually independent given an (possibly unobservable) additional auxiliary variable $\ve u$, i.e., $p(\s | \ve u) = \prod_i p_i(s_i | \ve u)$,
and the $p_i$ is sufficiently modulated by $\ve u$, the model was shown to be identifiable \citep{Hyvarinen2019}.
There exist many ways to consider $\ve u$; temporal- or spatial-segment-index \citep{Hyvarinen2016, Morioka2020}, components on the other samples \citep{Hyvarinen2017, Hyvarinen2019, Halva2021, Klindt2021},  state-index of hidden Markov process \citep{Halva2020}, positive-negative-samples \citep{Zimmermann2021}, 
observation-target-samples \citep{Roeder2021}, and so on.
Efficient estimation framework is also a crucial problem of NICA.
Maximum-likelihood estimation (MLE) \citep{Halva2020} or variational estimation \citep{Halva2021, Khemakhem2020} were shown to give reasonable inference.
On the other hand, if $\ve u$ is directly observable, a self-supervised (or weakly-supervised) learning framework might be also proposed by using it as a target label or weak supervision \citep{Hyvarinen2016, Hyvarinen2017, Hyvarinen2019, Morioka2020, Morioka2021, Zimmermann2021}, which are empirically known to achieve better performances \citep{Morioka2021}.


The other direction for the identifiable NICA is to give some constraints on the mixing models $\f$, with weaker assumptions on the latent components;
such as local isometry \citep{Horan2021}, volume preservation \citep{Yang2022}, independent mechanism \citep{Gresele2021}, sparseness \citep{Moran2021}, Brenier map \citep{Wang2021b}, and a piecewise affine function \citep{Kivva2022}.
It was found recently that local isometry gives identifiability up to linear transformation \citep{Horan2021}.
\citet{Yang2022} showed identifiability by assuming that $\f$ is volume-preserving, the latent components are factorial multivariate Gaussian, and there exist two distinctive observations of auxiliary variable.
Structural sparsity of the observational model is shown to give the identifiability \citep{Zheng2022, Zheng2023}.
Assuming multiple views from latent variables was also studied \citep{Gresele2020, Locatello2020}.
Independent mechanism analysis (IMA) was shown to solve some well-known type of indeterminacy of NICA \citep{Gresele2021}, though its full-identifiability has not been resolved.
\citet{Moran2021} proposed sparse VAE and showed the identifiability with additional anchor feature assumption.
\citet{Wang2021b} showed VAE with Brenier map enables identifiability without auxiliary information.
\citet{Willetts2022} empirically showed the possibility of identifiable NICA with clustering structure of the latent variables based on a Gaussian mixture model (VaDE; \citet{Jiang2017}),
and \citet{Kivva2022} actually showed the identifiability of such model by additionally assuming that $\f$ is a piecewise affine function.
They also showed that those assumptions make the latent components identifiable up to an affine transform even with (conditional) dependency between them,
though (conditional) mutual independence is required if we need stronger identifiability.


\paragraph{Causal Discovery}

Causal discovery aims to estimate causal relations among variables from their observations in a data-driven manner.
One of the major approaches is called a model-based approach, which models causal relations of variables based on some parametric models, such as Bayesian networks (BNs) or state-equation models (SEMs), and then estimate the causal graph (or adjacency matrix) from the observations in a data-driven manner.
One of the most general models is BNs \citep{Pearl2000}, which represent a causal graph among variables by a factorization of their joint distribution into some conditional distributions representing the conditional independence of the variables; i.e., $p(\x) = \prod_{a \in \V} p_a(x_a | \text{ pa}(x_a) )$.
Although BNs are flexible, recovering the graph from the joint distribution alone is not generally possible because many different graphs can have exactly the same joint distribution  \citep{Andersson1997, Spirtes2001}. 
Some studies showed that suitable assumptions on the type of the conditional distributions enable identifiability of the causal structure, such as Poisson distribution \citep{Park2015, Park2019}, generalized hypergeometric distribution \citep{Park2019aistats}, and zero-inflated Poisson model \citep{Choi2020}.
A very closely related framework is given by SEMs \citep{Bollen1989}.
Since SEMs are not generally identifiable \citep{Bollen1989, Geiger1994, Pearl2000}, similarly to BNs, 
some further assumptions were proposed to guarantee the identifiability:
linear acyclic models with non-Gaussian noise \citep{Shimizu2006, Shimizu2011},
additive noise models excluding linear functions \citep{Hoyer2008, Hyvarinen2013, Peters2014}, post-nonlinear models \citep{Zhang2009}, and so on.
The SEMs can be also extended to time series \citep{Gong2015, Hyvarinen2010},
and models with latent confounding factors \citep{Hoyer2008IJAR, Maeda2020, Shimizu2014}, and so on. More recently, general nonlinear SEMs with non-additive noise have been proven to be identifiable by assuming nonstationarity of the noise \citep{Monti2020, Wu2020}, though limited to bivariate settings. 

\paragraph{Causal Representation Learning}

Causal representation learning (CRL) \citep{Schlkopf2021} assumes the same observational models as NICA (Eq.~\ref{eq:fg}), while the latent variables are not mutually independent but causally dependent each other;
the causal mechanism $p(\s)$ is given, for example, by a BN or SEM as in causal discovery studies (see above).
The focus of CRL is to estimate both of the (high-level) latent causal variables and the causal graph from the (low-level) observations simultaneously,
which would be achievable by jointly performing representation learning and causal discovery.
However, CRL is supposed to be highly ill-posed without any assumptions, as a combination of two notoriously ill-posed problems of NICA and causal discovery (see above), 
and the degree of the indeterminacy should be even worse compared to those two individual problems;
 causal discovery can be seen as a special case of CRL, where the latent variables are directly observable ($\f$ is the identity mapping),
and NICA can be seen as a special case of CRL as well, where all variables are mutually independent and thus $p(\s)$ is simply given by a product of variable-wise distributions. 
%
Unfortunately, simply assuming causal structure on the latent space \citep{Leeb2022} would not be enough for giving identifiability.
Some studies recently succeeded to guarantee the identifiability by assuming some constraints on the latent variables, 
such as linear SEM with supervision or intervention \citep{Buchholz2023, Liu2023, Shen2022, Squires2023, Varici2023, Yang2021}, 
or more general causal relations but with do-interventions \citep{Ahuja2022b}, perfect intervention with unknown targets \citep{Kugelgen2023}, or linear mixing \citep{Varici2023},
discrete latent variables with {\it mixture oracles} \citep{Kivva2021},
purity of the children of (subsets of) variables \citep{Cai2019, Xie2020, Xie2022},
or access to paired counterfactual data \citep{Ahuja2022a, Brehmer2022}.
%
Independently Modulated Component Analysis (IMCA) \citep{Khemakhem2020beem} was proposed as an extension of NICA to allow some dependency across variables, 
though it requires weak-supervision (observable auxiliary variables), and only considers one-to-one relations between the latent variables and the auxiliary variables.
Crucially, all of those frameworks require some level of supervision or intervention for each sample for the identifiability.
\citet{Kivva2021} requires information of the number of components $k(S)$ of the mixture model (so-called {\it mixture oracle}) for every subsets $S$ of observational variables.
Estimating the number of componetns $k$ itself is a long-standing challenge of finite mixture models, and infeasible or unstable in high-dimensional cases in practice, though they proposed a heuristic method for low dimensional case.
Other studies \citep{Lachapelle2022, Lippe2022, Yao2022b, Yao2022a} have used temporal causal relations and are thus only applicable to time-series data (many of them also require weak-supervisions or interventions). 
Although \citet{Lippe2023} extended the temporal causality to also include instantaneous one, the identifiability condition is still highly dependent on the temporal structure.


Recently the concept of grouping of variables for CRL, similarly to our work, was proposed by \citet{Daunhawer2023, Lyu2022, Morioka2023, Sturma2023, Yao2023}.
Most of them, except for \citet{Morioka2023}, especially focused on the intersections between groups.
\citet{Sturma2023} showed that by considering multiple domains (corresponding to {\it groups} in this study) sharing some latent variables, 
the latent variables shared across all domains can be identified.
Their causal and observational models are limited to linear models since their identifiability theorem and estimation algorithm are in principle based on linear ICA.
%
\citet{Daunhawer2023, Lyu2022, Yao2023} considered more general causal and observational models;
a nonlinear observational mixing for each view/modality (corresponding to {\it group} in this study) without assuming latent causal models explicitly.
Those studies showed the identifiability of the latent variables corresponding to the intersection of the (subset of) groups.
However, due to their less-restrictive models compared to \citet{Sturma2023}, the identifiability is limited to only up to block(intersection)-wise transformations.
\citet{Lyu2022, Yao2023} also showed the identifiability of the group-specific (private) variables by assuming that they are independent on the intersections.
\citet{Daunhawer2023, Lyu2022} are limited to two group settings, while \citet{Yao2023} extended them to more than two groups.
%
%
\citet{Morioka2023} proposed a CRL framework called connectivity-contrastive learning (CCL) designed for (homogeneous-)sensor-network-type architectures.
In contrast to the group-based CRL frameworks mentioned above, the goal of this framework is to estimate the (causally-related) group-specific (private) variables.
Their model can be seen as a very special case of ours;
1)~CCL assumes the mixing functions $\f^m$ are the same for all groups $m$ rather than group-specific as in ours (Eq.~\ref{eq:f}; the dimensions $\dsm$ can be also different across groups $m$ in ours), 2)~only considers component-wise relations between groups similarly to NICA (in other words, the adjacency coefficients $\lambda_{ab}^{mm'}$ in Eq.~\ref{eq:ps} can have non-zero values only when $a = b$, rather than all pairs of $(a, b)$ as in ours)\footnote{Note the difference of the notation of $(a, b)$ in our model Eq.~\ref{eq:ps} and that in Eq.~2 of \citet{Morioka2023}. The indices $(a, b)$ in our model indicate the {\it variables}, while those in CCL indicate the {\it groups} (called nodes in CCL). The causal graphs are considered separately for each component in CCL ($j$ in Eq.~2 of \citet{Morioka2023}), while our model can even consider causal relations between every components (variables) without assuming independence.},
and 3)~the graph is a forest, while ours can be more general and even cyclic.
These indicate much higher generality and applicability of our model (see Illustrative Examples) compared to CCL (homogeneous-sensor-network-type architectures).
We also emphasize that the estimation frameworks are very different too, though both of them can be categorized as self-supervised (contrastive) learning;
CCL uses group(node)-paired data for taking contrast, while our G-CaRL uses group-wise-shuffled data (Eq.~\ref{eq:xtilde}).


\section{Implementation Detail for Experiments}
\label{sec:app_exp}

We give here more detail on the data generation, training, and evaluation in Experiments (Section~\ref{sec:exp}).
The codes are available at \url{https://github.com/hmorioka/GCaRL}.

\subsection{Simulation~1: DAG}
\label{sec:app_sim1}

\paragraph{Data Generation}

We generated artificial data based on the generative model described in Section~\ref{sec:model}.
Basically, the latent variables $\s^{(n)}$ were generated probabilistically for each sample $n$ based on the pairwise BN causal model parameterized by an adjacency matrix $\m L$ (Eq.~\ref{eq:ps}), and then observed through nonlinear observational mixings $\f^m$ for each group $m$, after being divided into $M$ groups (Eq.~\ref{eq:f}).

The whole causal graph $\m L$ was designed to be a DAG (see Supplementary Fig.~\ref{fig:w_sample}a for some examples).
More specifically; the variables are causally ordered from $1$ to $D_\mathcal{S}$, such that no later variable $b$ causes any earlier variable $a < b$,
and divided into non-overlapping $M$ groups in order (i.e., the first $d_{\mathcal{S}}^1$ variables are group-1, and the next $d_{\mathcal{S}}^2$ variables are group-2, and so on).
The whole causal graph $\m L$ was generated separately for each sub-graphs; intra-group sub-graphs denoted as $\{\m L^{mm}\}_{m \in \M}$, and inter-group sub-graphs $\{\m L^{mm'}\}_{(m, m')}$.
The intra-group sub-graph $\m L^{mm} \in \R^{\dsm \times \dsm}$ was generated as a DAG for each group $m$, 
where each variable $s_a^m$ was give one other randomly selected variable $s_{a'}^m$, $a' < a$, in the same group $m$ as a parent 
(except for the first variable in the group).
We used a special structure for the first group $\m L^{11}$, where the number of parents (if they have) were fixed to 2 to avoid strong correlations between variables within the group, which can happen when the variables have only one parent.
The inter-group sub-graphs $\m L^{mm'} \in \R^{\dsm \times \dsmp}$ were generated randomly for each group-pair $(m, m')$, $m <  m'$, so that each variable $s_a^m$ on a group $m$ has (almost) two children on other groups $m'$, and conversely, each variable $s_b^{m'}$ on a group $m'$ has (almost) two parents on the group $m$.
The inter-group sub-graph on the opposite direction $\m L^{m'm}$ is empty (zero-matrix) due to the causal ordering.
In total, each variable on the $m$-th group ($m \geq 2$) has almost $2m - 1$ causal parents on average.
The non-zero values of $\m L$ were randomly drawn from $[0.9, 1]$,
and then divided by the number of parents for each variable (column) so that the standard deviations of the variables were approximately the same
regardless of the number of parents (also see below).

The latent variables were then sampled based on the following conditional distribution for each variable $s_a^{m}$ at $n$-th sample:
\begin{align}
	s_{a}^{m(n)} \sim \exp \sum_{m' \neq m} \sum_{b \in \Vsmp} - \lambda_{ba}^{m'm} \left| s_a^{m(n)} - \alpha \tanh (\beta s_b^{m'(n)}) \right|,\label{eq:sim_ps}
\end{align}
where 
$\alpha = 3$ and $\beta=0.8$ are scalar coefficients.
This indicates that the sample $s_a^{m(n)}$ is randomly generated through a (piecewise) Laplace distribution
with a standard deviation modulated by the inverse of the (summation of) $\lambda_{ba}^{m'm}$, and its average is biased by the activities of its parents, after nonlinearly transformed by $\tanh(\cdot)$.
The non-parental variables do not directly influence $s_a^{m}$ because the corresponding coefficients $\lambda_{ba}^{m'm}$ are set to zeros as mentioned above.
This sampling distribution indicates that the function $\phi$ in Eq.~\ref{eq:ps} is given by, 
\begin{align}
        \phi(x, y) = \left| y - \alpha \tanh(\beta x) \right|. \label{eq:sim_phi}
\end{align}
This function $\phi$ slightly violates the Assumption~\ref{Aphi} of Theorem~\ref{thm:id}, so we can investigate the robustness of our theory at the same time.
A typical smooth approximation of the Laplace density, such as $\phi(x, y) = - \sqrt{(y - \alpha \tanh(\beta x))^2 + \epsilon}$ with some small $\epsilon$, would satisfy the assumption.

For the observation model $\f^m: \R^{\dsm} \rightarrow \R^{\dxm}$, we used a multilayer perceptron (MLP) with $L$~layers (excluding the input layer) with random parameters,
which takes a $\dsm$-dimensional latent variable $\s^{m(n)}$ and then outputs a $\dxm$-dimensional observation $\x^{m(n)}$ for each group $m \in \M$ and sample $n$.
To guarantee the invertibility, we fixed $\dsm = \dxm = d^m$ and the number of units of each layer to $d^m$,
and used leaky ReLU units for the nonlinearity except for the last layer which has no-nonlinearity.

The number of groups ($M$) was 3, the number of variables on each group ($\dsm = \dxm = d^m$) was 10 for all groups (i.e., the number of variables $D_\mathcal{S}$ was $30$ in total),
the number of data points $n$ was $2^{16} = 65,536$, and the complexity (the number of layers) of the observational mixing model $L$ was 3.
We also evaluated the performances with changing those parameters to see how they affect the estimation performances (Supplementary Fig.~\ref{fig:sim1_LMD}a).

\paragraph{Training (G-CaRL)}

We trained the nonlinear regression function in Eq.~\ref{eq:crl_r} with the observed data by G-CaRL.
We adopted MLP for each $\h^m: \R^{\dxm} \rightarrow \R^{\dsm}$ ($h_\text{MLP}$), whose outputs are supposed to represent the latent variables after the training (Theorem~\ref{thm:crl}). 
The number of layers was selected to be the same as that of the observation model ($L$),
and the number of units in each layer was $2d^m$ except for the output ($d^m$),
so as to make it have enough number of parameters as the demixing model.
A {\it maxout} unit was used as the activation function in the hidden layers,
which was constructed by taking the maximum across two affine fully connected weight groups,
while no-nonlinearity was applied at the output (last layer).

The function $\psi: \R^2 \rightarrow \R$ in the regression function (Eq.~\ref{eq:crl_r}) was parameterized by
\begin{align}
        w_{ab}^{mm'}  \psi(x, y) = w_{1ab}^{mm'} \left| w_{2ab}^{mm'} y + |w_{2ab}^{mm'}| \text{MLP}(x) \right|, \label{eq:psi_sim1}
\end{align}
where 
$w_{1ab}^{mm'}$ and $w_{2ab}^{mm'}$ are weight parameters, which are supposed to give the estimation of the causal structure $(\lambdat_{ab}^{mm'})$ by $(w_{1ab}^{mm'} |w_{2ab}^{mm'}|)$ after training (see Supplementary Fig.~\ref{fig:w_sample}a for some examples).
The nonlinear function $\text{MLP}(\cdot): \R \rightarrow \R$ was parameterized by a learnable two-layer MLP with hyperbolic tangent (tanh) activation units.
This function has enough degree of freedom to represent the true $\phi$ (Eq.~\ref{eq:sim_phi}).
The intra-group functions $\bar{\psi}^m$ in Eq.~\ref{eq:crl_r} were also parameterized by MLP.

Those nonlinear functions were then trained by back-propagation with a momentum term (SGD)
so as to optimize the cross-entropy loss of LR with a regression function Eq.~\ref{eq:crl_r} (Supplementary Algorithm~\ref{alg:gcarl}).
The initial parameters were randomly drawn from a uniform distribution or some non-informative constant values. 
The number of iterations for the optimization depends on the complexity of the model;
e.g., convergence of a three-layer model by G-CaRL took about 3 hours (Intel Xeon 3.6 GHz 16 core CPUs, 384 GB Memory, NVIDIA Tesla A100 GPU),
though we continued the training longer for safety.

\paragraph{Evaluation}
We evaluated the estimation performances of the latent variables and the causal structures by comparing the estimations with the true values.
The learning was performed for 10 runs with changing the parameters of the observation model and the causal structures.

The estimated latent variables $\h^m(\cdot)$ were evaluated by their Pearson correlation to the true values $\s^m$ across samples.
Since the order of the variable index is undetermined for each group (Theorems~\ref{thm:id} and \ref{thm:crl}), we performed
an optimal assignment of the variable indices ($\ve\sigma^m(\cdot): \Vsm \rightarrow \Vsm$) between the estimations and the true ones by the Munkres assignment algorithm \citep{Munkres1957}, 
maximizing the mean absolute-correlation coefficients, for each group $m$.
The variable-wise accuracies (correlations) were then averaged over all variables.

For evaluations of the estimated causal structures $\tilde{\m L}^{mm'} = (\lambdat_{ab}^{mm'}) = (w_{1ab}^{mm'} |w_{2ab}^{mm'}|)$ (see Supplementary Fig.~\ref{fig:w_sample}a for some examples), we at first converted them into binary directed (not necessarily DAG) adjacency matrices by the following procedure:
we determined the causal direction on every pairs $(a, b) \in \Vsm \times \Vsmp$ by comparing the absolute values of $\lambdat_{ab}^{mm'}$ and $\lambdat_{ba}^{m'm}$;
direction is $s_a^m \rightarrow s_b^{m'}$ if $| \lambdat_{ab}^{mm'} | > | \lambdat_{ba}^{m'm} |$, and vice versa.
We then removed edges whose absolute weights were less than a specific ratio (35\% for Simulation~1) of the maximum absolute values of both $\tilde{\m L}^{mm'}$ and $\tilde{\m L}^{m'm}$ for each group-pair $(m, m')$.
If both $| \lambdat_{ab}^{mm'} |$ and $| \lambdat_{ba}^{m'm} |$ are under the threshold, $s_a^m$ and $s_b^{m'}$ are considered to have no direct causal relation.
The obtained adjacency matrices were then compared with the (binarized) true causal structure $(\lambda_{ab}^{mm'})$,
and evaluated by F1-score ($=$ 2 $\cdot$ precision $\cdot$ recall / (precision $+$ recall)).
This kind of hard thresholding is known to be effective to reduce the number of false discoveries \citep{Zheng2018},
and seems to be especially important for methods like G-CaRL which do not explicitly impose sparseness or DAG structure constraints for the estimation.
The threshold-ratio was determined separately for each experiment (simulation~1 and 2, and gene regulatory network recovery), 
but it was not changed across the parameter settings or runs within each experiment.
Our preliminary analyses showed that the G-CaRL framework was not so sensitive to the selection of the threshold values,
which can be seen from the ROC curves with varying threshold (Supplementary Fig.~\ref{fig:roc}).
Although G-CaRL is supposed to have indeterminacy of the causal graphs with its group-pair-wise matrix-transposes (Theorem~\ref{thm:cd3}), 
we solved that indeterminacy by giving some level of constraints to the function $\psi$ (Eq.~\ref{eq:psi_sim1}),
where the function cannot be fitted into the opposite direction by its design.
Therefore we directly used $\tilde{\m L}^{mm'}$ as the final guess of
$\m L^{mm'}$ for each group-pair $(m, m')$ without considering the possible matrix transpose.
For solving the possible permutation of variables, we used the same permutations $\ve\sigma^m$ and $\ve\sigma^{m'}$ estimated based on the latent variables above.
We only evaluated the {\it inter-group} causal connections, since only those are identifiable in our model (Theorem~\ref{thm:cd3}).

\paragraph{Baselines} 
The baselines only include {\it unsupervised} frameworks with {\it instantaneous} (causal) dependency, since our experimental setting does not include supervision, intervention, nor temporal causality.
Specifically, we showed the comparisons to three CRL frameworks MVCRL \citep{Yao2023}, CausalVAE (\citet{Yang2021} in unsupervised setting), and CCL \citep{Morioka2023}, 
and three representation learning frameworks MFCVAE (\citet{Falck2021, Kivva2022}), VaDE \citep{Jiang2017, Kivva2022, Willetts2022} and $\beta$-VAE \citep{Higgins2017} (see Supplementary Material~\ref{sec:app_baseline} for details).
We used publicly available implementations of them.
We also applied a CRL framework \citet{Kivva2021} as well, though it failed due to the difficulty of the estimation of the mixture model in our data.
Note that we cannot apply many of the existing CRL frameworks designed for instantaneous causal models (e.g., \citet{Shen2022}) 
since they usually require supervision or interventions, which are not available in this experiment.
We cannot apply the multi-domain unsupervised CRL framework proposed by \citet{Sturma2023} since the groups have no overlap in our setting.
\citet{Daunhawer2023, Lyu2022} are not applicable since they are limited to two group settings. 
We do not consider frameworks based on temporal causality, such as \citet{Lachapelle2022, Lippe2022, Yao2022b, Yao2022a},
since the samples are generated independently, with only instantaneous causality in this experiment.
For a fair comparison, we used the same architecture (group-wise disentanglement) for the encoders of those baselines as in the feature extractors $\h^m$ of G-CaRL.

For baselines which do not estimate the (whole or part of) causal graph by themselves (MVCRL, CCL, MFCVAE, VaDE, and $\beta$-VAE), we additionally applied a causal discovery framework to the estimated latent variables as a post-processing,
from a wide variety of selections so as to maximize the performances (Supplementary Fig.~\ref{fig:h_cd});
DirectLiNGAM \citep{Shimizu2011}, NOTEARS \citep{Zheng2018}, NOTEARS-MLP \citep{Zheng2020},
GOLEM \citep{Ng2020}, PC \citep{Spirtes1991}, CAM \citep{Buhlmann2014}, and CCD \citep{Lacerda2008}.
Briefly, DirectLiNGAM, NOTEARS, and GOLEM are specialized at linear DAGs,
CAM and NOTEARS-MLP are for nonlinear DAGs,
and CCD assumes existence of directed cycles (See Supplementary Material~\ref{sec:app_baseline}).

We used the same evaluation criteria for those baselines to that of G-CaRL
(causal direction determination, thresholding, and variable assignments).
The threshold was determined separately for each method, so as to maximize the F1-score
(see Supplementary Fig.~\ref{fig:roc} for the effect of the varying threshold). 
Although some causal discovery frameworks listed above have a function to adjust the threshold so as to make the estimated graphs DAG, 
we instead applied the same thresholding method to ours, without constraining the acyclicity on the final graph.
For some causal discovery frameworks which output a binarized adjacency matrix, we directly compared them with the binarized true adjacency matrices, after variable assignments.
Since some of them output graphs possibly with some bi-directional (or undetermined) edges,
we gave the {\it true} directions to them favorably.

To see the difficulty of our latent causal model for the conventional causal discovery frameworks,
we also applied the causal discovery frameworks listed above directly to the latent variables (Supplementary Fig.~\ref{fig:h_cd}). 
In this case, we applied them after standardizing the latent variables (zero-mean and unit-variance for each variable), as suggested by \citep{Reisach2021}.



\subsection{Simulation~2: Cyclic Graphs with Latent Confounders}
\label{sec:app_sim2}

We give here more detail on the data generation and training in Simulation~2 (Section~\ref{sec:sim2}).
Evaluation methods are the same to those in Simulation~1 (see Supplementary Material~\ref{sec:app_sim1}),

\paragraph{Data Generation}
We generated artificial data in a similar manner to Simulation~1,
though the causal graphs were designed to be much more complex, due to the presence of directed cycles and latent confounders.
Basically, we firstly generated latent variables with twice of the target size of variables with possible cycles, and then simply masked half of them as unobservable variables (latent confounders) alternately for each group; there are 10 observable (non-confounder) variables and 10 latent confounders for each group $m$ ($\dsm=10, D=30$ for observable variables).

The whole causal graph (including latent confounders) was generated separately for each group-pair as in Simulation~1, but here without considering the causal order.
The intra-group sub-graphs $\m L^{mm}$ was generated randomly so that
each variable $s_a^m$ has one other randomly selected variable $s_{a'}^m$, $a' \neq a$, in the same group as a parent.
The inter-group sub-graphs $\m L^{mm'}$ were generated randomly for each group-pair $(m, m')$ so that each variable $s_a^m$ on a group $m$ has two children on other groups $m'$, and conversely, each variable $s_b^{m'}$ on a group $m'$ has (almost) two parents on the group $m$.
And similarly for the opposite direction $\m L^{m'm}$.
At this point, each variable is supposed to have $2M - 1$ causal parents (including the latent confounders).
The non-zero values of $\m L$ were randomly drawn from $[0.9, 1]$.

The latent variables were then sampled based on the following conditional distribution for each variable $s_a^m$ at $n$-th sample:
\begin{align}
	s_{a}^{m(n)} \sim \exp \left( \sum_{m' \neq m} \sum_{b \in \Vsmp} - \frac{\lambda_{ba}^{m'm}}{| \text{pa}(s_a^m) |} \left(s_b^{m'(n)} + | \text{pa}(s_a^m) | \text{Relu}(s_a^{m(n)}) \right)^2 \right),\label{eq:sim_ps}
\end{align}
where $\text{pa}(s_a^m)$ is the set of parents (including latent confounders) of variable $s_a^m$, deduced from the adjacency matrix, $| \text{pa}(s_a^m) |$ is the number of parents,
and $\text{Relu}(x) = \max(0, x)$ is a rectified linear unit.
This indicates that the activity $s_a^{m(n)}$ is randomly generated through a Gaussian distribution
with a standard deviation modulated by the inverse of root of summation of $\lambda_{ba}^{m'm}$, 
and its average is negatively biased by the positive-, but not by negative-, activities of its parents (nonlinear inhibitory connection).
The non-parental variables do not directly influence $s_a^{m}$ because the corresponding coefficients $\lambda_{ba}^{m'm}$ are zeros as mentioned above.
The inverse-scaling of $\lambda_{ba}^{m'm}$ by $| \text{pa}(s_a^m) |$ was used so that the (conditional) standard deviations of variables were approximately the same
regardless of the number of parents. 
This sampling distribution indicates that the function $\phi$ in Eq.~\ref{eq:ps_exp} is given by, with a simple calculation,
\begin{align}
        \phi(x, y) = y \text{Relu}(x). \label{eq:sim2_phi}
\end{align}
Since this causal graph can have a directed cycle, we generated the data realizations based on Gibbs sampling.

After generating the latent variables, we masked half of the variables as latent confounders alternately.
Since each variable needs to be causally related to at least one of the variables on some other group (Assumption~\ref{AL}),
we generated the causal graph under a constraint that each variable $s_a^m$ has one observable child and one observable parent on all of the other groups $m' \neq m$ after the masking, in the graph generation above.

We used MLPs for the observation models $\f^m: \R^{\dsm} \rightarrow \R^{\dxm}$, as in Simulation~1.

The number of groups ($M$) was 3, the number of the observable variables was 10 for all groups (i.e., $\dsm = \dxm = d^m = 10$;  the number of variables $D_\mathcal{S}$ was $30$ in total, and the number of latent confounds was also 30),
The number of data points $n$ was $2^{20}$, and the complexity (the number of layers) of the observational mixing model $L$ was 3.
We also evaluated the performances with changing those parameters to see how they affect the estimation performances (Supplementary Fig.~\ref{fig:sim1_LMD}b).

\paragraph{Training (G-CaRL)}

We train the nonlinear regression function in Eq.~\ref{eq:crl_r} with the observed data by G-CaRL.
The model is basically the same as that used in Simulation~1, except for the regression function.
The function $\psi: \R^2 \rightarrow \R$ in the regression function (Eq.~\ref{eq:crl_r}) was parameterized by
\begin{align}
        \psi(x, y) = y \max(a_1^{mm'} (x - b_1^{mm'}), a_{2}^{mm'} (x - b_2^{mm'}))
\end{align}
with some scalar parameters $a_1^{mm'}, b_1^{mm'}, a_{2}^{mm'}$ and $b_2^{mm'}$.
This is based on the idea of maxout unit, and has enough degree of freedom to represent the causal effect function $\phi$ (Eq.~\ref{eq:sim2_phi}).
The intra-group functions $\bar{\psi}^m$ were parameterized by MLP.
The weight parameters $w_{ab}^{mm'}$ in the regression function (Eq.~\ref{eq:crl_r}) are supposed to give the estimation of the causal structure $(\lambdat_{ab}^{mm'})$ after training 
(Theorem~\ref{thm:cd3}; see Supplementary Fig.~\ref{fig:w_sample}b for some examples).


\subsection{Recovery of Gene Regulatory Network}
\label{sec:app_gene}

We used synthetic single-cell gene expression data generated by SERGIO \citep{Dibaeinia2020},
where each gene expression is governed by a stochastic differential equation (SDE) derived from a chemical Langevin equation, 
with activating or repressing causal interactions with the other genes.
The gene expression data generated by SERGIO were shown to be statistically comparable to real experimental data \citep{Dibaeinia2020}.
We used the same parameters for the differential equations as in \citep{Dibaeinia2020}, 
but changed the hill coefficient from 2 to 6 to make the causal relations more nonlinear.

The causal graph was designed to be a DAG (as required of SERGIO) similarly to Simulation~1, 
but with latent confounders similarly to Simulation~2 (Supplementary Fig.~\ref{fig:w_sample}c shows  examples). 
More specifically, the intra-group sub-graphs $\m L^{mm}$ was generated randomly in same way used in Simulation~1,
while the inter-group sub-graphs $\m L^{mm'}$ were generated randomly for each group-pair $(m, m')$ in the same way used in Simulation~2, but only for the group pairs $m < m'$.
To make almost all genes have children on some other group, we designated the last gene of each group $m < M$ as the leaves (have no children),
and connected genes on the last group (group-$M$) to those genes as parents.
The number of variables (genes) including the latent confounders was fixed to 60, we then masked half of them as the latent confounders,
and divided them into 3 groups (i.e., $M=3$, and $\dsm=10, D=30$ for non-latent-confounder variables).
The maximum contributions (weights of edges) from parental genes to target genes were set to 0.25 for all edges.
We set half of the parents as activating, and the others as repressing for each gene.
See Supplementary Fig.~\ref{fig:w_sample}c for some example.
The genes (variables) which do not have any parents were assigned as master regulators (MRs), 
and controlled by basal production rates, randomly selected from $[0.25, 0.75]$.
We fixed the number of samples to $2^{18}$.

For the observation model $\f: \R^{\dsm} \rightarrow \R^{\dxm}$, we used a multilayer perceptron (MLP) with $L$ layers (excluding the input layer) similarly to Simulations~1 and 2
because there is no known realistic settings of the observational mixings in this kind of gene expression data, to the best of our knowledge.
The complexity (the number of layers) of the observational mixing model $L$ was fixed 3, similarly to Simulations~1 and 2.

The function $\psi: \R^2 \rightarrow \R$ in the regression function (Eq.~\ref{eq:crl_r}) was parameterized as
\begin{align}
        \psi(x, y) = y  \sum_{k=1}^K a_{k} \tanh(b_{k} x + c_{k}),
\end{align}
where $K=5$ is a model order, $a_{k}$, $b_{k}$, and $c_{k}$ are trainable scalar parameters.
The weight parameters $w_{ab}^{mm'}$ in the regression function (Eq.~\ref{eq:crl_r}) are supposed to give the estimation of the causal structure $(\lambdat_{ab}^{mm'})$ after training 
(Theorem~\ref{thm:cd3}; see Supplementary Fig.~\ref{fig:w_sample}c for some examples).

\subsection{High-Dimensional Image Observation}
\label{sec:app_sim3}

\paragraph{Data Generation}

We additionally evaluated our framework by using high-dimensional image-generation process as the observational model.
We especially used 3DIdent image dataset \citep{Zimmermann2021}.
The dataset is compose of images of an object (tea pot) rendered with image-size $224 \times 224$, conditioned on ten continuous factors (called hereafter {\it image-factors}; see Supplementary Fig.~\ref{fig:sim2_image}); three dimensional positions of the object ($\text{PosX}_\text{obj}$, $\text{PosY}_\text{obj}$, and $\text{PosZ}_\text{obj}$), three dimensional rotation of the object in Euler angles ($\text{RotX}_\text{obj}$, $\text{RotY}_\text{obj}$, and $\text{RotZ}_\text{obj}$), the color of the object and the ground of the scene ($\text{Color}_\text{obj}$ and $\text{Color}_\text{bg}$), and the position and color of the spotlight ($\text{Pos}_\text{light}$ and $\text{Color}_\text{light}$).

We firstly generated {\it candidates} of the latent causal variables based on the same way used in Simulation~2 (cyclic graphs with latent confounders),
especially with fixing the number of latent variables to 10 (there also exist 10 additional latent confounders) for all group (see Supplementary Fig.~\ref{fig:sim2_image} for example). 
We then picked a set of ten image-factors closest to the generated ten dimensional variables (after scaling-normalization) from the dataset,
and then adopted it as the actual latent causal variables ($\s^m$) and the corresponding image as the observational image ($\x^m$), for each group $m$.
Note that we cannot find the image-factors which exactly matches the generated ten dimensional variables since the number of images are limited (250,000) in this dataset.
This indicates possible misspecification of the causal model, which enables the evaluation of the robustness of G-CaRL.
We fixed the number of groups $M$ to 3, and the number of data points $n$ to $2^{20}$.



\paragraph{Training (G-CaRL)}

Since the observations are high-dimensional images, we used here convolutional neural networks as the feature extractors $\{\h^m\}$.
More specifically, we used ResNet-18 \citep{He2016} additionally with hyperbolic tangent units and a fully-connected layer on top of it,
which encodes an input image into 10-dimensional features.
Those 10 features are supposed to represent the original image-factors after the training.
We trained a single feature extractor shared across all groups since the observational models are supposed to be the same across groups in this experiment.
The learning was performed for 10 runs with changing the parameters of the causal structures.

\paragraph{Evaluation}
Due to the much higher nonlinearity of the observational model compared to the former simulations, we slightly changed the evaluation criteria, which is the same for the baselines too:
1)~for evaluating the estimation of the latent variables, we used Spearman's rank correlation instead of Pearson correlation, and
2)~for evaluating the causal graph, we optimally chose either the original estimate or its matrix-transpose so as to maximize the F1-score,
since the estimated causal graph could be flipped (matrix-transposed) from the true one sometimes as suggested by Theorem~\ref{thm:cd3}, which did not happen in the previous simulations.
We found that only considering a matrix-transpose of the whole causal graph was enough, rather than group-pair-specific matrix transposes as suggested in Theorem~\ref{thm:cd3}.
The other evaluation methods are the same to those in Simulation~1 and 2.

\section{Details of Baselines}
\label{sec:app_baseline}

\paragraph{CCL}
Connectivity-contrastive learning (CCL; \citet{Morioka2023}) is a CRL framework based on self-supervised learning,
whose generative model can be seen as a special case of ours.
CCL assumes a sensor-network-type generative model with homogeneous observations, where multidimensional observations are obtained for each group (called {\it node} in \citet{Morioka2023}) from latent components (corresponding to {\it latent variables} in this study),
which are causally-related across groups while mutually-independent across components.
More specifically, the observational model is given by
\begin{align}
	\x &= \begin{bmatrix} \x^1, & \ldots, &  \x^M \end{bmatrix} \nonumber \\
	&= \begin{bmatrix} \f(\s^1), & \ldots, & \f(\s^M) \end{bmatrix}, \label{eq:f_ccl}
\end{align}
with a group-common (node-homogeneous) observational mixing $\f: \R^d \rightarrow \R^d$, and the joint distribution of the latent causal variables $\s$ are assumed to be factorized as
\begin{align}
	p(\s) \propto & \prod_{a \in \Vs} \left[ \prod_{m \in \M} \exp \left(\bar{\phi}_a^m(s_a^m) \right) \prod_{m \neq m'} \exp \left( \lambda_{a}^{mm'} \phi_a(s_a^m, s_a^{m'}) \right)  \right], \label{eq:ps_ccl} 
\end{align}
with some (component-specific) potential functions $\bar{\phi}_a^m: \R \rightarrow \R$ and $\phi_a(s_a^m, s_a^{m'}): \R \times \R \rightarrow \R$, and (component-wise) adjacency coefficients between groups $\{\lambda_{a}^{mm'}\}_{(m, m')}$.
They then showed the identifiability of this model as a causal model with some causal structural assumptions, such as asymmetricity of the causal graph $\{\lambda_{a}^{mm'}\}_{(m, m')}$ and the potential functions $\phi_a$.
Although their models somewhat resembles to ours, they can be seen as a very special case of ours (see Section~\ref{sec:discussion}).
They also proposed a self-supervised learning framework called CCL, which estimates the latent components and the causal graph simultaneously, based on a multinomial logistic regression (MLR), whose pretext task is to well predict the group-pair label
of group-paired observations, for every group-pairs and samples. 
Note that the estimation framework is much different from ours, though both of them are categorized as self-supervised learning.

In our experiment, since CCL can consider only the diagonal part of the adjacency matrix $\m L^{mm'}$ for each group-pair $(m, m')$
(compare Eq.~\ref{eq:ps_ccl} to our model Eq.~\ref{eq:ps}; also see Supplementary Material~\ref{sec:related}), 
we applied a causal discovery algorithm to the estimated latent components as a post-processing for reconstructing the whole causal graph.

\paragraph{MVCRL}
Multi-view CRL (MVCRL; \citet{Yao2023}) is a CRL framework targeting multi-view data, based on alignments of the (embeddings of) intersections of the latent variables (called {\it content}) between views
(for the sake of consistency, we hereafter call {\it views} as {\it groups}).
Some multimodal representation learning frameworks can be seen as a special case of their work \citep{Ahuja2022a, Daunhawer2023, Gresele2020, Kugelgen2021, Locatello2020}.
It assumes group-wise observational model
\begin{align}
	\x^m &= \f^m(\s^m), \label{eq:f_mvcrl}
\end{align}
for each of the groups $m \in \M$, similarly to ours (Eq.~\ref{eq:f}),
while it considers overlaps of the latent variables (and possibly those of the observations) for (some) subsets of groups $\M$.
The authors then showed that sets of latent variables, each of which is the intersection of some subset of groups (called {\it content}), can be block-identified.
By additionally assuming that the non-shared (view-specific, called {\it style}) variables are independent on the shared variables (content), the non-shared variables are also block-identifiable.
MVCRL seeks to learn the embeddings of the latent variables by optimizing the alignments of the contents shared between views.

In our experiments, although our data do not have {\it contents variables} shared across groups, we considered that all group-specific variables $\s^m$ are the {\it contents} to be aligned for all subsets of groups in the estimation by MVCRL.
This should be practically better than explicitly considering them as group-specific (independent) variables, since they should have some similarity (though not {\it equivalence} like content variables) between groups due to their inter-group causal relations in our settings.

\paragraph{CausalVAE}
CausalVAE \citep{Yang2021} is a CRL framework based on VAE with some causal structural assumptions on the latent embedding.
The authors showed that some causal assumptions, such as linear directed acyclic causal graphs, and (weak) supervision on the latent variables give the identifiability of the model up to some indeterminacy.
We especially used its unsupervised setting (CausalVAE-unsup, \citep{Yang2021}) as a baseline,
which is composed of an encoder and a decoder, as in vanilla-VAE, and also a Causal Layer to represent the causal relations of the latent variables.
More specifically, the input signals (observations $\x$) passes through an encoder to obtain independent exogenous factors $\ve\epsilon \sim \mathcal{N}(\ve\epsilon; \ve 0, \m I)$,
which are passed through a Causal Layer to generate causal variables $\s$ via a linear SEM as
\begin{align}
	\s = \m A^\T \s + \ve\epsilon = (\m I - \m A^\T)^{-1} \ve\epsilon,
\end{align}
where $\m A$ is a matrix parameter representing an adjacency matrix,
which are then taken by the decoder to reconstruct the original observation $\x$.
CausalVAE estimates the model by optimizing the evidence lower bound (ELBO), similarly to vanilla-VAE, but additionally with regularization on DAG-ness of $\m A$.

\paragraph{$\beta$-VAE}
$\beta$-VAE \citep{Higgins2017} is a representation learning framework based on the vanilla-VAE \citep{Kingma2014}, but with an adjustable regularization parameter $\beta$ on the distribution of the embeddings.
The generative model is the same as the vanilla-VAE;
the observations $\x$ are obtained from the latent variables $\s$ through an unknown observational mixing $\f$ as
\begin{align}
	\x &=  \f(\s) + \ve\epsilon \label{eq:f_vae}
\end{align}
with observational noise $\ve\epsilon$, 
and the joint distribution of the latent variables $\s$ are assumed to be the standard multivariate Gaussian distribution
\begin{align}
	p(\s) =  \mathcal{N}(\s; \ve 0, \m I), \label{eq:ps_vae}
\end{align}
where $\mathcal{N}(\cdot; \ve 0, \m I)$ is the Gaussian density with zero-mean and unit diagonal covariance,
which implies that the latent variables are assumed to be mutually orthogonal.
$\beta$-VAE additionally has a regularization parameter $\beta$ 
 for adjusting the strength of the KL-divergence regularization controlling the discrepancy between the encoded latent variable distributions and their priors. When $\beta=1$, $\beta$-VAE simply leads to the vanilla-VAE. 
 
 In our experiments, we used the setting of $\beta=1$, which actually corresponds to the original-VAE \citep{Kingma2014}, since it gives the best estimation performance of the latent variables.

\paragraph{VaDE}
VaDE \citep{Jiang2017} is a representation learning framework based on VAE with Gaussian mixture priors.
The observational model is the same as the vanilla-VAE (Eq.~\ref{eq:f_vae}), while
the joint distribution of the latent variables is assumed to be given by a Gaussian mixture model as
\begin{align}
	p(\s) = \sum_{k = 1}^K \pi_k \mathcal{N}(\s; \ve\mu_k, \ve\Sigma_k), \quad \sum_{k=1}^K \pi_k = 1, \quad \pi_k > 0,
\end{align}
where $N(\cdot; \ve\mu_k, \ve\Sigma_k)$ is the Gaussian distribution with mean vector $\ve\mu_k$ and covariance matrix $\ve\Sigma_k$ of the $k$-th mixture component.
\citet{Willetts2022} later empirically showed that the estimation of this model gives rather consistent results across estimations,
\citet{Kivva2022} then showed that the identifiability of such model can be given by additionally assuming that $\f$ is a piecewise affine function.
VaDE model can be seen as an unsupervised version of identifiable-VAE (iVAE; \citet{Khemakhem2020}), where the indices of the mixture components are unknown rather than being given as an observable auxiliary variable for all samples as in \citet{Khemakhem2020}.

In our experiments, we used the implementation by \citet{Kivva2022}, and fixed the number of mixture components to be 5.


\paragraph{MFCVAE}
Multi-facet clustering variational autoencoder (MFCVAE; \citet{Falck2021}) extends the idea of VaDE \citep{Jiang2017}
so as to jointly consider multiple aspects (facets) of clustering features hidden in the data, such as {\it color} and {\it shape} in images.
MFCVAE assumes that the observations are obtained from $J$ multidimensional latent facets $\{\z_1, \ldots, \z_J\}$ as
 \begin{align}
	\x &=  \f(\z_1, \ldots, \z_J) + \ve\epsilon \label{eq:f_mfcvae}
\end{align}
with observational noise $\ve\epsilon$, and each facet $j$ has
its own unique clustering structure represented by a multivariate Gaussian mixture model with $K_j$ mixture components
\begin{align}
	p(\z_j) = \sum_{k_j=1}^{K_j} p(k_j) \mathcal{N}(\z_j; \ve\mu_{k_j}, \ve\Sigma_{k_j}), \quad p(k_j) = \text{Cat}(\ve\pi_j),
\end{align}
where $\text{Cat}(\cdot)$ denotes a categorical distribution, $\ve\pi_j$ is the $K_j$-dimensional vector of mixing weights, 
and $\ve\mu_{k_j}$ and $\ve\Sigma_{k_j}$ are the mean vector and (diagonal or full) covariance matrix of the $k_j$-th mixture component in facet $j$.
MFCVAE then learns clustering on multiple facets jointly, from shared (or progressively-trained ladder-architectured) latent variables in a fully unsupervised and end-to-end manner.
\citet{Kivva2022} later showed that the model is identifiable by additionally assuming that $\f$ is a piecewise affine function.

In our experiments, we fixed the number of facets to 3. The dimension of the latent space was set to be 5 for each facet and number of mixture components to be 25, as recommended in \citet{Falck2021}.

\paragraph{DirectLiNGAM}
DirectLiNGAM \citep{Shimizu2011} assumes SEMs with linear DAG and non-Gaussian errors.
In the first step, DirectLiNGAM finds the causal order of variables by iteratively finding a 
root variable by performing regression and independence testing for each pair of variables,
extracting one which is exogenous to the others,
and then removing the effect of the root variable from the other ones.
DirectLiNGAM then eliminates unnecessary edges using AdaptiveLasso \citep{Zou2006},
and outputs a weighted adjacency matrix.

\paragraph{NOTEARS}
NOTEARS \citep{Zheng2018} assumes linear SEMs of DAG.
It estimates a weighted adjacency matrix by minimizing a least-squares loss in scoring DAGs
with regularization terms imposing sparseness and DAG-ness of the adjacency matrix.
Since NOTEARS formulates the structure learning problem as a continuous optimization problem over real matrices,
it can effectively avoid the traditional combinatorial optimization problem (NP-hard) of learning DAGs.
We used the default parameters.

\paragraph{NOTEARS-MLP}
NOTEARS-MLP \citep{Zheng2020} is an extension of NOTEARS \citep{Zheng2018} to general nonparametric DAG models.
NOTEARS-MLP models variable-wise nonlinear causal functions by MLPs,
which are learned based on continuous optimization problem with regularizations for the sparseness of the MLP parameters 
and for DAG-ness of the causal functions.
We used the default parameters.

\paragraph{GOLEM}
GOLEM \citep{Ng2020} is an efficient version of NOTEARS \citep{Zheng2018}, which can reduce number of optimization iterations.
GOLEM assumes linear DAGs, and performs multivariate Gaussian maximum likelihood estimation (MLE)  
with a soft version of the differentiable acyclicity constraint proposed in \citep{Zheng2018}.
There are two proposed models; equal(EV) or unequal (NV) noise variances. We used EV here since the estimation performances were better than NV.
We used the hyper-parameters used in \citep{Ng2020}.

\paragraph{PC}
PC algorithm \citep{Spirtes1991} is a constraint-based method.
PC algorithm firstly constructs an undirected graph by removing edges from a fully connected graph based on independence and conditional independence tests.
It then constructs a DAG by directing the edges based on the information of separation sets and with some additional assumptions (no new v-structures and directed cycles).

\paragraph{GES}
Greedy equivalent search (GES) \citep{Chickering2003} algorithm is a score-based method.
GES starts with an empty graph and iteratively adds directed edges such that the improvement of Bayesian score (BIC score) is maximized,
until no single edge addition increases the score (forward phase).
GES then iteratively removes edges until no more improvements in the score can be made by single-edge deletions (backward phase).

%
%

\paragraph{CAM}
Causal additive model (CAM) \citep{Buhlmann2014} assumes SEMs specified by DAG and additive Gaussian errors,
which is an extension of linear SEMs by allowing for variable-wise scalar nonlinear functions (Collorary~31 in \citep{Peters2014}).
CAM at first estimates the causal order of variables based a greedy search algorithm so as to maximize the likelihood, then 
non-relevant edges were removed (pruning) by a sparse regression technique implemented based on significance testing of covariates.


\paragraph{CCD}
Cyclic causal discovery algorithm (CCD) \citep{Lacerda2008} assumes that the data are causally sufficient (no latent variables), while possibly includes directed cycles.
CCD extracts cyclic models using conditional independence tests, as with PC \citep{Spirtes1991}.
The output of CCD algorithm is a cyclic partial ancestral graph (PAG), which is a graphical object that represents a set of causal Bayesian networks that cannot be distinguished by the algorithm.

\begin{figure}[t]
\begin{algorithm}[H]
    \renewcommand{\algorithmicrequire}{\textbf{Input:}}
    \renewcommand{\algorithmicensure}{\textbf{Output:}}
    \caption{Pseudo-code of Grouped Causal Representation Learning (G-CaRL) based on stochastic gradient descent (SGD) optimization}
    \label{alg:gcarl}
    \begin{algorithmic}[1]
    \REQUIRE A set of observational data $\{\x^{(n)}\}_n$, observational grouping indices $\{\Vxm\}_m$, and hyper-parameters for SGD optimization.
    \STATE Initialization: Initialize the parameters of the regression function (Eq.~\ref{eq:crl_r}) with random values.
    \REPEAT
    \STATE Randomly pick some samples $n$ (mini-batch) from the observations $\{\x^{(n)}\}$ (label $1$) and $\{\x^{(n_\ast)}\}$ (label $0$) (Eq.~\ref{eq:xtilde}), where $\x^{(n_\ast)}$ are generated artificially from the original $\{\x^{(n)}\}_n$ by shuffling the index $n$ over all samples separately for each group $m$, indicated by $\Vxm$.
    \STATE Update the parameters of the regression function (Eq.~\ref{eq:crl_r}) so as to minimize the objective function (cross-entropy) of the logistic regression, discriminating the labels $1$ and $0$.
    \UNTIL the objective function converges.
    \STATE \textbf{return} the trained (group-wise) nonlinear feature extractors $\{\h^m(\cdot)\}$ representing the latent causal variables (Theorem~\ref{thm:crl}), and the weight parameters $\{w_{ab}^{mm'}\}$ representing the causal structures (Theorem~\ref{thm:cd3}).
    \end{algorithmic}
\end{algorithm}
\end{figure}

\begin{figure*}[t]
 \centering
 \includegraphics[width=\linewidth]{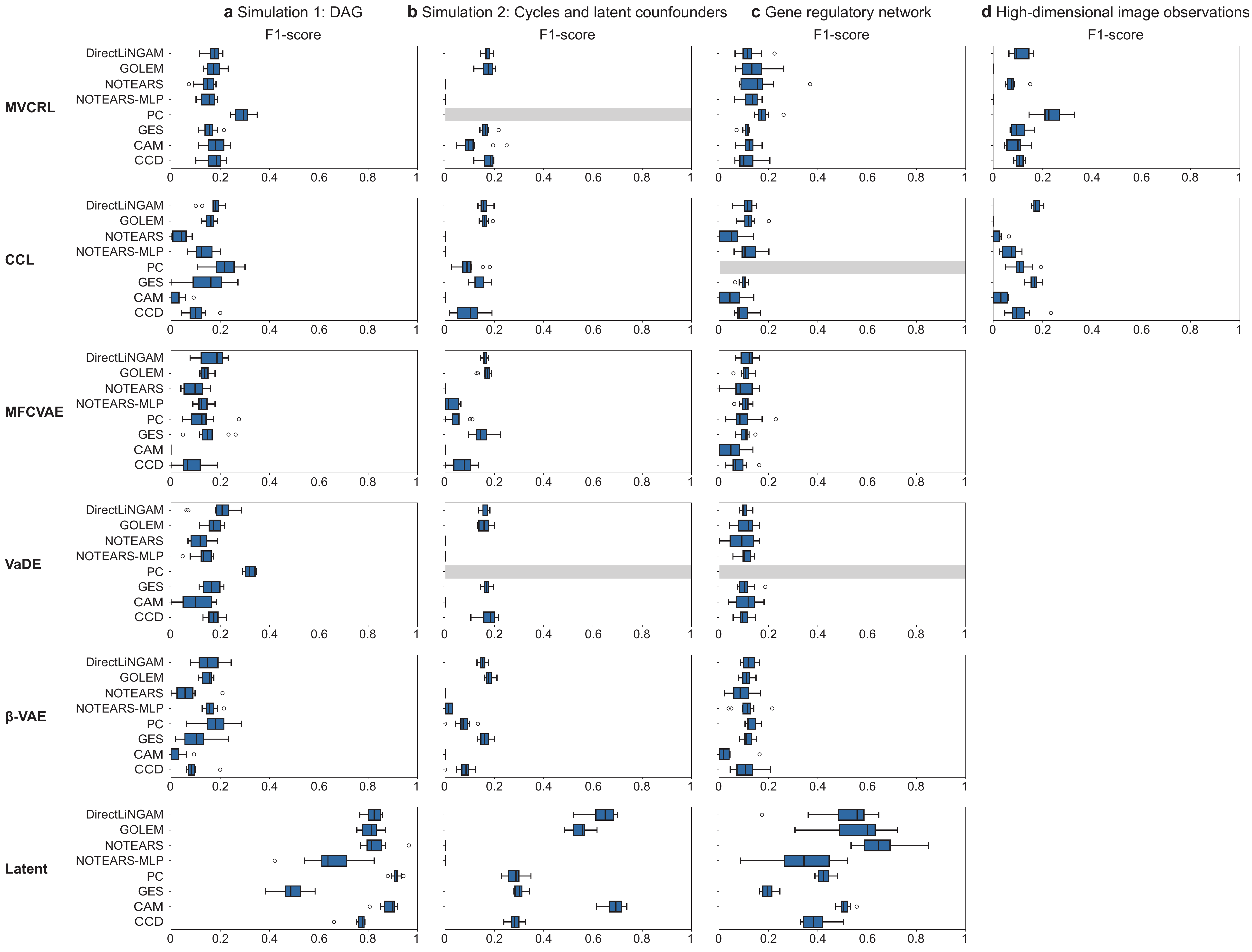}
 \caption{Comparison of a set of causal discovery frameworks (rows in each panel; measured by F1-score) applied to the baseline representation learning frameworks (rows), or directly to the latent variables (the last row: {\it Latent}; omitted in \textsf{\textbf{d}} since it is the same as \textsf{\textbf{b}}). We discarded some causal discovery frameworks (shaded by grey) on some panels since they did not converge within practical calculation time.
 (\textsf{\textbf{a}}) Simulation~1, (\textsf{\textbf{b}}) Simulation~2, (\textsf{\textbf{c}}) gene regulatory network recovery, and (\textsf{\textbf{d}}) high-dimensional image observations.
  Only the best performance for each panel was reported in Fig.~\ref{fig:sim}.
  }
 \label{fig:h_cd}
\end{figure*}

\begin{figure*}[t]
 \centering
 \includegraphics[width=\columnwidth]{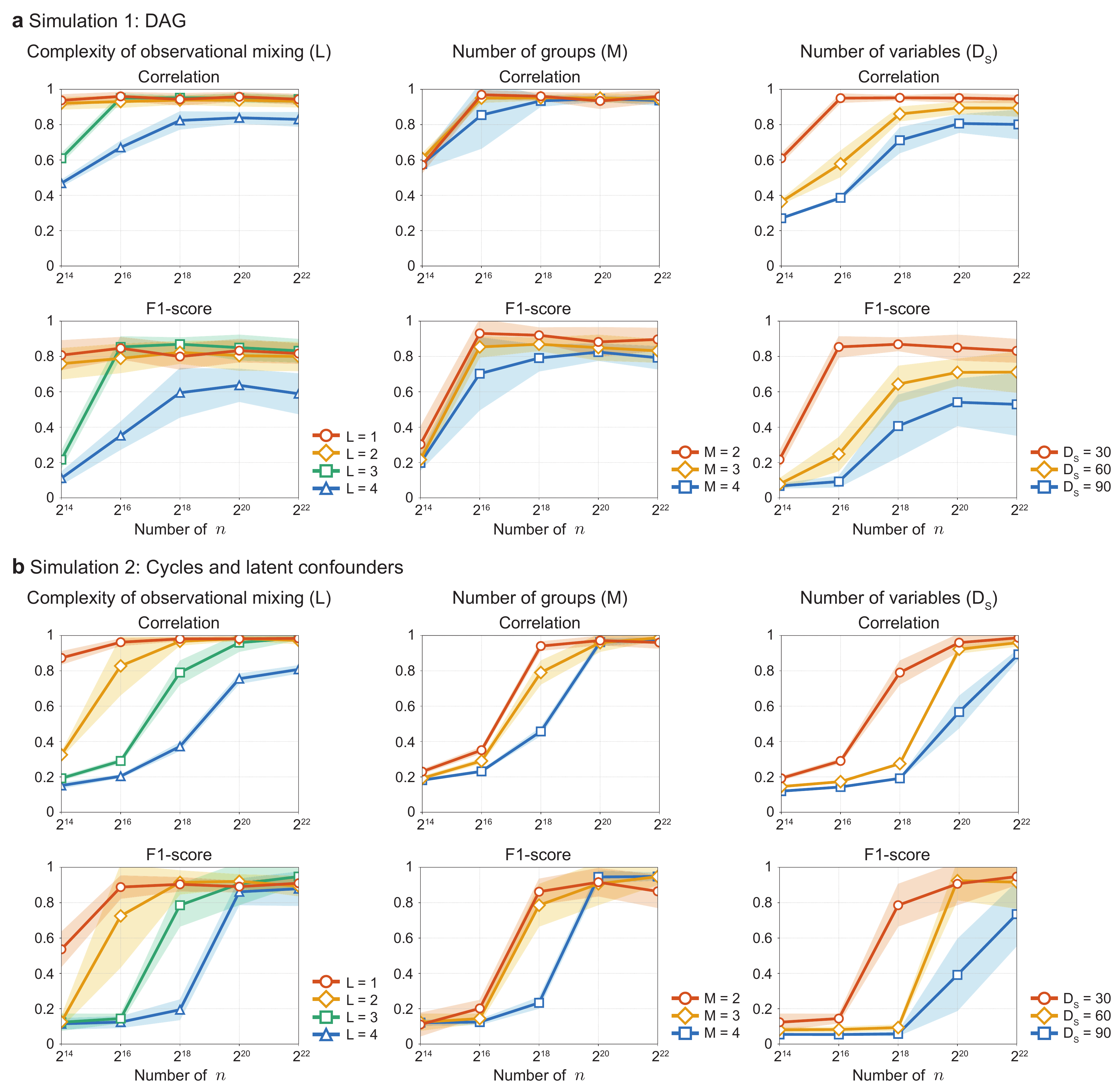}
 \caption{Estimation performances of the latent variables (Pearson correlation) 
 and the causal structures (F1-score) by the proposed framework G-CaRL, but
 different settings of (Left) the complexity of the observation models (the number of MLP-layers $L$ of the observation function $\ve f$), (Middle) the number of groups ($M$), and (Right) the number of variables ($D_\mathcal{S}$), with changing the number of samples $n$. 
  Simulation~1 (basic DAG). (\textsf{\textbf{b}}) Simulation~2 (cycles and latent confounders).
 The values are the averages of 10 runs for each setting, and the shaded regions show the standard deviations. \
 Fig.~\ref{fig:sim}a corresponds to the case $L=3$, $M=3$, $D_\mathcal{S}=30$, and $n=2^{16}$ in \textsf{\textbf{a}},
 and Fig.~\ref{fig:sim}b corresponds to the case $L=3$, $M=3$, $D_\mathcal{S}=30$, and $n=2^{20}$ in \textsf{\textbf{b}}.
 }
 \label{fig:sim1_LMD}
\end{figure*}


\begin{figure*}[t]
 \centering
 \includegraphics[width=\columnwidth]{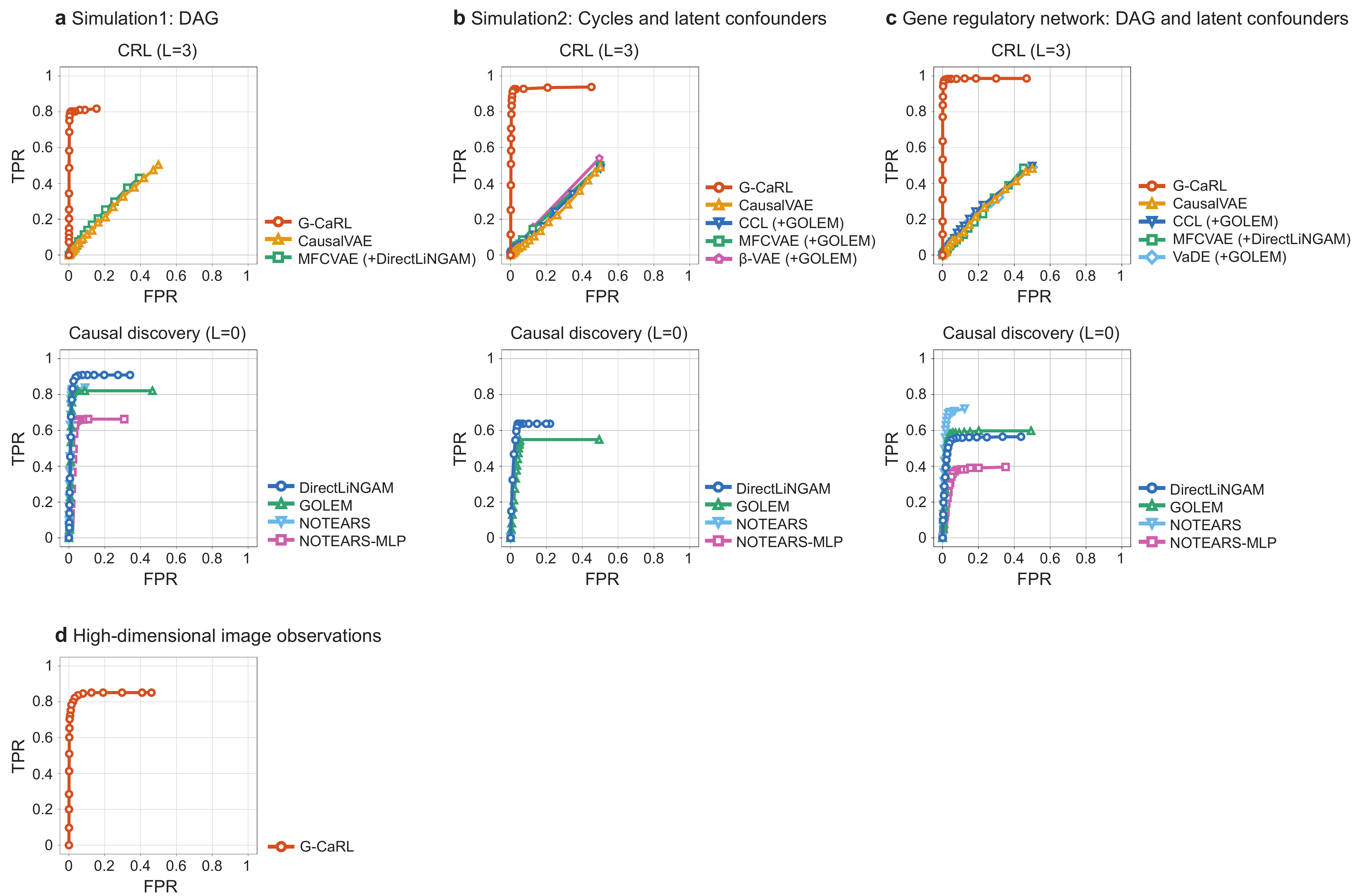}
 \caption{Illustration of the effect of the threshold for each method on (\textsf{\textbf{a}}) Simulation~1, (\textsf{\textbf{b}}) Simulation~2, (\textsf{\textbf{c}}) the gene regulatory network recovery task, and (\textsf{\textbf{d}}) the high-dimensional image observations.
  The upper panels show the results with unknown observational mixings (CRL), and lower panels show the results when we applied the causal discovery frameworks directly to the latent variables (causal discovery; omitted in \textsf{\textbf{d}} since it is the same as \textsf{\textbf{b}}).
 For each panel, ROC curve shows false positive rate (FPR) and true positive rate (TPR) with varying level of threshold, from 0\% to 100\% with interval of 5\%, for each method. The values are the averages of 10 runs for each threshold. In upper panels, we only showed the curves for the (combinations of) frameworks in Fig.~\ref{fig:sim} which give {\it wighted} adjacency matrices and thus require thresholding.
This result shows that G-CaRL was not so sensitive to the selection of the threshold values.}
 \label{fig:roc}
\end{figure*}

\begin{figure*}[t]
 \centering
 \includegraphics[width=\columnwidth]{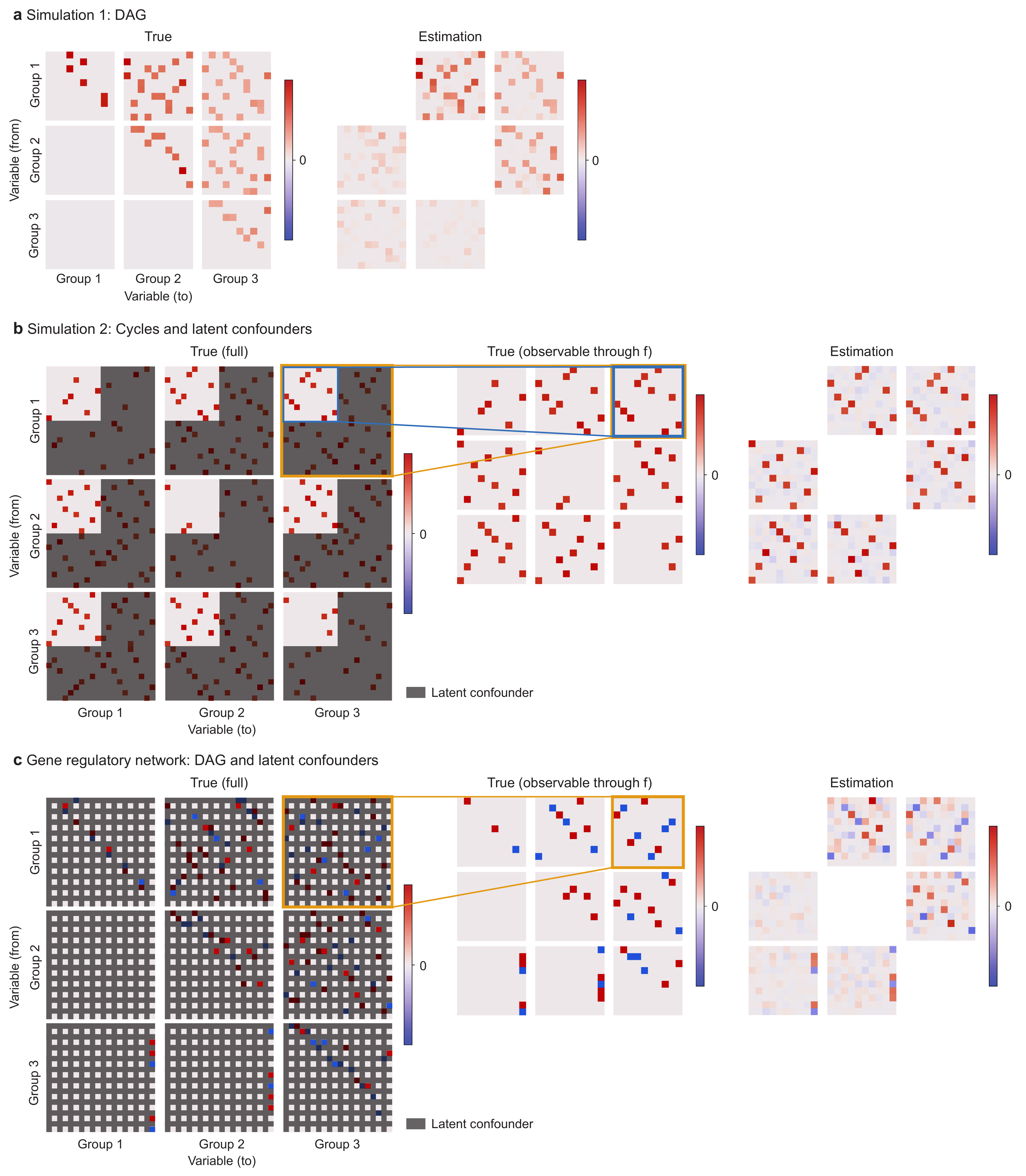}
 \caption{Example of the true causal structures (weighted adjacency matrices) and the estimations by G-CaRL (before causal direction determination and thresholding) in Simulation~1 (\textsf{\textbf{a}}), Simulation~2 (\textsf{\textbf{b}}), and the gene regulatory network recovery task (\textsf{\textbf{c}}).
 G-CaRL only identifies the inter-group-parts of the adjacency matrix, and thus the block-diagonal-parts are left unknown.}
 \label{fig:w_sample}
\end{figure*}

\begin{figure*}[t]
 \centering
 \includegraphics[width=0.9\columnwidth]{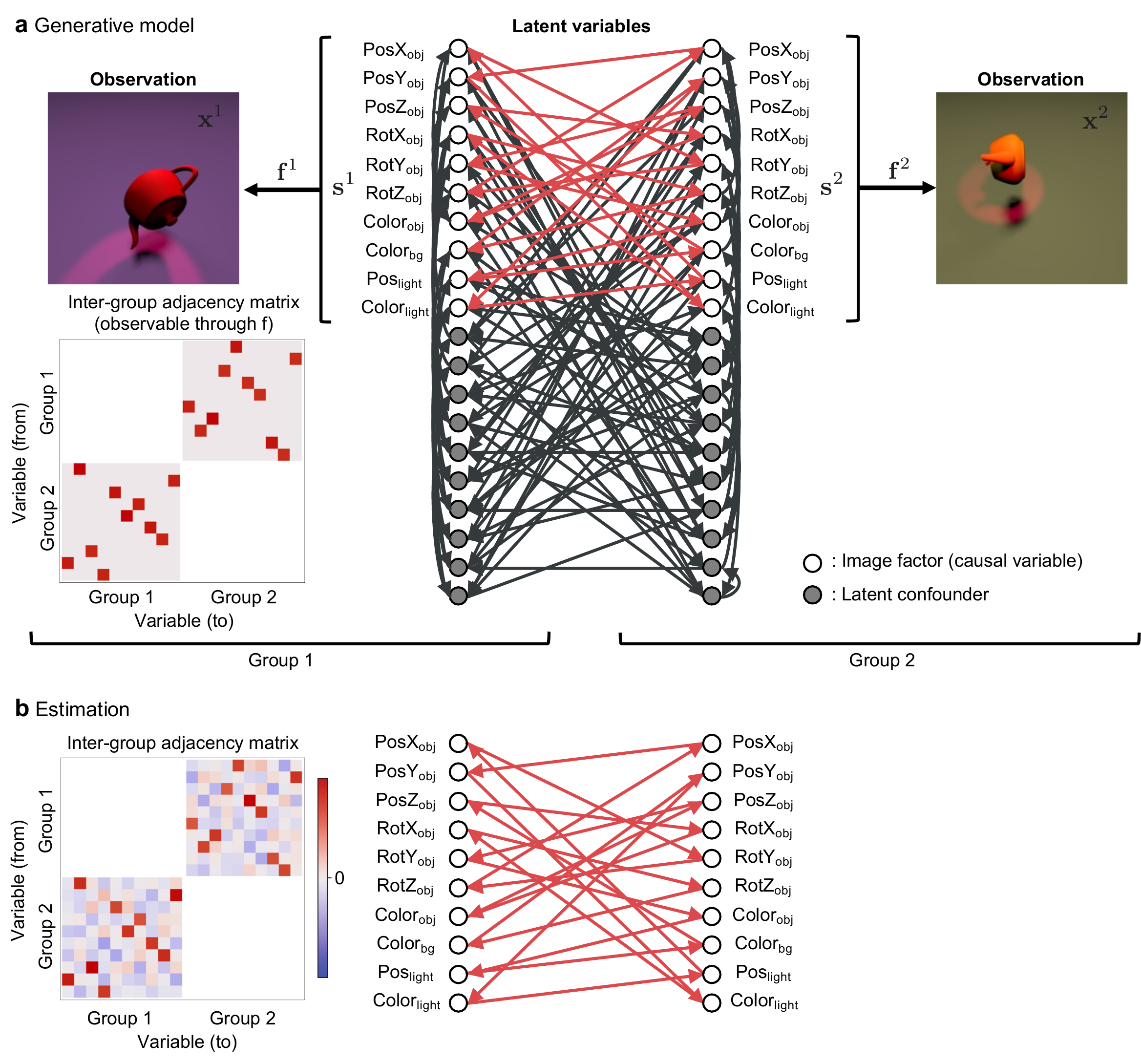}
 \caption{Evaluation of G-CaRL on high-dimensional image observations. We fixed the number of groups to $M=2$ in this figure for brevity. (\textsf{\textbf{a}}) Example of the true causal structures (weighted adjacency matrices) with high-dimensional image observations, and (\textsf{\textbf{b}}) the estimation by G-CaRL (the directed graph visualization is after applying thresholding). 
Each group has ten image-factors (latent causal variables; white circles) conditioning the observation images, and additional ten latent confounders (gray circles) affecting the other variables. 
 The ten image-factors are composed of XYZ positions of the object ($\text{PosX}_\text{obj}$, $\text{PosY}_\text{obj}$, and $\text{PosZ}_\text{obj}$), three dimensions describe the rotation of the object in Euler angles ($\text{RotX}_\text{obj}$, $\text{RotY}_\text{obj}$, and $\text{RotZ}_\text{obj}$), the color of the object and the ground of the scene ($\text{Color}_\text{obj}$ and $\text{Color}_\text{bg}$), and the position and color of the spotlight ($\text{Pos}_\text{light}$ and $\text{Color}_\text{light}$).
 The latent confounders do not have such physical interpretation, but still affect the observation images indirectly.
 What we aim to estimate are the causal edges colored by red in \textsf{\textbf{a}}, connecting the image-factors between groups, which are observable as high-dimensional images.
 G-CaRL only identifies the inter-group-parts of the adjacency matrix, and thus the intra-group graphs are left unknown. The causal graphs related to the latent confounders are also left unknown since the latent confounders are not observable. }
 \label{fig:sim2_image}
\end{figure*}


\end{document}